\documentclass[10pt]{article} % For LaTeX2e
\usepackage[accepted]{tmlr}
\usepackage{hyperref}
\usepackage{url}
\usepackage{float}

\usepackage[utf8]{inputenc} % allow utf-8 input
\usepackage[T1]{fontenc}    % use 8-bit T1 fonts
\usepackage{hyperref}       % hyperlinks
\usepackage{url}            % simple URL typesetting
\usepackage{booktabs}       % professional-quality tables
\usepackage{amsfonts}       % blackboard math symbols
\usepackage{nicefrac}       % compact symbols for 1/2, \emph{etc.}
\usepackage{microtype}      % microtypography
\usepackage{graphicx}
\usepackage{wrapfig}
\usepackage{subcaption}
\usepackage{caption}
\usepackage{tcolorbox}

\usepackage[ruled,linesnumbered]{algorithm2e}
\usepackage{appendix}
\usepackage{amsthm,amssymb,amsmath,amsfonts}
\usepackage{comment}
\usepackage{enumitem}
\usepackage{tikz}
\usepackage{tablefootnote}
\usetikzlibrary{arrows.meta}
\usetikzlibrary{decorations.pathreplacing}
\usetikzlibrary{positioning}
\usetikzlibrary{calc}
\usetikzlibrary{backgrounds}
\usepackage{cellspace}

\usepackage{xcolor}

\usepackage{tikz-cd}
\usepackage{multirow} 
\usepackage{enumitem}
\usepackage{makecell}

\numberwithin{equation}{section}

\newtheorem{theorem}{Theorem}[section]
\newtheorem{lemma}[theorem]{Lemma}
\newtheorem{assumption}{Assumption}[section]

\newtheorem{definition}[theorem]{Definition}
\newtheorem{proposition}[theorem]{Proposition}
\newtheorem{remark}[theorem]{Remark}

\usepackage{amsmath,amsfonts,bm}
\DeclareMathOperator{\Tr}{Tr} 

\def\eqref#1{(\ref{#1})}
\newcommand{\dif}{{\mathrm{d}}}

\def\vzero{{\bm{0}}}

\def\vb{{\bm{b}}}

\def\vn{{\bm{n}}}

\def\vs{{\bm{s}}}

\def\vw{{\bm{w}}}
\def\vx{{\bm{x}}}
\def\vy{{\bm{y}}}
\def\vz{{\bm{z}}}

\def\vgamma{{\bm{\gamma}}}

\def\veta{{\bm{\eta}}}

\def\vphi{{\bm{\phi}}}

\def\mA{{\bm{A}}}

\def\mH{{\bm{H}}}
\def\mI{{\bm{I}}}

\def\mK{{\bm{K}}}

\def\mP{{\bm{P}}}

\def\mLambda{{\bm{\Lambda}}}

\def\mSigma{{\bm{\Sigma}}}

\DeclareMathAlphabet{\mathsfit}{\encodingdefault}{\sfdefault}{m}{sl}
\SetMathAlphabet{\mathsfit}{bold}{\encodingdefault}{\sfdefault}{bx}{n}

\def\gA{{\mathcal{A}}}

\def\gN{{\mathcal{N}}}

\def\gT{{\mathcal{T}}}

\def\gW{{\mathcal{W}}}

\def\sI{{\mathbb{I}}}

\def\sR{{\mathbb{R}}}

\newcommand{\E}{\mathbb{E}}

\newcommand{\Ls}{\mathcal{L}}
\newcommand{\R}{\mathbb{R}}

\newcommand{\KL}{D_{\mathrm{KL}}}

\newcommand{\TV}{\mathrm{TV}}

\makeatletter
\DeclareRobustCommand{\cev}[1]{%
  {\mathpalette\do@cev{#1}}%
}
\newcommand{\do@cev}[2]{%
  \vbox{\offinterlineskip
    \sbox\z@{$\m@th#1 x$}%
    \ialign{##\cr
      \hidewidth\reflectbox{$\m@th#1\vec{}\mkern4mu$}\hidewidth\cr
      \noalign{\kern-\ht\z@}
      $\m@th#1#2$\cr
    }%
  }%
}
\makeatother

\usepackage{etoolbox}

\makeatletter
\patchcmd{\@algocf@start}% <cmd>
  {-1.5em}% <search>
  {0pt}% <replace>
  {}{}% <success><failure>
\makeatother

\PassOptionsToPackage{table}{xcolor} % or \usepackage[table]{xcolor} before others
\usepackage{xcolor}
\usepackage{colortbl}   % for \rowcolor, \cellcolor
\usepackage{multirow}   % for \multirow
\usepackage{graphicx}   % for \resizebox
\usepackage{natbib}     % because you use \citep

\definecolor{lightgray}{gray}{0.95}
\definecolor{bestperformance}{RGB}{187, 214, 232}

\title{Solving Inverse Problems via Diffusion-Based Priors:\\
An Approximation-Free Ensemble Sampling Approach}

\author{\name Haoxuan Chen \email haoxuanc@stanford.edu \\
      \addr Institute for Computational and Mathematical Engineering (ICME)\\
      Stanford University\\
      Stanford, CA 94305
      \AND
      \name Yinuo Ren \email yinuoren@stanford.edu \\
      \addr Institute for Computational and Mathematical Engineering (ICME)\\
      Stanford University\\
      Stanford, CA 94305
      \AND
      \name Martin Renqiang Min \email renqiang@nec-labs.com\\
      \addr Machine Learning Department\\
       NEC Labs America\\
       Princeton, NJ 08540
       \AND
       \name Lexing Ying \email lexing@stanford.edu \\
       \addr Department of Mathematics\\
       Institute for Computational and Mathematical Engineering (ICME)\\
       Stanford University\\
       Stanford, CA 94305
       \AND 
       \name Zachary Izzo \email zach@nec-labs.com\\
       \addr Machine Learning Department\\
       NEC Labs America\\
       Princeton, NJ 08540\\
}

\def\openreview{\url{https://openreview.net/forum?id=qN8ASsfjKs}} % Insert correct link to OpenReview for camera-ready version

\begin{document}

\maketitle

\begin{abstract}
Diffusion models (DMs) have proven to be effective in modeling high-dimensional distributions, leading to their widespread adoption for representing complex priors in Bayesian inverse problems (BIPs). However, current DM-based posterior sampling methods proposed for solving common BIPs rely on heuristic approximations to the generative process. To exploit the generative capability of DMs and avoid the usage of such approximations, we propose an ensemble-based algorithm that performs posterior sampling without the use of heuristic approximations. Our algorithm is motivated by existing work that combines DM-based methods with the sequential Monte Carlo (SMC) method. By examining how the prior evolves through the diffusion process encoded by the pre-trained score function, we derive a modified partial differential equation (PDE) governing the evolution of the corresponding posterior distribution. This PDE includes a modified diffusion term and a reweighting term, which can be simulated via stochastic weighted particle methods. Theoretically, we prove that the error between the true posterior and the empirical distribution of the generated samples can be bounded in terms of the training error of the pre-trained score function and the number of particles in the ensemble. Empirically, we validate our algorithm on several inverse problems in imaging to show that our method gives more accurate reconstructions compared to existing DM-based methods. Our code is available at the following Github repository~\url{https://github.com/HaoxuanSteveC00/AFDPS-TMLR}. 
\end{abstract}

\section{Introduction}

Inverse problems are fundamentally challenging tasks that span multiple scientific and engineering fields like fluid dynamics~\citep{cotter2009bayesian,sellier2016inverse}, geophysics~\citep{richter2021inverse}, medical imaging~\citep{lustig2007sparse}, microscopy~\citep{choi2007tomographic, bertero2021introduction}, etc. These problems basically involve reconstructing an unknown parameter $x$ from incomplete and noise-corrupted measurements $y$. Due to the inherent limitations in measurements, there is often substantial uncertainty in determining the true parameter $x$. Instead of pursuing a single point estimate, a more principled approach involves adopting a Bayesian framework, where we specify a prior distribution on $x$ and characterize the uncertainty through posterior sampling of $p(x|y)$. However, in many practical inverse problems, the prior distribution is already high-dimensional and may contain multiple well-separated modes. Coupled with an ill-posed forward model and noisy observations, such complex priors often induce posteriors that are likewise high-dimensional and strongly multimodal. Consequently, traditional Markov Chain Monte Carlo (MCMC) methods~\citep{neal2011mcmc, welling2011bayesian,cui2016dimension} often struggle with sampling from these posterior distributions primarily due to metastability, \emph{i.e.}, the difficulty in transitioning between distinct high-probability modes that are separated by regions of low probability.

To address these limitations, prior work has leveraged generative models like normalizing flows (NFs)~\citep{asim2020invertible,hou2019solving,zhang2021multiscale,whang2021composing,whang2021solving,hagemann2022stochastic} and generative adversarial networks (GANs)~\citep{patel2019bayesian,bora2017compressed} to model and sample from those high-dimensional and multimodal posterior distributions. Recently, Diffusion models (DMs) and probability flow-based models~\citep{albergo2023stochastic, albergo2022building,lipman2022flow, liu2022flow, sohl2015deep, ho2020denoising,song2020denoising,song2021maximum, song2019generative, song2020score, zhang2018monge} have emerged as leading methods in modern generative modeling. These models generate samples from a high-dimensional target distribution $p_0$ by inverting a diffusion process that transforms the target distribution $\vx_0 \sim p_0$ into a simple distribution $\vx_T \sim p_T$ (typically Gaussian). The effectiveness of DMs has led to their adoption as prior distributions in inverse problems, spawning various DM-based posterior sampling methods~\citep{chung2022diffusion,song2023pseudoinverse,wu2023practical,cardoso2023monte,dou2024diffusion,sun2024provable,xu2024provably,wu2024principled,bruna2024provable}. For a comprehensive review, we refer the readers to either Appendix~\ref{app: related work review}  or~\citep{daras2024survey}. These methods can be categorized into two main approaches:

\begin{enumerate} [leftmargin=15pt]
    \item Methods that leverage Bayes' formula to construct a conditional diffusion model using a pre-trained score function associated with the prior distribution: Specifically, for any time $t \in [0,T]$, applying Bayes' formula $p_{t}(\vx_t|\vy) \propto p_{t}(\vx_t)p_t(\vy|\vx_t)$ yields
    \begin{equation}
    \label{eqn: Bayes' formula of scores}
    \nabla_{\vx_t}\log p_{t}(\vx_t|\vy) = \nabla_{\vx_t} \log p_t(\vx_t) + \nabla_{\vx_t} \log p_t(\vy|\vx_t).    
    \end{equation}
    To implement this approach, one needs to evaluate the left-hand side of~\eqref{eqn: Bayes' formula of scores}, which is known as the conditional score function and defines a reverse-time diffusion process from $p_{T}(\vx_T|\vy)$ to $p_0(\vx_0|\vy)$. The first term on the right-hand side is the score function from the pre-trained DM modeling the prior distribution. The second term requires evaluating an integral $p_t(\vy | \vx_t) = \int p(\vy|\vx_0)p_{0|t}(\vx_0 | \vx_t)\dif \vx_0$ over all possible $\vx_0$'s that could lead to $\vx_t$ through the pre-trained DM, to address which methods in this category employ various approximations for $\nabla_{\vx_t}\log p_t(\vy|\vx_t)$. 
    
    Among different methods belonging to this approach, one group of methods~\citep{song2020score, choi2021ilvr, song2021solving, chung2022diffusion, song2023pseudoinverse, boys2023tweedie, wu2023practical} makes simplifying assumptions, while others~\citep{choi2021ilvr, wang2022zero, kawar2022denoising, rout2023solving} use empirically constructed updates without structured assumptions. These heuristic, problem-specific approximations might be inaccurate in certain scenarios. In particular, for the special case of linear inverse problems, we assume that it is modeled by $y=\vy_0 = \mA\vx_0 + \vn$ with $\vy_0 \in \mathbb{R}^m, \ \vx_0 \in \mathbb{R}^n, \ \mA \in \mathbb{R}^{m \times n }$ and $\vn \sim \gN(\boldsymbol{0},\kappa^2\mI_m)$. Consider a standard and widely used case of DMs, whose associated forward diffusion process is given by $\vx_t = \vx_0 + \sigma(t)\vw$ with injected noise $\vw \sim \mathcal{N}(0,\mI_n)$. The corresponding forward diffusion process for the measurement is further denoted by $\vy_t = \vy_0 +\sigma(t)\veta$, where $\veta = A\vw$ is a transformed multivariate Gaussian distribution in $\mathbb{R}^m$. Then we have the following examples of approximations to the term $\nabla_x \log p_t(\vy|\vx_t)$ used in existing work like Iterative Latent Variable Refinement (ILVR)~\citep{choi2021ilvr} and Diffusion Posterior Sampling (DPS)~\citep{chung2022diffusion}:
    \begin{align}
    \nabla_{\vx_t}\log p_t(\vy|\vx_t) &\approx -\frac{1}{\kappa^2}\left(\mA^\intercal \mA\right)^{-1}\mA^\intercal\left(\vy_t - \mA\vx_t\right), \tag{ILVR} \\
    \nabla_{\vx_t}\log p_t(\vy|\vx_t) &\approx \frac{1}{\kappa^2}(\mI_n + \sigma(t)^2\nabla^2_{\vx_t}\log p_t(\vx_t))^TA^T(\vy-A\mathbb{E}[\vx_0|\vx_t]). \tag{DPS}
    \end{align}
    For a detailed explanation of the intuitions behind the approximations above, we refer the readers to Appendix~\ref{subsubsec: intuition behind posterior score appoximations}.
    \item Approximation-free methods that integrate DMs with traditional posterior sampling methods: Examples include split Gibbs sampler (SGS) + DM methods~\citep{xu2024provably,wu2024principled,coeurdoux2024plug,zheng2025inversebench}, which are built upon the split Gibbs sampler for Bayesian inference~\citep{vono2019split,pereyra2023split}, and sequential Monte Carlo (SMC) + DM methods~\citep{wu2023practical,cardoso2023monte,dou2024diffusion,kelvinius2025solving,skreta2025feynman,lee2025debiasing,holderrieth2025leaps,achituve2025inverse}, which combine DMs with SMC~\citep{liu2001monte,chopin2002sequential,del2006sequential,doucet2009tutorial,del2013mean,moral2004feynman} to obtain asymptotically consistent posterior samples.

    % Approximation-free methods that integrate DMs with traditional posterior sampling methods: Examples include split Gibbs sampler (SGS) + DM methods~\citep{xu2024provably,wu2024principled,coeurdoux2024plug,zheng2025inversebench}, which are built upon the split Gibbs sampler for Bayesian inference~\citep{vono2019split,pereyra2023split}, and sequential Monte Carlo (SMC) + DM methods~\citep{wu2023practical,cardoso2023monte,dou2024diffusion,kelvinius2025solving,skreta2025feynman,lee2025debiasing,holderrieth2025leaps,achituve2025inverse}, which combine DMs with SMC~\citep{liu2001monte,chopin2002sequential,del2006sequential,doucet2009tutorial,del2013mean,moral2004feynman} to obtain asymptotically consistent posterior samples.
\end{enumerate}

%\yinuo{Please add sentences detailing what "approximation-free" means. Possibly summarizing what approximation is made in previous works and how our method is different.}

We advance the second approach by introducing a novel ensemble-based \emph{Approximation-Free Diffusion Posterior Sampler (AFDPS)}. Our method enhances the synergy between DMs and SMC methods, which use weighted particle ensembles and strategic resampling to approximate the posterior distribution. The key innovation stems from our principled utilization of pre-trained DMs for prior evolution and our derivation of the exact partial differential equation (PDE) governing the corresponding posterior evolution, which reveals fundamentally distinct dynamics compared to existing approaches. Leveraging the flexibility of our framework, we propose two different approaches based on SDE and ODE+Corrector, respectively. Through careful analysis of the discrepancy between the derived PDE dynamics and the time-reversal of the true diffusion process, we establish error bounds for our posterior sampling algorithm based on stochastic weighted particle method. In practice, our algorithm demonstrates versatile compatibility with various pre-trained diffusion models. Extensive experiments on various imaging inverse problems are provided to demonstrate the effectiveness of our method.

%%\vspace{-1em}
\paragraph{Our Contributions.} 
We summarize our main contributions as follows:
\begin{itemize}[leftmargin=1.5em, itemsep=0pt, topsep=0pt, parsep=0pt]
    \item We propose a novel ensemble-based posterior sampling method that integrates sequential Monte Carlo with diffusion models to achieve \textbf{exact posterior sampling without heuristic approximations}, founded on rigorously derived and previously unexplored PDE dynamics.
    \item We provide comprehensive theoretical guarantees demonstrating that our ensemble-based algorithm, implemented via stochastic weighted particle methods, \textbf{converges asymptotically to the derived PDE dynamics}. We additionally derive \textbf{precise error bounds} relating posterior sampling accuracy to the quality of the pre-trained score function.
    \item We empirically evaluate our method on multiple imaging inverse problems using large-scale datasets like FFHQ-$256$~\citep{karras2019style} and ImageNet-$256$~\citep{deng2009imagenet}, showing that our method \textbf{achieves better reconstruction quality} over existing methods.
\end{itemize}

%%\vspace{-1.0em}
\section{Preliminaries}
%%add subsection of notations if space allows
\label{sec: preliminaries}
%%\vspace{-1.0em}

In this section, we provide a quick overview of problem setup, basic concepts, and existing work related to solving Bayesian inverse problems (BIPs) with diffusion models. 

%%\vspace{-.5em}
\subsection{Basics of Inverse Problems}
\label{subsec: Basics of Inverse Problems}
%%\vspace{-.5em}

In BIPs, we aim to recover a ground truth parameter $\vx$ from measurements $\vy$. The relationship between $\vx$ and $\vy$ is described by:
\begin{equation}
\label{eqn: BIP setup}
\vy = \gA(\vx) + \vn,    
\end{equation}
where $\vx \in \mathbb{R}^n$, $\vy \in \mathbb{R}^m$, $\gA:\mathbb{R}^{n} \rightarrow \mathbb{R}^m$ is a differentiable forward operator (linear or nonlinear), and $\vn \in \mathbb{R}^m$ represents measurement noise. Under the Bayesian framework, the posterior distribution we seek to sample from is:
\begin{equation}
\label{eqn: Bayes formula}
p(\vx | \vy) \propto p_0(\vx) p(\vy | \vx) = p_0(\vx)\exp(-\mu_{\vy}(\vx)),    
\end{equation}
where $p_0(\vx)$ denotes the prior distribution and $\mu_{\vy}(\vx) = -\log p(\vy | \vx)$ is the negative log-likelihood function for a fixed observation $\vy$.

Many practical inverse problems are ill-posed due to measurement noise and non-injective forward models, making unique solutions impossible to obtain. Traditional optimization-based methods often fail to capture the complex solution landscape, motivating the use of Bayesian formulations where posterior sampling methods can systematically account for uncertainty and explore multiple plausible solutions. For a comprehensive treatment of BIPs, we refer readers to~\citep{stuart2010inverse}.

Deep generative models have emerged as powerful prior distributions that can capture complex solution spaces while remaining computationally tractable. Unlike traditional priors that rely on structural assumptions, these models effectively represent high-dimensional and multi-modal distributions given sufficient training data. In this work, we focus on diffusion models (DMs), which represent the current state-of-the-art in generative modeling with successful applications across physics~\citep{cotler2023renormalizing,habibi2024diffusion,zhu2024quantum}, chemistry~\citep{xu2022geodiff,alakhdar2024diffusion,riesel2024crystal}, biology~\citep{alamdari2023protein,watson2023novo}, computer vision~\citep{rombach2022high,chan2024tutorial}, and natural language processing~\citep{li2022diffusion}.

\subsection{Diffusion Models: the EDM Framework}
\label{subsec: EDM framework}

We adopt the ``Elucidating the design space of Diffusion Models (EDM)'' framework from~\citep{karras2022elucidating} to model prior distributions. The EDM framework provides a unified approach for the design of diffusion models by systematically analyzing noise schedules, sampling algorithms, and training objectives. 

Building on the continuous formulation of diffusion models~\citep{song2020score}, the framework starts off with a forward diffusion process governed by the stochastic differential equation (SDE):
\begin{equation}
\label{eqn: forward sde in standard dm}
\dif \vx_s = F(s)\vx_s\dif s + G(s)\dif \vw_s.
\end{equation}
where $(\vw_s)_{s\geq 0}$ is a standard Brownian motion and $p_s$ denotes the distribution of $\vx_s$, with $p_0$ being the prior distribution from~\eqref{eqn: Bayes formula}. Following~\citep{anderson1982reverse}, the corresponding reverse-time SDE is:
\begin{equation}
\label{eqn: backward sde in standard dm}
\dif \cev{\vx}_t = \left[-F(t)\cev{\vx}_t + \tfrac{G(t)^2 + V(t)^2}{2} \nabla_\vx \log \cev{p}_t(\cev{\vx}_t) \right]\dif t + V(t) \dif \vw_t,
\end{equation}
where $\cev{p}_0 = p_T$, $\cev{p}_T = p_0$, $\cev{\ast}_t$ denotes $\ast_{T-t}$, and $V:\mathbb{R} \rightarrow \mathbb{R}$ is a scalar-valued function. The score function $\nabla \log \cev{p}_t(\vx)$ is typically approximated by a neural network $\vphi_{\theta}(\vx,t)$ trained via score matching~\citep{hyvarinen2005estimation,vincent2011connection}. We use $\widehat{\cev{\vx}}_t$ and $\widehat{\cev{p}}_t$ to denote the particle trajectory and its distribution when using the approximated score function $\vphi_{\theta}(\vx,t)$, with $\widehat{\cev{p}}_0$ being an approximation of the distribution $p_T$ and $\widehat{\cev{p}}_T$ approximating the target distribution $p_0$.

The EDM framework reparameterizes the drift coefficient $F(t)$ and diffusion coefficient $G(t)$ using
\begin{equation*}
s(t):=\exp\left(\int_{0}^{t}F(\xi)\dif\xi\right) \quad \text{and} \quad \sigma(t):= \sqrt{\int_{0}^{t}\frac{G(\xi)^2}{s(\xi)^2}\dif\xi},
\end{equation*}
yielding $F(t) = \frac{\dot{s}(t)}{s(t)}$ and $G(t) = s(t)\sqrt{2\dot{\sigma}(t)\sigma(t)}$. This reparameterization enables more accurate score estimation under appropriate choices of $s$ and $\sigma$, as demonstrated empirically in~\citep{karras2022elucidating} and theoretically in~\citep{wang2024evaluating}. Also, the framework allows for different implementations based on the choice of diffusion coefficient $V$. Setting $V(t)=G(t) = s(t)\sqrt{2\dot{\sigma}(t)\sigma(t)}$ yields the SDE implementation:
\begin{equation}
\label{eqn: EDM SDE}
\dif \widehat{\cev{\vx}}_t = \left[-\tfrac{\dot{s}(t)}{s(t)}\widehat{\cev{\vx}}_t + 2s(t)^2\dot{\sigma}(t)\sigma(t)\vphi_\theta(\widehat{\cev{\vx}}_t,t)\right]\dif t + s(t)\sqrt{2\dot{\sigma}(t)\sigma(t)}\dif \vw_t.    
\end{equation}
Alternatively, setting $V(t) =0$ yields the probability-flow ODE (PF-ODE) implementation:
\begin{equation}
\label{eqn: EDM PF-ODE}
\dif \widehat{\cev{\vx}}_t = \left[-\tfrac{\dot{s}(t)}{s(t)}\widehat{\cev{\vx}}_t + s(t)^2\dot{\sigma}(t)\sigma(t)\vphi_\theta(\widehat{\cev{\vx}}_t,t)\right]\dif t. 
\end{equation}
In practice, it is common to focus on the cases when $p_T$ converges to some Gaussian distribution under the EDM framework. For instance, when $s(t) = 1$, we have that $p_T$ is given by the convolved distribution $p_0 \ast \mathcal{N}(0,\sigma^2(t)\mI_n)$. We note that this is also the setting adopted in both the theoretical analysis and the experiments of this paper.

%%\vspace{-1.0em}
\section{Methodology}
\label{sec: main result - algorithms}
%%\vspace{-1.0em}

In this section, we present the key derivation underlying our posterior sampling algorithm. Our approach can be interpreted as solving a high-dimensional PDE that governs posterior distribution evolution using either the (stochastic) weighted particle method~\citep{degond1989weighted,degond1990deterministic,rjasanow1996stochastic,bossy1997stochastic,talay2003stochastic,raviart2006analysis,chertock2017practical} or the SMC method~\citep{chopin2002sequential,del2006sequential,doucet2009tutorial,del2013mean,moral2004feynman}. Throughout the derivation, we assume the log-likelihood function $\mu_{\vy}(\vx)$ is at least twice differentiable w.r.t. $\vx$ for fixed $\vy$. Details of both algorithmic variants are given in the pseudocode in subsection~\ref{subsection: algorithm details}.
%which holds for most inverse problems in practice. 

\begin{wrapfigure}{r}{0.4\textwidth}
    \vspace{-1.5em}
    \centering
    \begin{tikzpicture}
        % Define Seaborn colors
        \definecolor{seabornBlue}{RGB}{31,119,180}

        % Define node positions
        \node (p0) at (0.5,1.3) {$\widehat{\cev{p}}_0$};
        \node (pT) at (3.2,1.3) {$\widehat{\cev{p}}_T$};
        \node (Q0) at (0.5,0) {$\widehat{Q}_{\alpha,\vy}(\vx, 0)$};
        \node (QT) at (3.2,0) {$\widehat{Q}_{\alpha,\vy}(\vx, T)$};
        \node[seabornBlue] (q0) at (0.5,-1.3) {$\widehat{q}_{\alpha,\vy}(\vx, 0)$};
        \node[seabornBlue] (qT) at (3.2,-1.3) {$\widehat{q}_{\alpha,\vy}(\vx, T)$};

        % Use Seaborn colors for shading
        \begin{scope}[on background layer]
            \fill[gray!10] (-0.8,1.8) rectangle (4.5,0.9);
            \fill[seabornBlue!15] (-0.8,0.9) rectangle (4.5,-1.8);
        \end{scope}
        
        % Draw dashed axes
        \draw[dashed] (-0.8,0.9) -- (4.5,0.9);
        % \draw[dashed] (2.8,3.5) -- (2.8,-0.5);

        % Add labels for quadrants with Seaborn colors
        \node[gray!50] at (1.75,1.3) {\large \textbf{Prior}};
        \node[seabornBlue!50] at (1.75,-0.45) {\large \textbf{Posterior}};

        \draw[->] (p0) -- node[above] {\eqref{eqn: PDE of prior diffusion}} (pT);
        \draw[->] (p0) -- node[left] {$e^{-\alpha_t \mu_\vy}$} (Q0);
        \draw[->] (pT) -- node[right] {$e^{-\alpha_t\mu_\vy}$} (QT);
        \draw[->] (Q0) -- node[above] {\eqref{eqn: PDE of unnormalized posterior diffusion}} (QT);
        \draw[->] (Q0) -- node[left] {$\widehat Z_{\alpha,\vy}^{-1}$} (q0);
        \draw[->] (QT) -- node[right] {$\widehat Z_{\alpha,\vy}^{-1}$} (qT);
        \draw[->, seabornBlue] (q0) -- node[above] {\eqref{eqn: PDE of posterior diffusion}} node[below, seabornBlue] {II} (qT);
        \draw[seabornBlue] (q0) +(-0.7,-0.3) rectangle +(0.7,0.3);
        \node[seabornBlue, left of=q0] {I};
    \end{tikzpicture}
    \caption{A roadmap for our posterior sampling method. I, II refers to the two stages of the proposed algorithm.}
    \vspace{-2.5em}
    \label{fig:commutative diagram for algo derivation}
\end{wrapfigure}

%%%\vspace{-0.5em}
\subsection{Algorithm Outline}
\label{subsec: algo outline}
%\vspace{-0.5em}

Following the setting in Section~\ref{sec: preliminaries}, we assume the prior distribution $p(\vx)$ is represented by a DM under the EDM framework. Specifically, $p_0(\vx)$ is approximated by $\widehat{\cev{p}}_T(\vx)$, obtained by simulating~\eqref{eqn: EDM SDE} or~\eqref{eqn: EDM PF-ODE} from a Gaussian $\widehat{\cev{p}}_0$. Fix some parametrized curve $\alpha_t = \alpha(t):[0,T] \rightarrow [0,1]$ with $\alpha_0 > 0$. We then define the associated time-dependent posterior distribution as:
\begin{equation}
    \setlength{\abovedisplayskip}{3pt}
    \setlength{\belowdisplayskip}{3pt}
    \label{eqn: defn of time-dependent posterior}
    \widehat{q}_{\alpha,\vy}(\vx,t):= \frac{\widehat{\cev{p}}_t(\vx)e^{-\alpha_t\mu_{\vy}(\vx)}}{\int_{\mathbb{R}^n}\widehat{\cev{p}}_t(\vx)e^{-\alpha_t\mu_{\vy}(\vx)}\dif\vx} := \frac{\widehat{Q}_{\alpha,\vy}(\vx,t)}{\widehat{Z}_{\alpha,\vy}(t)},
\end{equation} 
where $\widehat{Q}_{\alpha,\vy}(\vx,t) = \widehat{\cev{p}}_t(\vx)e^{-\alpha_t\mu_{\vy}(\vx)}$ denotes the unnormalized posterior and $\widehat{Z}_{\alpha,\vy}(t) = \int_{\mathbb{R}^n}\widehat{Q}_{\alpha,\vy}(\vx,t)\dif\vx$ represents the normalizing constant.

Our algorithm consists of the following two stages:
\vspace{-0.5em}
\paragraph{Stage I: Sample from the initial distribution $\widehat{q}_{\alpha,\vy}(\vx,0)$.}
We first need to sample from the initial distribution $\widehat{q}_{\alpha,\vy}(\vx,0) \propto \widehat{\cev{p}}_0(\vx)e^{-\alpha_0\mu_{\vy}(\vx)}$, which is analogous to the likelihood step in~\citep{wu2024principled}. This step is highly related to our choice of the parametrized curve $\alpha$, which in turn induces different distribution paths. In general, when the likelihood function $\mu_{\vy}(\vx)$ is differentiable, we note that $\widehat{q}_{\alpha,\vy}(\vx,0)$ can be approximately sampled by well-known gradient-based samplers like Metropolis Adjusted Langevin Algorithm (MALA)~\citep{roberts2002langevin}, Annealed 
Importance Sampling (AIS)~\citep{neal2001annealed}, or more advanced methods~\citep{lu2019accelerating,tan2023accelerate,chen2024ensemble,lindsey2022ensemble}.
However, for certain special choices of the curve $\alpha$, we note that the initial distribution can also be sampled in an exact way. For instance, when $\alpha_t$ is an annealing-type curve with $\alpha_0= 0$ and $\alpha_1 = 1$ (e.g., the linear schedule $\alpha_t = \frac{t}{T}$), the initial distribution $\widehat{q}_{\alpha,\vy}(\vx,0) = \widehat{\cev{p}}_0$ is Gaussian and can be sampled directly. Moreover, for linear BIPs with Gaussian noise, where $\gA:= \mA \in \mathbb{R}^{m \times n}$ and $\vn \sim \gN(\boldsymbol{0},\mSigma)$, the initial distribution is also normally distributed. Specifically, assuming $\widehat{\cev{p}}_0 = \gN(\vzero, \rho^2\mI_n)$, we have that $\widehat{q}_{\alpha,\vy}(\vx,0)$ simplifies to:
\begin{equation*}
    \setlength{\abovedisplayskip}{3pt}
    \setlength{\belowdisplayskip}{3pt}
    \widehat{q}_{\alpha,\vy}(\vx,0) \propto \exp\left(-\frac{\alpha_0}{2}(\vy - \mA\vx)^\intercal\mSigma^{-1}(\vy -\mA\vx)-\tfrac{1}{2\rho^2}\|\vx\|_2^2\right) = \gN(\vgamma, \mLambda^{-1}),
\end{equation*}
where $\mLambda = \alpha_0\mA^\intercal\mSigma^{-1}\mA+\frac{1}{\rho^2}\mI_n$ and $\vgamma = \alpha_0\mLambda^{-1}\mA^\intercal\mSigma^{-1}\vy$.  
% We remark that our first step here is analogous to the likelihood step in~\citep{wu2024principled}.

%\vspace{-0.5em}
\paragraph{Stage II: Solve the PDE dynamics governing the posterior evolution.}

Below we first derive the PDE dynamics $\left(\widehat{Q}_{\alpha,\vy}(\vx,t)\right)_{t\in[0, T]}$ from $\left(\widehat{\cev{p}}_t\right)_{t\in[0, T]}$ based on the diffusion process~\eqref{eqn: backward sde in standard dm}, which consists of the following two steps and corresponds to the first phase in Figure~\ref{fig:commutative diagram for algo derivation} above.
%Then normalizing these dynamics yields the PDE that evolves $(\widehat{q}_{\vy}(\vx,t))_{t\in[0, T]}$, as illustrated in Figure~\ref{fig:commutative diagram for algo derivation}. 

\begin{enumerate}[leftmargin=*, topsep=0pt, itemsep=2pt, wide]
    \item The Fokker-Planck equation evolving from $\widehat{\cev{p}}_0$ to $\widehat{\cev{p}}_T$ is given by
    \begin{equation}
        \setlength{\abovedisplayskip}{3pt}
        \setlength{\belowdisplayskip}{3pt}
        \label{eqn: PDE of prior diffusion}
        \frac{\partial}{\partial t}\widehat{\cev{p}}_t = -\nabla_{\vx} \cdot \left(\left(-F(t)\vx + \tfrac{G(t)^2 + V(t)^2}{2} \vphi_\theta(\vx,t)\right)\widehat{\cev{p}}_t\right) + \tfrac{1}{2}V(t)^2\Delta_{\vx}\widehat{\cev{p}}_t.
    \end{equation}
    \item Substituting $\widehat{\cev{p}}_t(\vx) = \widehat{Q}_{\alpha,\vy}(\vx,t)\exp(\alpha_t\mu_{\vy})$ into \eqref{eqn: PDE of prior diffusion} yields:
    \begin{equation}
        \setlength{\abovedisplayskip}{3pt}
        \setlength{\belowdisplayskip}{3pt}
        \label{eqn: PDE of unnormalized posterior diffusion}
        \begin{aligned}
        \frac{\partial}{\partial t}\widehat{Q}_{\alpha,\vy} = &-\nabla_{\vx} \cdot \left(\left(\widehat{\mH}(\vx,t)-\alpha_tV(t)^2\nabla_{\vx}\mu_{\vy}\right)\widehat{Q}_{\alpha,\vy}\right)
        +\tfrac{1}{2}V(t)^2\Delta_{x}\widehat{Q}_{\alpha,\vy}\\
        &+\left(U_{\alpha,\vy}(\vx,t) - \alpha_t\widehat{\mH}(\vx,t)^\intercal\nabla_{\vx}\mu_{\vy}-\alpha'_t\mu_{\vy}\right)\widehat{Q}_{\alpha,\vy},
        \end{aligned}    
    \end{equation}
    where $\widehat{\mH}(\vx,t) := -F(t)\vx + \frac{G(t)^2 + V(t)^2}{2} \vphi_\theta(\vx,t)$ and $U_{\alpha,\vy}(\vx,t) := \frac{1}{2}V(t)^2\left(\alpha^2_t\|\nabla_{\vx}\mu_{\vy}\|_2^2-\alpha_t\Delta_{\vx}\mu_{\vy}\right)$. We remark that the PDE above consists of three main components: the first term corresponds to the drift, the second to the diffusion, and the third to the reweighting term. On the one hand, the term $\widehat{\mH}(x,t)$ above is exactly the original drift used in the Fokker Planck equation~\eqref{eqn: PDE of prior diffusion} that $(p_t)_{t \in [0,T]}$ satisfies. On the other hand, both the additional term $V(t)^2\nabla_{\vx}\mu_{\vy}$ in the drift and the function $U(\vx,t)$ in the reweighting term are only dependent on $V(t)$ and $\mu_{\vy}$ from our derivation, which originates from the tilted term $e^{\mu_{\vy}}$. For a complete derivation of~\eqref{eqn: PDE of unnormalized posterior diffusion}, the readers may refer to Lemma~\ref{lem: simplifying dynamics} in Appendix~\ref{sec: derivations for Sec 3}.
    
    \item By averaging the linear term on the RHS of~\eqref{eqn: PDE of unnormalized posterior diffusion} to ``normalize'' the PDE dynamics, which corresponds to the second phase in Figure~\ref{fig:commutative diagram for algo derivation}, we obtain the PDE dynamics governing the evolution of $\widehat{q}_{\alpha,\vy}(\vx,t)$ as follows:
    \begin{tcolorbox}[colback=lightgray,boxrule=0pt,sharp corners]
        %%\vspace{-0.3em}
        \underline{\bfseries PDE Dynamics for Posterior Evolution}
        \begin{equation}
            \setlength{\abovedisplayskip}{3pt}
            \setlength{\belowdisplayskip}{3pt}
            \label{eqn: PDE of posterior diffusion}
            \begin{aligned}
            \frac{\partial}{\partial t}\widehat{q}_{\alpha,\vy} = &-\nabla_{\vx} \cdot \left(\left(\widehat{\mH}(\vx,t)-\alpha_tV(t)^2\nabla_{\vx}\mu_{\vy}\right)\widehat{q}_{\alpha,\vy}\right)
            +\tfrac{1}{2}V(t)^2\Delta_{x}\widehat{q}_{\alpha,\vy}\\
            &+\left(W_{\alpha,\vy}(\vx,t) - \int_{\mathbb{R}^n}W_{\alpha,\vy}(\vz,t)\widehat{q}_{\alpha,\vy}(\vz)\dif \vz\right)\widehat{q}_{\alpha,\vy}.
            \end{aligned}
        \end{equation}
        % %%\vspace{-1.0em}
    \end{tcolorbox}
    where $W_{\alpha,\vy}(\vx,t):= U_{\alpha,\vy}(\vx,t) - \alpha_t\widehat{\mH}(\vx,t)^\intercal\nabla_{\vx}\mu_{\vy}-\alpha'_t\mu_{\vy}(\vx)$ in the reweighting part above. We note that the normalization technique have also been used in recent works like~\citep{skreta2025feynman,lee2025debiasing,holderrieth2025leaps}. For a complete proof one may refer to Lemma~\ref{lem: normalizing linear term in dynamics} in Appendix~\ref{sec: derivations for Sec 3}. Moreover, one of our method's key novelties is that the derived PDE~\eqref{eqn: PDE of posterior diffusion} includes a linear term proportional to $\widehat{q}_{\alpha,\vy}$, which needs to be simulated via keeping track of an ensemble of weighted samples. Intuitively, particles with larger weights are more likely to lie in high-probability regions of the posterior distribution and thus correspond to reconstructions of higher quality. In fact, such procedure also admits an interpretation based on debiasing, as maintaining an ensemble of weighted particles essentially helps reduce simulation bias.
\end{enumerate}

%\begin{remark}
%\label{rmk: derivation of true dynamics}
%By replacing $\vphi_{\theta}(\vx,t)$ and $\widehat{\mH}(\vx,t)$ with $\nabla\log\cev{p}_{t}(\vx)$ and $H(\vx,t):=-F(t)\vx + \frac{G(t)^2+V(t)^2}{2}\nabla\log\cev{p}_{t}(\vx)$ respectively, we can perform the derivation of the PDE dynamics~\eqref{eqn: PDE of prior diffusion},~\eqref{eqn: PDE of unnormalized posterior diffusion} and~\eqref{eqn: PDE of posterior diffusion} in exactly the same way to obtain new PDE dynamics associated with the true score function. Though these dynamics can not lead to algorithms we used in practice, they will be useful for the theoretical analysis conducted in subsection~\ref{subsec: error bound btween posterior in terms of prior} later on. 
%\end{remark}

%%\vspace{-0.5em}
\subsection{Posterior Sampling via Weighted Particles}
\label{subsection: algorithm details}
%%\vspace{-0.5em}

We now present two ensemble-based posterior samplers within the SMC framework, which can also be interpreted as solving the PDE~\eqref{eqn: PDE of posterior diffusion} numerically via (stochastic) weighted particles. Below we use $(\vx_t,\beta_t)$ to denote the time-dependent position and weight associated with a single particle. The joint probability distribution of $(\vx_t,\beta_t)$ is further denoted by $\gamma_t = \gamma_t(\vx,\beta)$.

\paragraph{(Stochastic) Weighted Particle / Sequential Monte Carlo Methods.}

As shown in Lemma~\ref{lem: derivation of single particle dynamics} of Appendix~\ref{sec: derivations for Sec 3}, the posterior evolution~\eqref{eqn: PDE of posterior diffusion} can be simulated via the following dynamics of a single weighted particle:
\begin{equation}
    \setlength{\abovedisplayskip}{3pt}
    \setlength{\belowdisplayskip}{3pt}
    \label{eqn: PDE soln via single particle} 
    \begin{cases}
    \dif \vx_t &= \left(\widehat{\mH}(\vx_t,t)-\alpha_tV(t)^2\nabla_{\vx}\mu_{\vy}(\vx_t)\right)\dif t + V(t)\dif \vw_t,\\
    \dif \beta_t &= \left(U_{\alpha,\vy}(\vx_t,t) - \alpha_t\widehat{\mH}(\vx_t,t)^\intercal\nabla_{\vx}\mu_{\vy}(\vx_t)-\alpha'_t \mu_{\vy}(\vx_t)\right)\beta_t \dif t \\
    &- \left(\displaystyle\int_{\mathbb{R}^n}\left(U_{\alpha,\vy}(\vx,t) - \alpha_t\widehat{\mH}(\vx,t)^\intercal\nabla_{\vx}\mu_{\vy}(\vx)-\alpha'_t \mu_{\vy}(\vx)\right)\left(P_{\beta}\gamma_t\right)(\vx)\dif \vx\right)\beta_t \dif t,    
    \end{cases}
\end{equation}
where $P_{\beta}\gamma_t(\vx) := \int_{\mathbb{R}}\beta\gamma_t(\vx,\beta)\dif \beta$ above denotes the weighted projection of $\gamma_t$ onto $\vx$. To effectively approximate the integral in $P_\beta \gamma_t$, we then use the empirical measure $\gamma_t(\vx,\beta) \approx \frac{1}{N}\sum_{i=1}^{N}\delta_{(\vx^{(i)}_t,\beta^{(i)}_t)}$ formed by $N$ weighted particles to approximate $\gamma_t(\vx,\beta)$. This leads to the following joint dynamics for $\{(\vx^{(i)}_t, \beta^{(i)}_t)\}_{i=1}^{N}$:
\begin{tcolorbox}[colback=lightgray,boxrule=0pt,sharp corners]
    %%\vspace{-0.3em}
    \underline{\bfseries Weighted Particle Dynamics for Posterior Evolution}
    %%\vspace{-0.5em}
    \begin{equation}
    \label{eqn: PDE soln via ensemble of particles} 
    \begin{cases}
    \dif \vx^{(i)}_t &= \left(\widehat{\mH}(\vx^{(i)}_t,t)-\alpha_tV(t)^2\nabla_{\vx}\mu_{\vy}(\vx^{(i)}_t)\right)\dif t + V(t)\dif \vw^{(i)}_t,\\
    \dif \beta^{(i)}_t &= \left(U_{\alpha,\vy}(\vx^{(i)}_t,t) - \alpha_t\widehat{\mH}(\vx^{(i)}_t,t)^\intercal\nabla_{\vx}\mu_{\vy}(\vx^{(i)}_t)-\alpha'_t \mu_{\vy}(\vx^{(i)}_t)\right)\beta^{(i)}_t \dif t \\
    &- \left(\frac{1}{N}\sum_{j=1}^{N}\left(U_{\alpha,\vy}(\vx^{(j)}_t,t) - \alpha_t\widehat{\mH}(\vx^{(j)}_t,t)^\intercal\nabla_{\vx}\mu_{\vy}(\vx^{(j)}_t)-\alpha'_t \mu_{\vy}(\vx^{(j)}_t)\right)\beta^{(j)}_t\right)\beta^{(i)}_t \dif t,  
    \end{cases}
    \end{equation}
    %%\vspace{-1.0em}
\end{tcolorbox}
with initial conditions $\vx^{(i)}_0 \sim \widehat{q}_{\alpha,\vy}(\cdot,0)$ and $\beta^{(i)}_0 = 1$ for $i \in [N]$. The weighted projection equals $\frac{1}{N}\beta^{(j)}_t$ when $\vx = \vx^{(j)}_t$ for some $j$, and zero otherwise. While numerical discretization of~\eqref{eqn: PDE soln via ensemble of particles} yields a prototypical sampling algorithm, the particle weights $\beta^{(i)}_t$ may diverge during simulation, reducing the ensemble's Effective Sample Size (ESS). To address this, we employ a resampling strategy commonly used in the SMC methods~\citep{liu2001monte,chopin2002sequential,del2006sequential,doucet2009tutorial,del2013mean,moral2004feynman}, whose detailed description is provided in Algorithm~\ref{alg:resampling} below. 
 Such resampling sub-routine essentially performs global moves by eliminating low-weight particles and duplicating high-weight ones, similar to the birth-death process used in~\citep{moral2004feynman,lu2019accelerating,tan2023accelerate,chen2024ensemble,lindsey2022ensemble}. However, the resampling approach is computationally more efficient as the weight dynamics~\eqref{eqn: PDE soln via ensemble of particles} can be parallelized.

%\vspace{-0.5em}
\IncMargin{1.5em}
% \begin{wrapfigure}{r}{0.53\textwidth}
\begin{algorithm}[H]
        \caption{Resampling Step}
        \label{alg:resampling}
        \Indm
        \KwIn{Threshold $c \in (0,1)$, weighted particles $\{(\vx^{(j)}, \beta^{(j)})\}_{j=1}^{N}$}
        \KwOut{Updated particles $\{(\widehat{\vx}^{(j)}, \widehat{\beta}^{(j)})\}_{j=1}^{N}$}
        \Indp
        \uIf{$\text{ESS} = \frac{\left(N^{-1}\sum_{j=1}^{N}\beta^{(j)}\right)^2}{N^{-1}\sum_{j=1}^{N}(\beta^{(j)})^2}< c$}{
        Sample $\{\widehat{\vx}^{(j)}\}_{j=1}^{N}$ with replacement from $\{\vx^{(j)}\}_{j=1}^{N}$ with probability \small$\left\{\frac{\beta^{(j)}}{\sum_{i=1}^{N}\beta^{(i)}}\right\}_{j=1}^N$\;
        $\widehat{\beta}^{(j)} \leftarrow 1$, for $j \in [N]$\;
        }
        \Else{
        $\{(\widehat{\vx}^{(j)}, \widehat{\beta}^{(j)})\}_{j=1}^{N} \leftarrow \{(\vx^{(j)}, \beta^{(j)})\}_{j=1}^{N}$\;
        }
\end{algorithm}
%\vspace{-0.5em}
% \end{wrapfigure}

\paragraph{SDE Approach (AFDPS-SDE).}

We first consider the SDE implementation~\eqref{eqn: EDM SDE} of the diffusion model, where $V(t) = G(t) = s(t)\sqrt{2\dot{\sigma}(t)\sigma(t)}$. We directly discretize~\eqref{eqn: PDE soln via ensemble of particles} with an Euler-Maruyama scheme and add Algorithm~\ref{alg:resampling} as an adjustment step at the end of each iteration, which leads to Algorithm~\ref{alg:sde}. We have omitted the averaging term, \emph{i.e.}, the last line of~\eqref{eqn: PDE soln via ensemble of particles} in Algorithm~\ref{alg:sde}, in practical implementation since the update is the same for all particles and therefore cancels out when we normalize the weights. Such cancellation property also holds for the ODE approach presented below. For high-dimensional problems, we can further reduce the computational cost of both the SDE and the ODE approach via practical techniques like using a smaller ensemble, omitting the resampling step, and simply returning the particle with the highest weight as the best estimator, as discussed in subsection~\ref{subsec: linear inverse problems} and Appendix~\ref{sec: implementation detail app}.

\paragraph{ODE+Corrector Approach (AFDPS-ODE).}

Next, we consider an alternative implementation based on the probability flow ODE~\eqref{eqn: EDM PF-ODE} by setting $V(t)=0$. While this leads to the ODE dynamics~\eqref{eqn: PDE soln via ensemble of particles}, relying solely on deterministic evolution may not sufficiently explore the target distribution. To enhance exploration, we incorporate a stochastic corrector step inspired by predictor-corrector schemes in diffusion models~\citep{song2020score,chen2024probability,bradley2024classifier}. The corrector uses the Unadjusted Langevin Algorithm (ULA, Algorithm~\ref{alg:corrector}) to draw samples from the intermediate posterior distribution $\widehat{q}_{\alpha,\vy}(\vx,t) \propto \widehat{\cev{p}}_t(\vx)\exp(-\alpha_t\mu_{\vy}(\vx))$ at each timestep. The complete ODE+Corrector algorithm (Algorithm~\ref{alg:edm pf-ode + corrector}) is thus obtained by discretizing the probability flow ODE~\eqref{eqn: PDE soln via ensemble of particles}, and applying both resampling (Algorithm~\ref{alg:resampling}) and ULA correction (Algorithm~\ref{alg:corrector}) steps for adjustments.

\IncMargin{1.5em}
\setlength{\textfloatsep}{10pt} 
\begin{algorithm}[H]
    \caption{Approximation-Free Diffusion Posterior Sampler via SDE (AFDPS-SDE)}
    \label{alg:sde}
    \Indm
    \KwIn{Noisy observation $\vy$, log-likelihood $\mu_{\vy}(\cdot)$, functions $s(t)$ and $\sigma(t)$, parametrized curve $\alpha_t$, time grid $\{t_i\}_{i=0}^{K}$ with $t_0 = 0$ and $t_K = T$, thresholds $\{c_l\}_{l=1}^{K}$, score function $\vphi_\theta(\cdot,t)$, ensemble size $N$, initial weights $\beta^{(j)}_0 = 1$ for $j\in[N]$.}
    \KwOut{Posterior approximation $\sum_{j=1}^{N}\beta^{(j)}_T \delta_{\vx^{(j)}_T}/\sum_{j=1}^{N}\beta^{(j)}_T$.}
    \Indp
    Draw $\{\vx^{(i)}_{0}\}_{i=1}^{N}$ i.i.d. from $\widehat{q}_{\alpha,\vy}(\cdot,0)$ via Stage I samplers in Section~\ref{subsec: algo outline}\;
    \For{$k = 0$ \KwTo $K-1$}{
    Draw $\{\xi^{(j)}_{k}\}_{j=1}^{N}$ i.i.d. from $\gN(\boldsymbol{0},\mI_n)$\;
        \For{$j = 1$ \KwTo $N$}{\small
        $\begin{aligned}
            \widehat{\vx}^{(j)}_{t_{k+1}} &\leftarrow \left(1-(t_{k+1}-t_k)\tfrac{\dot{s}(t_k)}{s(t_k)}\right)\vx^{(j)}_{t_{k}} + s(t_k)\sqrt{2\dot{\sigma}(t_k)\sigma(t_k)(t_{k+1}-t_k)}\xi^{(j)}_{k}\\ 
            &+ 2(t_{k+1}-t_k)s(t_k)^2\dot{\sigma}(t_k)\sigma(t_k)\left(\vphi_{\theta}(\vx^{(j)}_{t_k},t_k) -\alpha_{t_k}\nabla_{\vx}\mu_{\vy}(\vx^{(j)}_{t_k})\right);
        \end{aligned}$
        $\begin{aligned}
            \log\widehat{\beta}^{(j)}_{t_{k+1}} &\leftarrow \log\beta^{(j)}_{t_{k+1}} -\alpha'_{t_k}(t_{k+1}-t_k)\mu_{\vy}(\vx^{(j)}_{t_k}) +(t_{k+1}-t_k)\alpha_{t_k}\tfrac{\dot{s}(t_k)}{s(t_k)}\nabla_{\vx}\mu_{\vy}(\vx^{(j)}_{t_k})^\intercal \vx^{(j)}_{t_k}\\
            &-2(t_{k+1}-t_k)\alpha_{t_k}s(t_k)^2\dot{\sigma}(t_k)\sigma(t_k)\nabla_{\vx}\mu_{\vy}(\vx^{(j)}_{t_k})^\intercal\vphi_{\theta}(\vx^{(j)}_{t_k},t_k)\\
            &+ (t_{k+1} - t_k)s(t_k)^2\dot{\sigma}(t_k)\sigma(t_k)\left(\alpha^2_{t_k}\|\nabla_{\vx}\mu_{\vy}(\vx^{(j)}_{t_k})\|_2^2-\alpha_{t_k}\Delta_{\vx}\mu_{\vy}(\vx^{(j)}_{t_k})\right);
        \end{aligned}$
        }
    {\small$\{(\vx^{(j)}_{t_{k+1}},\beta^{(j)}_{t_{k+1}})\}_{j=1}^{N} \leftarrow$ Algorithm~\ref{alg:resampling}$\left(c_{k+1},\{(\widehat{\vx}^{(j)}_{t_{k+1}},\widehat{\beta}^{(j)}_{t_{k+1}})\}_{j=1}^{N}\right)$\;}
    }
\end{algorithm}

\begin{remark}[A brief comparison between the two proposed methods AFDPS-SDE and AFDPS-ODE] On the one hand, we remark that the ODE-based method typically requires tuning of the hyperparameters associated with the corrector, such as the step size and number of steps, whereas the SDE-based method does not. On the other hand, the ODE-based method is often more amenable to the design of higher-order numerical solvers with improved accuracy and faster sampling speed. More broadly, we note that the relative advantages of SDE versus ODE formulations remain an active research topic, and the readers may refer to~\citep{song2020score,cao2023exploring,chen2024probability} for related studies.    
\end{remark}

%\vspace{-0.5 em}

%$V(t) = G(t) = s(t)\sqrt{2\dot{\sigma}(t)\sigma(t)}$ function our algorithms might also need to be modified accordingly. For SDE one can directly obtain %We note that Algorithm 1~\ref{alg:resampling} essentially  can also be implemented via the birth death process used in~\citep{lu2019accelerating,tan2023accelerate,chen2024ensemble,lindsey2022ensemble} in~\eqref{eqn: PDE soln via ensemble of particles},  For high-dimensional problems with high computational cost, we may also keep an ensemble of moderate we can also save for largest weight only,Note that, which helps us save computational cost. also can retur just one, implement in parallel ,

%\vspace{-0.5em}

%\vspace{-0.5 em}

\begin{remark}[Connection with Feynman-Kac corrector~\citep{skreta2025feynman} and Guidance~\citep{dhariwal2021diffusion,bradley2024classifier}]
Here we will discuss the main novelty of our method compared to a concurrent work~\citep{skreta2025feynman}, which also proposed an ensemble-based sampler within the SMC framework. For a more detailed comparison with other SMC-based methods, the readers may refer to Appendix~\ref{subsubsec: comparison with other SMC methods}. Specifically, the dynamics derived in our setting differ from those in Proposition~D.5 of~\citep{skreta2025feynman}, which is essentially the ODE case without correctors in our method. The key difference is the presence of a gradient term, $\nabla_{\vx}\mu_{\vy}$, in the dynamics of $\vx_t$~\eqref{eqn: PDE soln via single particle}, which is absent in their formulation. Such component, previously used in SGS + DM methods~\citep{xu2024provably,wu2024principled} and in optimization-based denoising algorithms such as ADMM~\citep{gabay1976dual,wang2008new,boyd2011distributed,sun2016deep,chan2016plug,ryu2019plug} and FISTA~\citep{beck2009fast,zhang2018ista,xiang2021fista}, is incorporated into our DM-based framework in a systematic way. The derivation illustrated in Figure~\ref{fig:commutative diagram for algo derivation} is shown to be essential for the method's empirical performance (\emph{cf.}~Section~\ref{sec: results of numerics}). A detailed comparison is given in Remark~\ref{rmk: difference between our dynamics and FK corrector} of Appendix~\ref{sec: derivations for Sec 3}.

In contrast to prior work on guided diffusion sampling~\citep{dhariwal2021diffusion,bradley2024classifier,wu2024theoretical,chidambaram2024does} and its extensions~\citep{ho2022classifier,bansal2023universal,song2023loss,he2023manifold,guo2024gradient,lu2024guidance,zheng2024ensemble,ye2024tfg}, which augment single-particle dynamics with a gradient term such as $\nabla_\vx\log p_{t}(\vy|\vx_t)$ or $\nabla_\vx\log p_{t}(\vx_t|\vy)$, our PDE-based derivation naturally yields the gradient term $\nabla_{\vx}\mu_{\vy}$ within a principled framework. Additionally, our formulation introduces a linear term that must be simulated via an ensemble of weighted particles, rather than from a single trajectory. Such ensemble-based structure allows us to integrate gradient-based guidance and diffusion sampling under the SMC framework in a unified way, resulting in improved empirical performance.
\end{remark}

% \begin{wrapfigure}{r}{0.5\textwidth}
    %%\vspace{-1.5 em}
    %%\vspace{-1.5em}
% \end{wrapfigure}

\IncMargin{1.5em}
\setlength{\textfloatsep}{10pt} 
\begin{algorithm}[H]
        \caption{Corrector Step}
        \label{alg:corrector}
        \Indm
        \KwIn{Initialization $\widehat{x}_0$, time $t$, iterations $L$, stepsize $h$, log-likelihood $\mu_{\vy}(\cdot)$, score function $\vphi_{\theta}(\cdot)$, parametrized curve $\alpha_t$.}
        \KwOut{Sample $\widehat{x}_{L} \sim \widehat{q}_{\alpha,\vy}(\vx,t) \propto \widehat{\cev{p}}_t(\vx)\exp(-\alpha_t\mu_{\vy}(\vx))$.}
        \Indp
        Draw $\left\{\xi_{l}\right\}_{l=1}^{L}$ i.i.d. from $\gN(\boldsymbol{0},\mI_n)$\;
        \For{$l = 0$ \KwTo $L-1$}{
            $\widehat{\vx}_{l+1} \leftarrow \widehat{\vx}_{l} + h\left(\vphi_{\theta}(\widehat{\vx}_l,t)-\alpha_t\nabla_{\vx}\mu_{\vy}(\widehat{\vx}_{l})\right) + \sqrt{2h}\xi_{l+1};$
        }
\end{algorithm}

\begin{algorithm}[H]
    \caption{Approx.-Free Diffusion Posterior Sampler via ODE+Corrector (AFDPS-ODE)}
    \label{alg:edm pf-ode + corrector}
    \Indm
    \KwIn{Noisy observation $\vy$, log-likelihood $\mu_{\vy}(\cdot)$, functions $s(t)$ and $\sigma(t)$, parametrized curve $\alpha_t$, time grid $\{t_i\}_{i=0}^{K}$ with $t_0 = 0$ and $t_K = T$, thresholds $\{c_l\}_{l=1}^{K}$, score function $\vphi_\theta(\cdot,t)$, corrector iterations $n_c$, stepsize $h_c$, ensemble size $N$, initial weights $\beta^{(j)}_0 = 1$ for $j\in[N]$.}
    \KwOut{Posterior approximation $\sum_{j=1}^{N}\beta^{(j)}_T\delta_{\vx^{(j)}_T}/\sum_{j=1}^{N}\beta^{(j)}_T$.}
    \Indp
    Draw $\{\vx^{(i)}_{0}\}_{i=1}^{N}$ i.i.d. from $\widehat{q}_{\alpha,\vy}(\cdot,0)$ via Stage I samplers in Section~\ref{subsec: algo outline}\;
    \For{$k = 0$ \KwTo $K-1$}{
        \For{$j = 1$ \KwTo $N$}{
            \small$\widetilde{\vx}^{(j)}_{t_{k+1}} \leftarrow \left(1-(t_{k+1}-t_k)\frac{\dot{s}(t_k)}{s(t_k)}\right)\vx^{(j)}_{t_{k}} + (t_{k+1}-t_k)s(t_k)^2\dot{\sigma}(t_k)\sigma(t_k)\vphi_{\theta}(\vx^{(j)}_{t_k},t_k)$\;
            $\widehat{\vx}^{(j)}_{t_{k+1}} \leftarrow \text{Algorithm}~\ref{alg:corrector}\left(\widetilde{\vx}^{(j)}_{t_{k+1}},t_{k+1},n_c,h_c, \mu_{\vy}(\cdot),\vphi_\theta(\cdot,t)\right)$\;
            $\begin{aligned}
            \log\widehat{\beta}^{(j)}_{t_{k+1}} \leftarrow \log\beta^{(j)}_{t_{k+1}} &-\alpha'_{t_k}(t_{k+1}-t_k)\mu_{\vy}(\vx^{(j)}_{t_k}) +(t_{k+1}-t_k)\alpha_{t_k}\tfrac{\dot{s}(t_k)}{s(t_k)}\nabla_{\vx}\mu_{\vy}\left(\vx^{(j)}_{t_k}\right)^\intercal \vx^{(j)}_{t_k}\\
            &-(t_{k+1}-t_k)\alpha_{t_k}s(t_k)^2\dot{\sigma}(t_k)\sigma(t_k)\nabla_{\vx}\mu_{\vy}\left(\vx^{(j)}_{t_k}\right)^\intercal\vphi_{\theta}\left(\vx^{(j)}_{t_k},t_k\right);
            \end{aligned}$
        }
        {\small$\{(\vx^{(j)}_{t_{k+1}},\beta^{(j)}_{t_{k+1}})\}_{j=1}^{N} \leftarrow$ Algorithm~\ref{alg:resampling}$\left(c_{k+1},\{(\widehat{\vx}^{(j)}_{t_{k+1}},\widehat{\beta}^{(j)}_{t_{k+1}})\}_{j=1}^{N}\right)$\;}
    }
\end{algorithm}

%(The reviewer may also refer to Part 2 of the Rebuttal associated with Reviewer XoVc above to see how the noisy sequence is generated)

%%\vspace{-.45em}

%\vspace{-1.0em}
\section{Theoretical Analysis}
\label{sec: main result - theory}
%\vspace{-0.5em}
In this section, we present our theoretical results of the ensemble-based posterior samplers introduced in Section~\ref{sec: main result - algorithms}. Our analysis is conducted in continuous time, based on the weighted particle dynamics~\eqref{eqn: PDE of posterior diffusion} and~\eqref{eqn: PDE soln via ensemble of particles}. The impact of numerical discretization, as implemented in Algorithm~\ref{alg:sde} and Algorithm~\ref{alg:edm pf-ode + corrector}, is not considered here and is left for future work. Without loss of generality, we focus on the backward SDE setting~\eqref{eqn: EDM SDE}, specifically using $s(t)=1$ and $\sigma(t)=t$. We begin by introducing several technical assumptions.

\begin{assumption}[Regularity of the log-likelihood]
\label{assump: regularity and boundedness of log-likelihood}
The log-likelihood function $\mu_{\vy}$ is twice differentiable and lower bounded by some constant $C^{(1)}_{\vy}$  depending only on the observation $\vy$.     
\end{assumption}
%%\vspace{-0.5em}
\begin{assumption}[Bounded second moment]
\label{assump: bounded second moment}
The prior distribution $p_0$ satisfies a second-moment bound: $\mathbb{E}{p_0}\left[\|\vx\|_{2}^2\right] \leq m_2^2$.
\end{assumption}
%%\vspace{-0.5em}
\begin{assumption}[Score matching error]
\label{assump: score matching training error}
The neural network estimator $\vphi_{\theta}(\vx,t)$ approximates the score function $\nabla_{\vx}\log \cev{p}_t(\vx)$ with uniformly bounded error across $t \in [0,T]$:
\begin{equation}
\int_{\mathbb{R}^n}\left\|\vphi_{\theta}(\vx,t)-\nabla_{\vx}\log \cev{p}_{t}(\vx)\right\|_{2}^2\cev{p}_t(\vx)\dif \vx  \leq \epsilon_{\vs}^2.    
\end{equation}
\end{assumption}

On the one hand, Assumption~\ref{assump: regularity and boundedness of log-likelihood} ensures the absence of singularities in the log-likelihood $\mu_{\vy}$, which is a condition  adopted in existing work on BIPs~\citep{stuart2010inverse} and satisfied by common noise models such as Gaussian and Poisson (with $C^{(1)}_{\vy} = 0$). We note that this assumption also ensures the existence of the time-dependent constant $\widehat{Z}_{\alpha,\vy}(t)$ introduced at the beginning of subsection~\ref{subsec: algo outline}. Specifically, for any time $t \in [0,T]$, a direct computation yields $\widehat{Z}_{\alpha,\vy}(t) = \int \widehat{\cev{p}}_t(\vx)e^{-\alpha_t\mu_{\vy}(\vx)}\dif\vx  \leq \int \widehat{\cev{p}}_t(\vx) e^{-\alpha_t C_{\vy}^{(1)}}\dif\vx \leq \max\{1, e^{-C_{\vy}^{(1)}}\}$.
On the other hand, Assumptions~\ref{assump: bounded second moment} and~\ref{assump: score matching training error} are aligned with recent theoretical frameworks for diffusion models~\citep{wang2024evaluating,chen2022sampling,chen2023improved,benton2023linear,chen2024probability}. In particular, Assumption~\ref{assump: score matching training error} quantifies the approximation error due to neural network training and reflects the quality of the pre-trained score function. We now present our first main result, which quantifies the discrepancy between the true posterior $q_{\vy,0}$ and the distribution $\widehat{q}_{\alpha,\vy,T}$ obtained by evolving our derived PDE dynamics~\eqref{eqn: PDE of posterior diffusion} for time $T$. Below we use $\epsilon_{I}$ to denote an upper bound on the error incurred when sampling from the initial distribution $\widehat{q}_{\alpha,\vy}(\cdot,0)$, \emph{i.e.}, $\TV\left(\widehat{q}_{\alpha,\vy,0}, \widehat{q}_{\alpha,\vy}(\cdot,0)\right) \leq \epsilon_{I}$.

\begin{tcolorbox}[colback=lightgray,boxrule=0pt,sharp corners]
\begin{theorem}[Error Bound for Posterior Estimation within the Continuous-Time Dynamics]
\label{thm: bound on true and approx posterior}
When $s(t) = 1$, $\sigma(t) = t$ and Assumptions~\ref{assump: regularity and boundedness of log-likelihood}–\ref{assump: score matching training error} all hold, the total variation (TV) distance between the approximate and true posterior distributions satisfies:
\begin{equation}
\TV(\widehat{q}_{\alpha,\vy,T},q_{\vy,0}) \leq C^{(2)}_{\vy}\sqrt{\frac{m_2^2}{T^2} + T^2\epsilon_{\vs}^2} + C^{(3)}_{\vy,T}\epsilon_{I},    
\end{equation}
where $q_{\vy,0}(\vx) \propto p_0(\vx)\exp(-\mu_{\vy}(\vx))$ is the exact posterior, and $\widehat{q}_{\alpha,\vy,T}$ is the solution to the posterior evolution dynamics~\eqref{eqn: PDE of posterior diffusion} with initial distribution $\widehat{q}_{\alpha,\vy,0}$. The constant $C^{(2)}_{\vy}$ depends only on the observation $\vy$, while $C^{(3)}_{\vy,T}$ on both $\vy$ and the time $T$. In particular, when the initial error $\epsilon_{I}=0$, \emph{i.e.}, $\widehat{q}_{\alpha,\vy}(\cdot,0)$ can be sampled exactly, optimizing the right-hand side yields the asymptotic bound $\TV(\widehat{q}_{\vy,T},q_{\vy,0}) \lesssim \sqrt{\epsilon_\vs}$ when $T \asymp \sqrt{\epsilon_\vs^{-1}}$.
\end{theorem}
\end{tcolorbox}

A detailed proof of Theorem~\ref{thm: bound on true and approx posterior} can be found in Section~\ref{proof: posterior bound by prior error}. It essentially combines techniques from the theory of diffusion models~\citep{chen2023improved,wang2024evaluating} and the well-posedness theory of Bayesian inverse problems, which is closely related to~\cite[Theorem 4.1]{purohit2024posterior}. The result of Theorem~\ref{thm: bound on true and approx posterior} reveals that under Assumptions~\ref{assump: regularity and boundedness of log-likelihood}–\ref{assump: score matching training error}, the discrepancy between the posterior distributions can be upper-bounded by a term proportional to the discrepancy between the prior distributions when the initial distribution can be sampled exactly. The upper bound further captures a tradeoff between the error of the forward process and the score matching error, which is controlled by the time horizon $T$. Next, we study the particle approximation to the PDE solution~\eqref{eqn: PDE of posterior diffusion}. In particular, we examine the convergence of the dynamics of the weighted particle ensemble~\eqref{eqn: PDE soln via ensemble of particles} in the many-particle limit.

\begin{assumption}[Boundedness and Lipschitz continuity of $I$]
\label{assump: boundedness and Lipschitz of I}
Define the function
$$
I(\vx,t):=\alpha^2_t\|\nabla_{\vx}\mu_{\vy}(\vx)\|_2^2 - \alpha_t\Delta_{\vx}\mu_{\vy}(\vx)-2\alpha_t\vphi_{\theta}(\vx,t)^\intercal\nabla_{\vx}\mu_{\vy}(\vx)-\alpha'_t\mu_{\vy}(\vx).
$$
We assume that $I(\vx,t)$ is uniformly bounded and Lipschitz continuous over $\mathbb{R}^n \times [0,T]$:
$\max\{\|I\|_{L^{\infty}(\mathbb{R}^n \times [0,T])},\mathrm{Lip}(I)\} \leq B_{\vy}$,
for some constant $B_{\vy}$ depending only on $\vy$.
\end{assumption}
%\vspace{-1.0em}
We note that function $I(\vx,t)$ in Assumption~\ref{assump: boundedness and Lipschitz of I} is highly related to the function $U_{\alpha,\vy} - \alpha_t\widehat{\mH}^\intercal\nabla_{\vx}\mu_{\vy}-\alpha'_t\mu_{\vy}$ in~\eqref{eqn: PDE soln via single particle} above, which depicts the evolution of the particle weights. Specifically, for the special case when $s(t) =1$ and $\sigma(t) = t$, a direct computation yields
\begin{equation*}
\begin{aligned}
U_{\alpha,\vy} - \alpha_t\widehat{\mH}^\intercal\nabla_{\vx}\mu_{\vy} -\alpha'_y\mu_{\vy} &= \frac{2t}{2}\left(\alpha^2_t\left\|\nabla_{\vx}\mu_{\vy}\right\|_2^2 - \alpha_t\Delta_{\vx}\mu_{\vy}\right) - \frac{4t}{2}\alpha_t\vphi_\theta^\intercal\nabla_{\vx}\mu_{\vy}= tI(\vx,t).
\end{aligned}  
\end{equation*}Hence, we have that Assumption~\ref{assump: boundedness and Lipschitz of I} essentially controls the particle weights during evolution and thereby mitigates weight degeneracy, which is analogous to Assumption 1 in~\citep{domingo2020mean}. In practice, however, the particle weights are often controlled via the resampling step and ESS threshold in Algorithm~\ref{alg:resampling}. While we adopt this assumption for analytical tractability, relaxing it and developing a rigorous theory of the resampling step remain important future directions, which might link to existing theoretical studies~\citep{lu2023birth,chen2023sampling,yan2024learning} on sampling algorithms that use birth-death dynamics or Fisher–Rao gradient flow.

\begin{tcolorbox}[colback=lightgray,boxrule=0pt,sharp corners]
\begin{theorem}[Convergence in the Many-Particle Limit]
\label{thm: many-particle limit}
When $s(t) = 1$, $\sigma(t) = t$ and Assumptions~\ref{assump: regularity and boundedness of log-likelihood}–\ref{assump: boundedness and Lipschitz of I} all hold, the empirical distribution of the particle system converges to the solution of the posterior PDE. Specifically, for all $t \in [0,T]$,
$$\lim_{N \rightarrow \infty}\mathbb{E}[\gW^2_2(\gamma^{N}_{t}, \gamma_{t})] = 0,$$
where $\gamma_t$ is the law of a single weighted particle pair $(\vx_t, \beta_t)$ governed by~\eqref{eqn: PDE soln via single particle}, such that the marginal $\widehat{q}_{\alpha,\vy}(\cdot,t)$ is recovered via $P_{\beta}\gamma_t(\cdot) = \int_{\mathbb{R}}\beta\gamma_t(\cdot,\beta)\dif\beta = \widehat{q}_{\alpha,\vy}(\cdot,t)$, and $\gamma^{N}_{t}$ is the empirical measure of the $N$-particle system $\left\{\left(\vx^{(i)}_t, \beta^{(i)}_t\right)\right\}_{i=1}^N$ governed by~\eqref{eqn: PDE soln via ensemble of particles}.
% Here,
\end{theorem}
\end{tcolorbox}

Theorem~\ref{thm: many-particle limit} establishes the mean-field consistency of the weighted ensemble approximation in the 2-Wasserstein sense. Its proof, which is presented in Appendix~\ref{proof: mean-field limit}, is based on results from propagation of chaos~\citep{sznitman1991topics,lacker2018mean}. Together with Theorem~\ref{thm: bound on true and approx posterior}, our theoretical results provide both rigorous guarantee for the accuracy of the continuum posterior approximation and justification of the proposed ensemble-based implementation. 

Overall, we would like to emphasize that our theoretical analysis in this section provides a principled continuous-time framework for bounding the distributional discrepancy between the true posterior and the distribution induced by the generated samples. In particular, the bound consists of two components, where Theorem 4.1 quantifies the contribution from score-matching error and Theorem 4.2 captures the approximation error induced by using a finite number of particles. More broadly, these results also open the door to future work on consistency and stability with respect to time discretization. To the best of our knowledge, pursuing this direction will require more advanced tools from the numerical analysis of stochastic weighted particle/SMC methods. For a list of related works, we refer the readers to the end of section~\ref{sec: discussion and conclusion}.

%\vspace{-1.0em}
\section{Experiments}
\label{sec: results of numerics}
%\vspace{-0.375em}
In this section, we evaluate the empirical performance of our method on several BIPs in imaging. We emphasize that, although we develop a posterior-sampling algorithm here, our empirical evaluations primarily focus on the reconstruction quality under standard metrics, which is most relevant for practical imaging applications. This also aligns with the selection criterion adopted in our algorithm, which picks the best-performing sample among a collection of candidates. Due to space constraints, here we focus on the case of linear inverse problems in the main text. For implementation details and complete experimental results in the linear setting, as well as supplementary experiments, we refer the readers to Appendix~\ref{sec: implementation detail app}, ~\ref{sec: experiment results detail app} and~\ref{sec: supplementary experiments in app}, respectively. 
%\vspace{-0.375em}

%%\vspace{-2.5em}
%\rev{\subsection{Ground-Truth Posterior from Compressed Sensing}}
%%\vspace{-.5em}
%\rev{Here we }

%\vspace{-.5em}
\subsection{Linear Inverse Problems}
\label{subsec: linear inverse problems}
%\vspace{-.5em}
\paragraph{Problem Setting.}
We consider the following four canonical linear inverse problems: Gaussian Deblurring (GD), Motion Deblurring (MD), Super Resolution (SR), and Box Inpaint (BI). In all these tasks, we assume that the observational noise is isotropic Gaussian with variance $0.2$, \emph{i.e.}, $\vn \sim \gN(\vzero, 0.2 \mI_m)$ in~\eqref{eqn: BIP setup}, a more challenging setting compared to the commonly used low-noise scenario with variance $2.5 \times 10^{-3}~$\citep{dou2024diffusion,wu2024principled}. Experiments are conducted on FFHQ-256~\citep{karras2019style} and ImageNet-256~\citep{deng2009imagenet}, two widely used datasets in imaging and vision.

%\vspace{-.5em}
\paragraph{Baselines.}

We compare our proposed algorithms with several state-of-the-art diffusion model-based posterior sampling methods:
\begin{itemize}[leftmargin=1.5em, itemsep=0pt, topsep=0pt, parsep=0pt]
    \item \emph{DPS~\citep{chung2022diffusion}}: a sampler that guides the pretrained DM with approximations of manifold-constrained gradients derived from the measurement likelihood.
    \item \emph{DCDP~\citep{li2024decoupled}}: a framework alternating between optimization steps that ensure data consistency and pretrained DMs for posterior sampling.
    \item \emph{SGS-EDM~\citep{wu2024principled}}: a split Gibbs sampler coupled with a DM for efficient posterior inference.
    \item \emph{FK-Corrector~\citep{skreta2025feynman}}: an SMC-based sampler using Feynman-Kac formula to correct trajectories.
    \item \emph{PF-SMC-DM~\citep{dou2024diffusion}}: a particle filtering framework combining SMC with diffusion models.
\end{itemize}

\begin{wrapfigure}{r}{0.40\textwidth}
    %\vspace{-1.5em}
    \centering
    % First row
    \begin{subfigure}[b]{0.10\textwidth}
        \includegraphics[width=\linewidth]{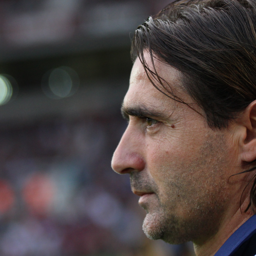}
    \end{subfigure}%
    \begin{subfigure}[b]{0.10\textwidth}
        \includegraphics[width=\linewidth]{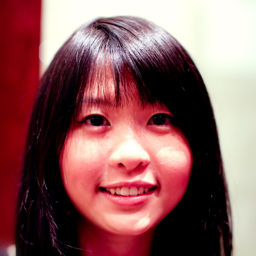}
    \end{subfigure}%
    \begin{subfigure}[b]{0.10\textwidth}
        \includegraphics[width=\linewidth]{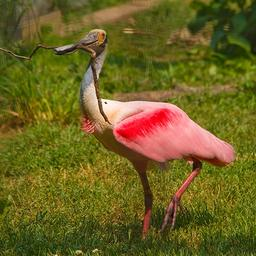}
    \end{subfigure}%
    \begin{subfigure}[b]{0.10\textwidth}
        \includegraphics[width=\linewidth]{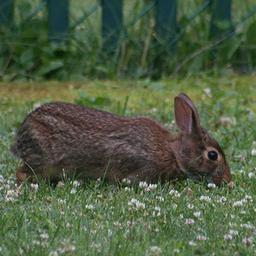}
    \end{subfigure}
    
    %\vspace{-0.1\baselineskip}
    
    \begin{subfigure}[b]{0.10\textwidth}
        \includegraphics[width=\linewidth]{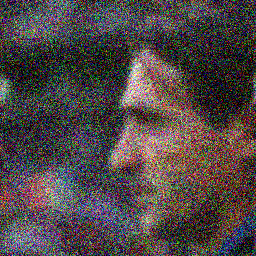}
    \end{subfigure}%
    \begin{subfigure}[b]{0.10\textwidth}
        \includegraphics[width=\linewidth]{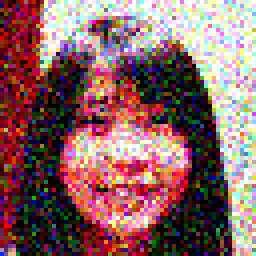}
    \end{subfigure}%
    \begin{subfigure}[b]{0.10\textwidth}
        \includegraphics[width=\linewidth]{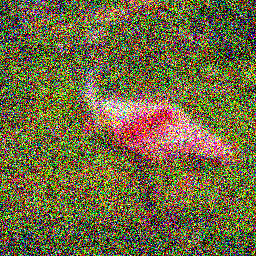}
    \end{subfigure}%
    \begin{subfigure}[b]{0.10\textwidth}
        \includegraphics[width=\linewidth]{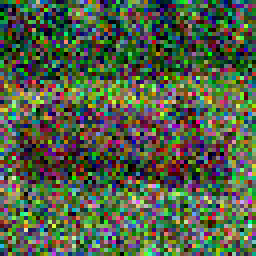}
    \end{subfigure}
    
    %\vspace{-0.1\baselineskip}
    
    \begin{subfigure}[b]{0.10\textwidth}
        \includegraphics[width=\linewidth]{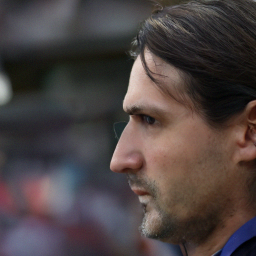}
    \end{subfigure}%
    \begin{subfigure}[b]{0.10\textwidth}
        \includegraphics[width=\linewidth]{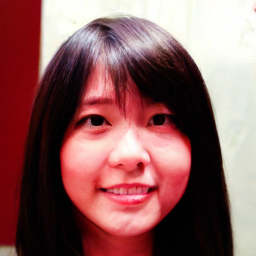}
    \end{subfigure}%
    \begin{subfigure}[b]{0.10\textwidth}
        \includegraphics[width=\linewidth]{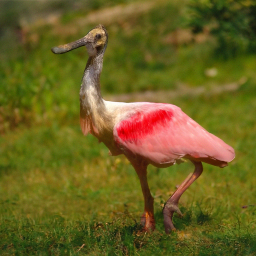}
    \end{subfigure}%
    \begin{subfigure}[b]{0.10\textwidth}
        \includegraphics[width=\linewidth]{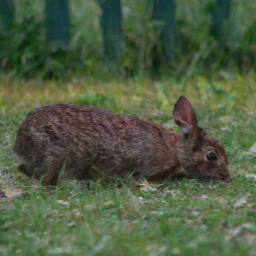}
    \end{subfigure}
    \caption{Visualization of posterior samples by AFDPS on linear inverse problems. \textbf{Upper:} Original; \textbf{Middle:} Noisy (Measurement); \textbf{Lower:} Reconstructed.}
    \label{fig:2x4compact}
    %\vspace{-1.5em}
\end{wrapfigure}

%%\vspace{-.5em}
\paragraph{Experimental Settings.}

To ensure a fair comparison, we use the same checkpoints for the two pre-trained score functions provided in~\citep{chung2022diffusion} and fix the number of function evaluations (NFE) to $2 \times 10^4$ across all methods. For ensemble-based approaches, the number of particles is set to $N = 10$. In the case of AFDPS-ODE (Algorithm~\ref{alg:edm pf-ode + corrector}), we reduce the number of particles to $N = 5$ to offset the additional computational cost from the corrector step, while maintaining the total NFE consistent with AFDPS-SDE (Algorithm~\ref{alg:sde}). Within each ensemble, the particle with the largest weight is returned as the final estimator of the recovered image, which is analogous to the best-of-$N$ strategy. Here we choose the parametrized curve $\alpha_t$ to be constant, \emph{i.e.}, $\alpha_t \equiv 1$. We evaluate reconstruction quality using two metrics: PSNR (Peak Signal-to-Noise Ratio), which quantifies pixel-level accuracy, and LPIPS (Learned Perceptual Image Patch Similarity)~\citep{zhang2018unreasonable}, which measures perceptual similarity. Both metrics are computed between the selected reconstruction sample and the ground truth image over a set of 100 validation images.

\begin{table}[htbp]
\centering
\caption{Results on 4 linear inverse problems for 100 validation images from FFHQ-256.}
\resizebox{\textwidth}{!}{
\begin{tabular}{l|cc|cc|cc|cc}
\noalign{\hrule height 1.2pt} 
\multirow{2}{*}{\bfseries Method} 
& \multicolumn{2}{c|}{\bfseries Gaussian Deblurring}
& \multicolumn{2}{c|}{\bfseries Motion Deblurring}
& \multicolumn{2}{c|}{\bfseries Super Resolution} 
& \multicolumn{2}{c}{\bfseries Box Inpainting}\\
& PSNR ($\uparrow$)  & LPIPS ($\downarrow$) 
& PSNR ($\uparrow$)  & LPIPS ($\downarrow$) 
& PSNR ($\uparrow$)  & LPIPS ($\downarrow$) 
& PSNR ($\uparrow$)  & LPIPS ($\downarrow$)\\
\noalign{\hrule height 1pt}
\rowcolor{lightgray} DPS \citep{chung2022diffusion} & 22.57  & 0.2976  & 21.00 & 0.3280 & 19.09 & 0.5627 & 21.57  & 0.3245  \\
DCDP \citep{li2024decoupled} & 24.77 & 0.2868 & 21.57 & 0.3487 & 21.23 & 0.5139  & 22.05 & 0.4525 \\
\rowcolor{lightgray} SGS-EDM \citep{wu2024principled} & 24.78  & 0.2776  &  23.45 & 0.3009 & 22.41  & 0.3225  & 23.69 & 0.2301 \\
FK-Corrector \citep{skreta2025feynman} & 21.22  & 0.4023 &  20.51 & 0.4275 & 20.67 & 0.4133 & 16.97 & 0.5490 \\
\rowcolor{lightgray} PF-SMC-DM \citep{dou2024diffusion} & 23.00  & 0.3940 & \cellcolor{bestperformance}\bfseries 26.59 & 0.3435 & 18.92 & 0.5049 & 25.54 & 0.3391\\
\noalign{\hrule height 1pt} 
\bfseries AFDPS-SDE (Alg.~\ref{alg:sde}) & 24.83 & 0.2580 &  23.58  & \cellcolor{bestperformance}\bfseries 0.2869  & \cellcolor{bestperformance}\bfseries 22.96 & \cellcolor{bestperformance}\bfseries 0.3063 &  25.45 & 0.2084 \\
\rowcolor{lightgray} \bfseries AFDPS-ODE (Alg.~\ref{alg:edm pf-ode + corrector}) & \cellcolor{bestperformance}\bfseries 24.98 & \cellcolor{bestperformance}\bfseries 0.2560 & 23.52 & 0.2905  & 21.47 & 0.3345 & \cellcolor{bestperformance}\bfseries 25.73 & \cellcolor{bestperformance}\bfseries 0.1969\\
\noalign{\hrule height 1.2pt} 
\end{tabular}
}
\label{tab:ffhq result}
\end{table}

%\vspace{-0.5em}
\begin{table}[htbp]
\centering
\caption{Results on 4 linear inverse problems for 100 validation images from ImageNet-256.}
\resizebox{\textwidth}{!}{
\begin{tabular}{l|cc|cc|cc|cc}
\noalign{\hrule height 1.2pt} 
\multirow{2}{*}{\bfseries Method} 
& \multicolumn{2}{c|}{\bfseries Gaussian Deblurring}
& \multicolumn{2}{c|}{\bfseries Motion Deblurring}
& \multicolumn{2}{c|}{\bfseries Super Resolution} 
& \multicolumn{2}{c}{\bfseries Box Inpainting}\\
& PSNR ($\uparrow$)  & LPIPS ($\downarrow$) 
& PSNR ($\uparrow$)  & LPIPS ($\downarrow$) 
& PSNR ($\uparrow$)  & LPIPS ($\downarrow$) 
& PSNR ($\uparrow$)  & LPIPS ($\downarrow$)\\
\noalign{\hrule height 1pt}
\rowcolor{lightgray} DPS \citep{chung2022diffusion} & 20.60  & 0.4351  & 20.46 & 0.5328 & 19.17 & 0.4940 & 22.70 & 0.3765 \\
DCDP \citep{li2024decoupled} & 22.34 & 0.4821 & 20.59 & 0.5338 & 20.26 & 0.5597 & 21.67 & 0.4344 \\
\rowcolor{lightgray} SGS-EDM \citep{wu2024principled} & 19.31 & 0.4807  & 20.54 & 0.4653 & 19.61 & 0.4986 & 21.42 & 0.4643 \\
FK-Corrector \citep{skreta2025feynman} & 18.39 & 0.5973 & 18.34  & 0.6022  & 18.57 & 0.5887 & 16.28 & 0.7132 \\
\rowcolor{lightgray} PF-SMC-DM \citep{dou2024diffusion} & 20.06  & 0.5927 & \cellcolor{bestperformance}\bfseries23.91  & \cellcolor{bestperformance} \bfseries 0.4195  & 18.42 & 0.6462 & 21.34 & 0.4195\\
\noalign{\hrule height 1pt}
\bfseries AFDPS-SDE (Alg.~\ref{alg:sde}) & 22.38 & \cellcolor{bestperformance}\bfseries 0.3925 &   19.46 & 0.4936  & \cellcolor{bestperformance}\bfseries 20.97 & \cellcolor{bestperformance}\bfseries 0.4643 & \cellcolor{bestperformance}\bfseries 23.15 & 0.3051 \\
\rowcolor{lightgray} \bfseries AFDPS-ODE (Alg.~\ref{alg:edm pf-ode + corrector}) & \cellcolor{bestperformance}\bfseries 22.42  & 0.4633  & 21.54   & 0.4944   & 19.60 & 0.5634 & 22.76 & \cellcolor{bestperformance}\bfseries 0.2716 \\
\noalign{\hrule height 1.2pt} 
\end{tabular}
}
\label{tab:imagenet result}
\end{table}

%\vspace{-0.5em}
\paragraph{Results.}
The quantitative performance of our proposed methods - AFDPS-SDE (Algorithm~\ref{alg:sde}) and AFDPS-ODE (Algorithm~\ref{alg:edm pf-ode + corrector}) - is presented in Table~\ref{tab:ffhq result} for the FFHQ-256 dataset and Table~\ref{tab:imagenet result} for the ImageNet-256 dataset. On FFHQ-256, both methods consistently demonstrate strong or highly competitive results across all evaluated inverse problems, frequently outperforming existing baselines in terms of both PSNR and LPIPS. The two variants show complementary strengths across different tasks, underscoring the benefit of incorporating both formulations. Similar trends are observed on the more diverse ImageNet-256 dataset, where both AFDPS methods continue to achieve robust and often superior performance. Qualitative examples are provided in Figure~\ref{fig:2x4compact} and more in Appendix~\ref{sec: experiment results detail app}, illustrating the visual quality of reconstructions across tasks with comparisons to baselines.

%\rev{\subsection{Nonlinear Inverse Problems}}

%%\vspace{-1.0em}
\section{Discussion and Conclusion}
\label{sec: discussion and conclusion}
%%\vspace{-1.0em}

In this paper, we introduced a new method for solving Bayesian inverse problems using diffusion models as the prior. Our method derives a novel PDE that exactly characterizes the exact posterior dynamics under an evolving diffusion prior, avoiding the heuristic approximations employed by previous methods and leading to better SMC-type algorithms in practice.
Theoretically, we provide the error bounds of the posterior sampling algorithm in terms of the score function error, and justify the convergence of the ensemble method in the many-particle limit. Empirically, our method outperforms state-of-the-art diffusion-based solvers across a range of computational imaging tasks.

This work opens several promising directions for future research. One straight forward extension is to apply our proposed methods to other inverse problems arising in various fields with twice-differentiable log-likelihoods, including optics, medical imaging, video analytics, geoscience, astronomy, fluid dynamics, chemistry and biology~\citep{sun2024provable,wu2024principled,zheng2025inversebench,jaganathan2016phase,fienup1982phase,candes2015wirtinger,candes2015phase,kantas2014sequential,daras2024warped,zhang2025step,jing2024alphafold,maddipatla2025inverse,sridharan2022deep,hu2024accurate}.

Methodologically, our framework could also be extended to settings such as multi-marginal sampling~\citep{albergo2023multimarginal,lindsey2025mne}, conditional sampling~\citep{zhu2023conditional}, reward-guided sampling~\citep{uehara2025reward}, and other variants of DMs, such as latent diffusion models (LDMs)~\citep{rombach2022high,song2023solvinglatent}, discrete diffusion models~\citep{murata2024g2d2, luan2025ddps,chu2025split,
austin2021structured, hoogeboom2021autoregressive, hoogeboom2021argmax, meng2022concrete, sun2022score, richemond2022categorical, lou2023discrete, floto2023diffusion, santos2023blackout, chen2024convergence, ren2024discrete}, flow matching~\citep{zhang2024flow}, or to the general framework of denoising Markov model with variants like generator matching~\citep{benton2024denoising,holderrieth2024generator,ren2025unified}. Essentially speaking, Let $\cev{p}_T$ denote some high-dimensional prior parameterized by a diffusion model (e.g., a discrete or latent DM), with PDE dynamics $\mathcal{P}$ that maps a tractable base distribution $\cev{p}_0$ to $\cev{p}_T$. An analogous task is to sample from the tilted/posterior distribution $q_T \propto \cev{p}_T e^{\mu}$, where $\mu$ is some log-likelihood term or reward function. By defining $q_t \propto \cev{p}_t e^{\mu}$ for $t\in[0,T]$ and deriving the corresponding PDE $\mathcal{Q}$ governing $(q_t)_{t\in[0,T]}$ based on $\mathcal{P}$, an algorithm of similar type can be developed by simulating $\mathcal{Q}$ with a weighted-particle (SMC-type) method to obtain samples from $q_T$. 

Theoretically, further work could explore numerical analysis of our method~\citep{chen2022sampling,chen2023improved,chen2024probability} or incorporate it with faster inference methods like parallel sampling \citep{shih2024parallel, tang2024accelerating, cao2024deep, selvam2024self, chen2024accelerating, gupta2024faster}, high-order solvers~\citep{karras2022elucidating,lu2022dpm++,liu2022pseudo,lu2022dpm, zheng2023dpm, li2024accelerating,wu2024stochastic,ren2025fast} and their variants. In addition, it will also be of interest to investigate how one may theoretically refine the dependency of the asymptotic bound on the score matching error $\epsilon_s$ in Theorem~\ref{thm: bound on true and approx posterior}.

\section{Acknowledgments and Disclosure of Funding}
Part of this research was conducted during Haoxuan Chen's internship at NEC Labs America. Haoxuan Chen, Yinuo Ren and Lexing Ying also acknowledge support of the National Science Foundation (NSF) under Award No. DMS-2208163. Haoxuan Chen would also like to thank Prof. Yifan Chen, Prof. Yiping Lu, Dr. Chenguang Duan and Qizheng Zhang for their generous help and valuable feedback. 

\newpage
\bibliography{reference}
\bibliographystyle{tmlr}

\appendix

\section{Further Discussion on Related Work and Notations}

In this section, we provide additional discussion and context around our work through a comprehensive literature review and clarification of notations used throughout the paper.

\subsection{Related Work}
\label{app: related work review}
In this subsection, we provide a more comprehensive overview of related work.

\paragraph{Solving Inverse Problems via Machine Learning Techniques}

A wide body of work has tried applying machine learning (ML) based techniques to tackle inverse problems. In particular, one class of such ML-based methods deploy the Maximum a posteriori (MAP) approach by directly modeling the inverse mapping via some neural network. In the context of physical sciences, examples of work include~\citep{yoon2018analytic, khoo2019switchnet, fan2019solving1, fan2019solving2, fan2020solving, fournier2020artificial, sun2020extrapolated, sun2021deep, li2021accurate, li2022wide, zhou2023neural, fan2023solving,molinaro2023neural,melia2025multi}. For a more detailed overview of methods belonging to such class, one may refer to~\citep{arridge2019solving, ying2022solving}. Similar methodologies~\citep{zhang2018ista,gilton2019neumann,xiang2021fista} have also been applied to inverse problems in computational imaging and computer vision. The second class of ML-based methods~\citep{hou2019solving, zhang2021multiscale, whang2021solving, park2024solving, tao2025map,dasgupta2025unifying}, however, employ a Bayesian approach by leveraging generative priors like normalizing flows and diffusion models. Such methods have been widely applied in various areas like medical imaging~\citep{song2021solving, chung2022score, tuscore2025}, cryo-electron microscopy~\citep{kreis2022latent, levy2024solving}, PDE-constrained inverse problems~\citep{jiang2025ode}, sampling marginal densities~\citep{lindsey2025mne}, inverse scattering~\citep{zhang2024back}, traveltime tomography~\citep{cao2024subspace}, nonlinear data assimilation~\citep{ding2024nonlinear}, inverse protein folding~\citep{hsu2022learning, zhu2024bridge}, as well as fluid dynamics~\citep{chen2024probabilistic, xu2025diffusion,molinaro2024generative}. For a complete review of applying diffusion models to solve inverse problems, one may refer to~\citep{daras2024survey}. Moreover, for the second class of methods that deploy a posterior sampling approach, recent work have also tried to combine diffusion models with existing sampling methods like SMC~\citep{wu2023practical,cardoso2023monte, dou2024diffusion,albergo2024nets,chen2024sequential,vargas2023transport}, SGS~\citep{xu2024provably,wu2024principled,wang2025ensemble}, parallel tempering~\citep{zhang2025generalised} and ensemble Kalman filtering~\citep{zheng2024ensemble}. For methods using gradients of the log-likelihood in their algorithm design, we note that they also relate to guidance-based methods~\citep{dhariwal2021diffusion,wu2024theoretical,chidambaram2024does,ho2022classifier,bansal2023universal,song2023loss,he2023manifold,guo2024gradient,ye2024tfg} proposed for conditional sampling.

\paragraph{Gradient Flows for Sampling and Generative Modeling}
Gradient flow perspectives, particularly those based on the Wasserstein metric with foundational insights stemming from optimal transport and the JKO scheme~\citep{jordan1998variational}, have been extensively studied for both sampling and variational inference. Recent works in this direction include~\citep{gao2019deep, ansari2020refining, fan2021variational,lambert2022variational, diao2023forward}, with ongoing developments like~\citep{wild2023rigorous, shaul2023kinetic, zhang2023mean, cheng2024particle, yao2024minimizing, choi2024scalable, zhu2024neural}. A few other work~\citep{vidal2023taming, cheng2024convergence, xu2024normalizing, xie2025flow, boffi2024flow, kassraie2024progressive} also discuss algorithms formulated via proximal operators and local-map learning strategies. Related developments in quantum Monte Carlo (QMC), particularly diffusion Monte Carlo (DMC)~\citep{caffarel1988development1, caffarel1988development2}, are reviewed in~\citep{gubernatis2016quantum, becca2017quantum} with further applications to quantum many-body problems discussed in~\citep{lu2020full}. 

\paragraph{(Stochastic) Weighted Particle Methods and Wasserstein-Fisher-Rao Dynamics}

Weighted particle methods, such as those based on the birth-death process and Wasserstein–Fisher–Rao (WFR) distances~\citep{kondratyev2016new,liero2018optimal,chizat2018interpolating}has motivated a series of studies on ensemble-based sampling dynamics~\citep{lindsey2022ensemble, lu2019accelerating, maurais2024sampling, gabrie2022adaptive, tan2023accelerate, chen2024ensemble, pathiraja2024connections, ren2025driftlite} that have been applied to solving  high-dimensional Bayesian inverse problems~\citep{qu2024uncertainty,chen2024efficient} and PDEs~\citep{han2020solving,zhang2024sequential, neklyudov2024wasserstein,chen2024teng}. These techniques have also been applied to multi-objective optimization~\citep{ren2024multi}, density estimation via Gaussian mixtures~\citep{chen2023sampling, yan2024learning}, and reinforcement learning and MDPs~\citep{muller2024fisher}. Their connection to min-max optimization is explored in~\citep{domingo2020mean, ying2022lyapunov, lascu2024fisher}. 

%Moreover, classical SMC methods in filtering settings remain relevant~\citep{liu2001monte,chopin2002sequential,del2006sequential,doucet2009tutorial,del2013mean,moral2004feynman}, with current research also addressing the continuous-time and gradient-flow based formulation~\citep{lu2024guidance}, with connections to Feynman-Kac frameworks reviewed in~\citep{moral2004feynman}. 

\subsection{Notations}
We use $\nabla_{\vx}, \nabla_{\vx} \cdot$ and $\Delta_{\vx}$ to denote the gradient, divergence, and Laplacian operators with respect to any fixed variable $\vx$. The set of positive real numbers is denoted by $\mathbb{R}^{+}$. We further use $\delta$ for the Dirac delta function. For measuring distances between probability distributions, we use the Kullback-Leibler (KL) divergence $\KL$, Total Variation (TV) divergence $\TV$, and Wasserstein-$p$ distance $\gW_p$. The $l_2$ norm is denoted by $\|\cdot\|_2^2$. In general, for any real number $p \in \mathbb{R}^+ \cup \{\infty\}$, we use $\|\cdot\|_{L^p}$ to denote the $L^p$ norm. 

%We begin by recalling the notations used in the main text and introducing a few new notations. Let $\vy_0 := \vy = A\vx_0+\vn$ denote the forward model, where $\vx_0 \in \mathbb{R}^n, \vy=\vy_0 \in \mathbb{R}^m$ and $\vn \in \mathbb{R}^m$ denote the input, measurement, and Gaussian observational noise respectively, with $\vn \sim \mathcal{N}(0,\sigma^2I_m)$. Here we consider a special case of the EDM framework, whose corresponding forward process is given by $\vx_t = \vx_0 + \gamma(t)\vw$ denote the forward diffusion process with injected noise $w \sim \mathcal{N}(0,I_n)$. The corresponding forward noising noising process for the measurements is further denoted by $Y_t = y_0 + \gamma(t)\eta$, where $\eta = Aw$ is a transformed multivariate Gaussian distribution in $\mathbb{R}^m$. 

\subsection{Further Discussion on Related Work}
\subsubsection{Intuitive Explanations of the Posterior Score Approximations}
\label{subsubsec: intuition behind posterior score appoximations}
This subsection is devoted to describing the intuitions behind the approximations adopted in existing work like ILVR~\citep{choi2021ilvr} and DPS~\citep{chung2022diffusion}. We start with deriving the approximation used in ILVR first. In fact, the approximation in ILVR can be understood as a preconditioned version of the approximation used in~\citep{song2020score}, which we will derive and explain first. Based on the Bayes' formula $p_{t}(\vx_t|\vy) \propto p_{t}(\vx_t)p_t(\vy|\vx_t)$, which implies $\nabla_{\vx_t}\log p(\vy| \vx_t) = \nabla_{\vx_t}\log p(\vx_t|\vy_0) - \nabla_{\vx_t}\log p(\vx_t)$, we have that it suffices to derive an approximation of the term $\nabla_{\vx_t}\log p(\vx_t|\vy_0)$. Following the derivation in Appendix I.4 of~\citep{song2020score}, we have that $\vy_t$ is almost the same as $\vy_0$ when $t$ is small, which gives us the approximation $p(\vx_t|\vy_t,\vy_0) \approx p(\vx_t | \vy_t)$. Moreover, when $t$ is relatively large, we have that $\vx_t$ is almost the same as Gaussian noise, which is away from $\vy_0$ and again implies $p(\vx_t|\vy_t,\vy_0) \approx p(\vx_t | \vy_t)$. Combining the two cases further yields the following approximation:
$$p(\vx_t|\vy_0) = \int p(\vx_t|\vy_t,\vy_0)p(\vy_t|\vy_0)\dif\vy_t \approx \int p(\vx_t|\vy_t)p(\vy_t|\vy_0)\dif\vy_t \approx p\left(\vx_t|\hat{\vy}_t\right).$$
where $\hat{\vy}_t$ above denotes a realization of $\vy_t$, \emph{i.e.}, $\hat{\vy}_t$ can be treated as being sampled from $p(\vy_t|\vy_0)$. Then we can plug in the approximation $p(\vx_t|\vy_0) \approx p(\vx_t|\hat{\vy}_t)$ derived above and apply Bayes' formula again, which indicate
$$\nabla_{\vx_t}\log p(\vy|\vx_t) \approx \nabla_{\vx_t}\log p(\vx_t|\hat{\vy}_t) - \nabla_{x_t}\log p(x_t) = \nabla_{\vx_t}\log \frac{p(\vx_t|\hat{\vy}_t)}{p(\vx_t)} = \nabla_{x_t}\log p(\hat{\vy}_t|x_t)$$
Furthermore, from our definition of $\vy_t$ above we have that 
$$\vy_t - A\vx_t = \vy_0 + \sigma(t)\eta - A(\vx_0 +\sigma(t)\vw) = \vy_0 - A\vx_0 = \vn \sim \mathcal{N}(0,\sigma^2I_m) \Rightarrow p(\hat{\vy}_t|x_t) \propto e^{-\frac{1}{2\kappa^2}\left\|\hat{\vy}_t - Ax_t\right\|^2}$$
Then we can further deduce that
$$\nabla_{\vx_t}\log p(\vy|\vx_t) \approx \nabla_{\vx_t}\log p(\hat{\vy}_t|\vx_t) = -\frac{1}{2\kappa^2}\nabla_{\vx_t}\left\|\hat{\vy}_t-A\vx_t\right\|^2 = \frac{1}{\kappa^2}A^T(\hat{\vy}_t-A\vx_t)$$
which is exactly the approximation adopted in~\citep{song2020score}. The ILVR approximation can then be interpreted as a preconditioned version of the approximation in~\citep{song2020score} derived above, where $(A^TA)^{-1}$ is the preconditioner.

Then we proceed to derive the DPS-based approximation used in~\citep{chung2022diffusion}, which is mainly based on the following factorization:
\begin{equation*}
\begin{aligned}
p(\vy|\vx_t) &= \int p(\vy, \vx_0|\vx_t)\dif\vx_0  = \int p(\vy|\vx_0,\vx_t)p(\vx_0|\vx_t)\dif\vx_0 \\
&= \int p(\vy|\vx_0)p(\vx_0|\vx_t)\dif\vx_0 = \mathbb{E}_{\vz \sim p(\vx_0|\vx_t)}[p(\vy|\vz)] \approx  p\left(\vy | \hat{\vx}_0(t)\right)
\end{aligned}    
\end{equation*}
where $\hat{\vx}_0(t) = \mathbb{E}_{\vz \sim p(\vx_0|\vx_t)}[\vz] = \mathbb{E}[\vx_0|\vx_t]$ denotes the conditional expectation of $\vx_0$ given $\vx_t$. Specifically, we note that the second equation above follows from Bayes' formula and the third equation above is derived based on the fact that $\vy$ is conditionally independent of $\vx_t$ given $\vx_0$. Moreover, the last equation above follows from the fact that the expectation of a function can be approximated via its evaluation at the expectation, \emph{i.e.}, $\mathbb{E}[f(Z)] \approx f(\mathbb{E}[Z])$ for any random variable $Z$. Furthermore, applying Tweedie's formula to $\vx_t = \vx_0 + \sigma(t)\vw$ yields that
$$\hat{\vx}_0(t) = \mathbb{E}[\vx_0|\vx_t] = \vx_t + \sigma(t)^2\nabla_{\vx_t}\log p_t(\vx_t)$$
Then we can apply the operator $\nabla_{\vx_t}\log(\cdot)$ on both sides of the approximation $p(\vy|\vx_t) \approx  p\left(\vy | \hat{\vx}_0(t)\right)$ above to deduce that
\begin{equation*}
\begin{aligned}
\nabla_{\vx_t}\log p(\vy|\vx_t) &\approx \nabla_{\vx_t}\log p(\vy|\hat{\vx}_0(t)) = -\frac{1}{2\kappa^2}\nabla_{\vx_t}\left\|\vy - A\hat{\vx}_0(t)\right\|^2\\
&= \frac{1}{\kappa^2}(\nabla_{\vx_t}(A\hat{\vx}_0(t)))^T(\vy-A\hat{\vx}_0(t)) = \frac{1}{\kappa^2}(\nabla_{\vx_t}\hat{\vx}_0(t))^TA^T(\vy-A\hat{\vx}_0(t))
\end{aligned}    
\end{equation*}
Plugging in the expression of $\hat{\vx}_0(t)$ then gives us the DPS approximation of the posterior score function:
$$\nabla_{\vx_t}\log p(\vy|\vx_t) \approx \frac{1}{\kappa^2}(\nabla_{\vx_t}\hat{\vx}_0(t))^TA^T(\vy-A\hat{\vx}_0(t)) = \frac{1}{\kappa^2}(\mI_n + \sigma(t)^2\nabla^2_{\vx_t}\log p_t(\vx_t))^TA^T(\vy-A\mathbb{E}[\vx_0|\vx_t])$$
For a complete derivation of the ILVR and DPS approximations described above, as well as a list of all related approximations, we refer the readers to Figure 1 and Subsection 3.1 of~\citep{daras2024survey}.
\subsubsection{Comparison with SMC-based methods} 
\label{subsubsec: comparison with other SMC methods}
Below we discuss the key novelties of our method compared to a class of existing works that solve inverse problems by combining SMC with diffusion-based priors, which can be summarized from the following two aspects. The first aspect is our derivation from a continuous-time perspective, which yields a PDE-based formulation of the algorithm. In contrast, prior approaches such as~\citep{wu2023practical,cardoso2023monte,dou2024diffusion} are typically defined via proposal kernels based on the discrete-time formulation. Our PDE-based viewpoint may facilitate future work on designing more efficient numerical schemes and extending the framework to incorporating with more complicated diffusion processes for the prior evolution, such as continuous-time Markov chains or Lévy processes. A more detailed discussion on these future directions is postponed to the end of section~\ref{sec: discussion and conclusion}. The second aspect concerns the choice of likelihood functions, which determines the induced time-dependent distribution paths. Here we expand on the choices used in the representative examples~\citep{wu2023practical,cardoso2023monte,dou2024diffusion}. Specifically,~\citep{cardoso2023monte,dou2024diffusion} adopted the setting of particle filtering by generating a corresponding noisy observation $\vy_t$ based on the observed $\vy$ for any $\vx_t$, which in turn leads to the following distribution path $l^{(1)}_t(\vx) \propto \widehat{\cev{p}}_t(\vx)e^{-\mu_{\vy_t}(\vx)}$. For a detailed description of how $\vy_t$ is generated, the readers may refer to Appendix~\ref{subsubsec: intuition behind posterior score appoximations} above. Furthermore,~\citep{wu2023practical} considers a sequence of evolving posterior distributions of the form below:
\begin{equation*}
\begin{aligned}
l^{(2)}_t(\vx) &\propto \widehat{\cev{p}}_t(\vx)p(\vy|\vx_t=\vx) = \widehat{\cev{p}}_t(\vx)\left(\int p(\vy|\vx_0=\vz)p(\vx_0=\vz|\vx_t=\vx)\dif\vz\right)\\
&\propto \widehat{\cev{p}}_t(\vx)\left(\int p(\vx_0=\vz|\vx_t=\vx)e^{-\frac{1}{2\sigma^2}\left\|\vy-\mathcal{A}(\vz)\right\|^2}\dif\vz\right).   
\end{aligned}    
\end{equation*}
where Tweedie's formula is further applied to compute $p(\vx_0|\vx_t)$ and approximate the integral above. From the formulas of $l_t$'s listed above, we have that the two kinds of distribution paths used in~\citep{wu2023practical,cardoso2023monte,dou2024diffusion} are all different from $\hat{q}_{\alpha,\vy}(\vx,t) \propto \widehat{\cev{p}}_t(\vx)e^{-\alpha_t\mu_{\vy}(\vx)}$ considered in our work. Empirically, numerical evidence are also provided in section~\ref{sec: results of numerics} to demonstrate the advantages of our method.

\section{Supplementary Proofs and Justifications for Section~\ref{sec: main result - algorithms}}
\label{sec: derivations for Sec 3}
In this section, we provide detailed proofs and justifications for claims listed in Section~\ref{sec: main result - algorithms}. We will use the shorthand notation $f_{\alpha,\vy}(\vx) = \exp(\alpha_t \mu_{\vy}(\vx))$ for the time-dependent likelihood factor. 

\begin{lemma}
    \label{lem: simplifying dynamics}    
    The PDE dynamics governing the evolution of the unnormalized posterior distribution $\widehat{Q}_{\alpha,\vy}(\vx,t):\mathbb{R}^n \times [0, T] \rightarrow \mathbb{R}^{+}$ is given by
    \begin{equation}
    \label{eqn: PDE of unnormalized posterior diffusion app}
        \begin{aligned}
        \frac{\partial}{\partial t}\widehat{Q}_{\alpha,\vy} = &-\nabla_{\vx} \cdot \left(\left(\widehat{\mH}(\vx,t)-\alpha_tV(t)^2\nabla_{\vx}\mu_{\vy}\right)\widehat{Q}_{\alpha,\vy}\right)
        +\tfrac{1}{2}V(t)^2\Delta_{x}\widehat{Q}_{\alpha,\vy}\\
        &+\left(\frac{1}{2}V(t)^2\left(\alpha^2_t\|\nabla_{\vx}\mu_{\vy}\|_2^2-\alpha_t\Delta_{\vx}\mu_{\vy}\right) - \alpha_t\widehat{\mH}(\vx,t)^\intercal\nabla_{\vx}\mu_{\vy}-\alpha'_t\mu_{\vy}\right)\widehat{Q}_{\alpha,\vy},
        \end{aligned}    
    \end{equation}
    where 
    \begin{equation}
    \label{eqn: original drift with score}
    \widehat{\mH}(\vx,t) := -F(t)\vx + \frac{G(t)^2 + V(t)^2}{2} \vphi_\theta(\vx,t)    
    \end{equation}
    denotes the original drift in the prior diffusion.
\end{lemma}

\begin{proof}
    We begin by rewriting the PDE dynamics that needs simplification:
    \begin{equation}
    \begin{aligned}
    \left(\frac{\partial}{\partial t}\widehat{Q}_{\alpha,\vy}\right)f_{\alpha,\vy} + \left(\alpha'_t\mu_{\vy}f_{\alpha,\vy}\right)\widehat{Q}_{\alpha,\vy} &= \frac{\partial}{\partial t}\left(\widehat{Q}_{\alpha,\vy}f_{\alpha,\vy}\right)\\
    &= -\nabla_{\vx} \cdot \left(\widehat{\mH}(\vx,t) \widehat{Q}_{\alpha,\vy}f_{\alpha,\vy}\right) + \frac{1}{2}V(t)^2\Delta_{x}(\widehat{Q}_{\alpha,\vy}f_{\alpha,\vy}),   
    \end{aligned}    
    \end{equation}
    where the first identity above follows from the product rule and the second identity above is derived by substituting $\widehat{\cev{p}}_t = \widehat{Q}_{\alpha,\vy}f_{\alpha,\vy}$ into the Fokker-Planck PDE~\eqref{eqn: PDE of prior diffusion} satisfied by $\widehat{\cev{p}}_t$. Rearranging the equation above further implies
    \begin{equation}
    \begin{aligned}
    \frac{\partial}{\partial t}\widehat{Q}_{\alpha,\vy} = -\frac{1}{f_{\alpha,\vy}}\nabla_{\vx} \cdot \left(\widehat{\mH}(\vx,t) \widehat{Q}_{\alpha,\vy}f_{\alpha,\vy}\right) + \frac{1}{2f_{\alpha,\vy}}V(t)^2\Delta_{x}(\widehat{Q}_{\alpha,\vy}f_{\alpha,\vy}) - \alpha'_t\mu_{\vy}\widehat{Q}_{\alpha,\vy}.    
    \end{aligned}    
    \end{equation}
    Let $I_1$ and $I_2$ denote the first two terms on the right-hand side:
    \begin{equation}
    I_1 := -\frac{1}{f_{\alpha,\vy}}\nabla_{\vx} \cdot \left(\widehat{\mH}(\vx,t) \widehat{Q}_{\alpha,\vy}f_{\alpha,\vy}\right), \
    I_2 := \frac{1}{2f_{\alpha,\vy}}V(t)^2\Delta_{x}(\widehat{Q}_{\alpha,\vy}f_{\alpha,\vy}).
    \end{equation}
    Note that $\widehat{\mH}(\vx,t) :\mathbb{R}^{n+1} \rightarrow \mathbb{R}^n$ is vector-valued, while both $\widehat{Q}_{\alpha,\vy}:\mathbb{R}^{n+1} \rightarrow \mathbb{R}$ and $f_{\alpha,\vy}:\mathbb{R}^n \rightarrow \mathbb{R}$ are scalar-valued. A direct computation shows that the first term $I_1$ simplifies to:
    \begin{equation}
    \label{eqn: simplified I_1}
    \begin{aligned}
    I_1 &= -\frac{1}{f_{\alpha,\vy}}\nabla_{\vx} \cdot \left(\widehat{\mH}(\vx,t) \widehat{Q}_{\alpha,\vy}f_{\alpha,\vy}\right)\\
    &=  -\frac{1}{f_{\alpha,\vy}}\left(\nabla_{\vx}\cdot\left(\widehat{\mH}(\vx,t) \right)\widehat{Q}_{\alpha,\vy}f_{\alpha,\vy} + \widehat{\mH}(\vx,t) ^\intercal\nabla_{\vx}\left(\widehat{Q}_{\alpha,\vy}f_{\alpha,\vy}\right)\right)\\
    &=-\nabla_{\vx} \cdot \left(\widehat{\mH}(\vx,t)\right)\widehat{Q}_{\alpha,\vy} -\frac{1}{f_{\alpha,\vy}}\widehat{\mH}(\vx,t) ^\intercal\left(\left(\nabla_{\vx}\widehat{Q}_{\alpha,\vy}\right)f_{\alpha,\vy} + \widehat{Q}_{\alpha,\vy}\left(\nabla_{\vx}f_{\alpha,\vy}\right)\right)\\
    &= -\nabla_{\vx} \cdot \left(\widehat{\mH}(\vx,t)\right)\widehat{Q}_{\alpha,\vy} - \widehat{\mH}(\vx,t) ^\intercal\nabla_{\vx}\widehat{Q}_{\alpha,\vy}-\alpha_t\left(\widehat{\mH}(\vx,t)^\intercal\nabla_{\vx}\mu_{\vy}\right)\widehat{Q}_{\alpha,\vy},
    \end{aligned}    
    \end{equation}
    where the last equality above follows from the fact that $\frac{1}{f_{\alpha,\vy}}\nabla_{\vx}f_{\alpha,\vy} = \alpha_t\nabla_{\vx}\mu_{\vy}$ for $f_{\alpha,\vy} = \exp(\alpha_t \mu_{\vy})$. Similarly, expanding the Laplacian term $\Delta_{\vx}(\widehat{Q}_{\alpha,\vy}f_{\alpha,\vy})$ allows us to simplify the second term $I_2$ as follows:
    \begin{equation}
    \label{eqn: simplified I_2}
    \begin{aligned}
    I_2 &= \frac{1}{2f_{\alpha,\vy}}V(t)^2\left(\left(\Delta_{\vx}\widehat{Q}_{\alpha,\vy}\right)f_{\alpha,\vy} + 2\left(\nabla_{\vx}\widehat{Q}_{\alpha,\vy}\right)^\intercal\nabla_{\vx}f_{\alpha,\vy} + \widehat{Q}_{\alpha,\vy}\left(\Delta_{\vx}f_{\alpha,\vy}\right)\right)\\
    &= \frac{1}{2}V(t)^2\Delta_{\vx}\widehat{Q}_{\alpha,\vy} + \alpha_tV(t)^2\left(\nabla_{\vx}\widehat{Q}_{\alpha,\vy}\right)^\intercal\nabla_{\vx}\mu_{\vy} + \frac{1}{2}V(t)^2\left(\alpha_t\Delta_{\vx}\mu_{\vy} + \alpha^2_t\|\nabla_{\vx}\mu_{\vy}\|_2^2\right)\widehat{Q}_{\vy},
    \end{aligned}    
    \end{equation}
    where the last equality above follows from the fact that $\frac{1}{f_{\alpha,\vy}}\Delta_{\vx}f_{\alpha,\vy} = \alpha_t\Delta_{\vx}\mu_{\vy} + \alpha^2_t\|\nabla_{\vx}\mu_{\vy}\|_2^2$ for $f_{\alpha,\vy} = \exp(\alpha_t \mu_{\vy})$. Summing the two expressions in~\eqref{eqn: simplified I_1} and~\eqref{eqn: simplified I_2} then yields
    \begin{equation}
    \begin{aligned}
    \frac{\partial}{\partial t}\widehat{Q}_{\alpha,\vy} &= I_1 + I_2 = -\nabla_{\vx} \cdot \left(\widehat{\mH}(\vx,t)\right)\widehat{Q}_{\alpha,\vy} + \left(\alpha_tV(t)^2\nabla_{\vx}\mu_{\vy} - \widehat{\mH}(\vx,t) \right)^\intercal\nabla_{\vx}\widehat{Q}_{\alpha,\vy}\\
    &+ \frac{1}{2}V(t)^2\Delta_{\vx}\widehat{Q}_{\alpha,\vy} + \left(\frac{1}{2}V(t)^2\left(\alpha_t\Delta_{\vx}\mu_{\vy} + \alpha^2_t\|\nabla_{\vx}\mu_{\vy}\|_2^2\right) -\alpha_t\widehat{\mH}(\vx,t) ^\intercal\nabla_{\vx}\mu_{\vy}\right)\widehat{Q}_{\vy} -\alpha'_t\mu_{\vy}\widehat{Q}_{\alpha,\vy}\\
    &=-\nabla_{\vx} \cdot \left(\left(\widehat{\mH}(\vx,t)-\alpha_tV(t)^2\nabla_{\vx}\mu_{\vy}\right)\widehat{Q}_{\vy}\right) - \alpha_tV(t)^2\Delta_{\vx}\mu_{\vy}\widehat{Q}_{\vy}\\
    &+ \frac{1}{2}V(t)^2\Delta_{\vx}\widehat{Q}_{\alpha,\vy} + \left(\frac{1}{2}V(t)^2\left(\alpha^2_t\|\nabla_{\vx}\mu_{\vy}\|_2^2 + \alpha_t\Delta_{\vx}\mu_{\vy}\right) - \alpha_t\widehat{\mH}(\vx,t)^\intercal\nabla_{\vx}\mu_{\vy}-\alpha'_t \mu_{\vy}\right)\widehat{Q}_{\alpha,\vy}
    \end{aligned}    
    \end{equation}
    which is exactly the dynamics given in~\eqref{eqn: PDE of unnormalized posterior diffusion app}, as desired. 
\end{proof}

\begin{lemma}
    \label{lem: normalizing linear term in dynamics}
    Consider the following PDE dynamics governing the evolution of some unnormalized density $\widehat{Q}(\vx,t):\mathbb{R}^n \times [0, T] \rightarrow \mathbb{R}^{+}$
    \begin{equation}
    \label{eqn: pde unnormalized in lemma}
    \frac{\partial}{\partial t}\widehat{Q}(\vx,t) = -\nabla_{\vx} \cdot \left(K(\vx,t)\widehat{Q}(\vx,t)\right) + \zeta(t)\Delta_{\vx}\widehat{Q}(\vx,t) + J(\vx,t)\widehat{Q}(\vx,t),
    \end{equation}
    where $\zeta:[0,T] \rightarrow \mathbb{R}^{+}$ and $K, J:\mathbb{R}^{d} \times [0,T] \rightarrow \mathbb{R}$. Then we consider the normalized density $\widehat{q}(\vx,t) : \mathbb{R}^n \times [0,T] \rightarrow [0,1]$ defined as below
    \begin{equation}
    \widehat{q}(\vx,t) := \frac{\widehat{Q}(\vx,t)}{\int_{\mathbb{R}^n}\widehat{Q}(\vx,t)\dif \vx}, \ t \in [0,T].    
    \end{equation}
    The PDE dynamics governing the evolution of the normalized density $\widehat{q}(\vx,t)$ is then given by
    \begin{equation}
    \label{eqn: pde normalized in lemma}
    \begin{aligned}
    \frac{\partial}{\partial t}\widehat{q}(\vx,t) &= -\nabla_{\vx} \cdot \left(K(\vx,t)\widehat{q}(\vx,t)\right) + \zeta(t)\Delta_{\vx}\widehat{q}(\vx,t) + \left(J(\vx,t) - \int_{\mathbb{R}^n}J(\vx,t)\widehat{q}(\vx,t)\dif \vx\right)\widehat{q}(\vx,t).
    \end{aligned}
    \end{equation}
\end{lemma}

\begin{proof}
    By using $Z(t) := \int_{\mathbb{R}^n}\widehat{Q}(\vx,t)\dif\vx$ to denote the normalizing constant for any $t \in [0,T]$, we can then compute the time derivative of $Z(t)$ by plugging in~\eqref{eqn: pde unnormalized in lemma} as follows
    \begin{equation}
    \label{eqn: time derivative of normalizing constant}
    \begin{aligned}
    \frac{\partial}{\partial t}Z(t) &= \frac{\partial}{\partial t}\left(\int_{\mathbb{R}^n}\widehat{Q}(\vx,t)\dif\vx\right) = \int_{\mathbb{R}^n}\left(\frac{\partial}{\partial t}\widehat{Q}(\vx,t)\right)\dif\vx\\
    &=\int_{\mathbb{R}^n}\left(-\nabla_{\vx} \cdot \left(K(\vx,t)\widehat{Q}(\vx,t)\right) + \zeta(t)\Delta_{\vx}\widehat{Q}(\vx,t) + J(\vx,t)\widehat{Q}(\vx,t)\right)\dif\vx\\
    &=\int_{\mathbb{R}^n}J(\vx,t)\widehat{Q}(\vx,t)\dif\vx + \int_{\mathbb{R}^n}\nabla_{\vx} \cdot \left(\zeta(t)\nabla_{\vx}\widehat{Q}(\vx,t) - K(\vx,t)\widehat{Q}(\vx,t)\right)\dif\vx\\
    &=\int_{\mathbb{R}^n}J(\vx,t)\widehat{Q}(\vx,t)\dif\vx.
    \end{aligned}    
    \end{equation}
    Furthermore, we may rewrite the normalized density as $\widehat{q}(\vx,t) = \frac{1}{Z(t)}\widehat{Q}(\vx,t)$ and differentiate the expression with respect to $t$, which yields
    \begin{equation*}
    \begin{aligned}
    \frac{\partial}{\partial t}\widehat{q}(\vx,t) &= \frac{1}{Z(t)^2}\left(\left(\frac{\partial}{\partial t}\widehat{Q}(\vx,t)\right)Z(t) - \left(\frac{\partial}{\partial t}Z(t)\right)\widehat{Q}(\vx,t)\right)\\   
    &= \frac{1}{Z(t)}\left(\frac{\partial}{\partial t}\widehat{Q}(\vx,t)\right) - \frac{1}{Z(t)}\left(\frac{\partial}{\partial t}Z(t)\right)\left(\frac{1}{Z(t)}\widehat{Q}(\vx,t)\right)\\
    &= \frac{1}{Z(t)}\left(-\nabla_{\vx} \cdot \left(K(\vx,t)\widehat{Q}(\vx,t)\right) + \zeta(t)\Delta_{\vx}\widehat{Q}(\vx,t) + J(\vx,t)\widehat{Q}(\vx,t)\right)\\
    &-\frac{1}{Z(t)}\left(\int_{\mathbb{R}^n}J(\vx,t)\widehat{Q}(\vx,t)\dif\vx\right)\widehat{q}(\vx,t)\\
    &= -\nabla_{\vx} \cdot \left(K(\vx,t)\widehat{q}(\vx,t)\right) + \zeta(t)\Delta_{\vx}\widehat{q}(\vx,t) + \left(J(\vx,t) - \int_{\mathbb{R}^n}J(\vx,t)\widehat{q}(\vx,t)\dif \vx\right)\widehat{q}(\vx,t).
    \end{aligned}
    \end{equation*}
    where the second last equality above follows from~\eqref{eqn: pde unnormalized in lemma} and~\eqref{eqn: time derivative of normalizing constant} the last equality is deduced from the definition of the normalized density $\widehat{q}(\vx,t)$. This concludes our proof.
\end{proof}

\begin{remark}
    % We remark that derivation used in the proof of Lemma~\ref{lem: normalizing linear term in dynamics} above has also been partially provided in a recent work~\cite[Appendix A]{skreta2025feynman}. 
    By setting 
    $$
    K(\vx,t) := \widehat{\mH}(\vx,t)-\alpha_tV(t)^2\nabla_{\vx}\mu_{\vy}(\vx), \quad \zeta(t):=\frac{1}{2}V(t)^2,
    $$ 
    and 
    $$J(\vx,t) := \frac{1}{2}V(t)^2\left(\alpha^2_t\|\nabla_{\vx}\mu_{\vy}(\vx)\|_2^2 - \alpha_t\Delta_{\vx}\mu_{\vy}(\vx)\right) - \alpha_t\widehat{\mH}(\vx,t)^\intercal\nabla_{\vx}\mu_{\vy}(\vx) -\alpha'_t \mu_{\vy},$$ 
    one can use Lemma~\ref{lem: normalizing linear term in dynamics} to deduce~\eqref{eqn: PDE of posterior diffusion} from~\eqref{eqn: PDE of unnormalized posterior diffusion}.     
\end{remark}

\begin{lemma}
    \label{lem: derivation of single particle dynamics}    
    Consider a single particle $(\vx_t,\beta_t)$ governed by
    \begin{equation}
    \label{eqn: PDE soln via single particle app} 
    \begin{cases}
    \dif \vx_t &= \left(\widehat{\mH}(\vx_t,t)-\alpha_tV(t)^2\nabla_{\vx}\mu_{\vy}(\vx_t)\right)\dif t + V(t)\dif \vw_t,\\
    \dif \beta_t &= \left(U_{\alpha,\vy}(\vx_t,t) - \alpha_t\widehat{\mH}(\vx_t,t)^\intercal\nabla_{\vx}\mu_{\vy}(\vx_t)-\alpha'_t \mu_{\vy}(\vx_t)\right)\beta_t \dif t \\
    &- \left(\displaystyle\int_{\mathbb{R}^n}\left(U_{\alpha,\vy}(\vx_t,t) - \alpha_t\widehat{\mH}(\vx_t,t)^\intercal\nabla_{\vx}\mu_{\vy}(\vx_t)-\alpha'_t \mu_{\vy}(\vx_t)\right)\left(P_{\beta}\gamma_t\right)(\vx)\dif \vx\right)\beta_t \dif t,    
    \end{cases}
    \end{equation}
    with initial condition $\vx_0 = \vx^\ast$ and $\beta_0 = 1$, where $\vx^\ast$ is sampled from the initial posterior distribution $\widehat{q}_{\vy}(\vx,0)$, $(\vw_t)_{t\geq 0}$ is a standard Brownian motion in $\mathbb{R}^n$, $\gamma_t(\vx,\beta)$ denotes the joint probability distribution of $(\vx_t,\beta_t)$ on $\mathbb{R}^n \times \mathbb{R}$, 
    $$
        P_{\beta}\gamma_t(\vx) := \int_{\mathbb{R}}\beta\gamma_{t}(\vx,\beta)\dif\beta
    $$ 
    denotes the weighted projection of $\gamma_t$ onto $\vx$, and $U_{\alpha,\vy}(\vx_t,t) := \frac{1}{2}V(t)^2\left(\alpha^2_t\|\nabla_{\vx}\mu_{\vy}\|_2^2-\alpha_t\Delta_{\vx}\mu_{\vy}\right)$. Below we further define
    \begin{equation}
    \label{eqn: defn of W in appendix}
    \begin{aligned}
    W_{\alpha,\vy}(\vx,t) &:=  \frac{1}{2}V(t)^2\left(\alpha^2_t\|\nabla_{\vx}\mu_{\vy}\|_2^2-\alpha_t\Delta_{\vx}\mu_{\vy}\right) - \alpha_t\widehat{\mH}(\vx,t)^\intercal\nabla_{\vx}\mu_{\vy}-\alpha'_t\mu_{\vy}(\vx)\\
    &= U_{\alpha,\vy}(\vx_t,t) - \alpha_t\widehat{\mH}(\vx,t)^\intercal\nabla_{\vx}\mu_{\vy}-\alpha'_t\mu_{\vy}(\vx).
    \end{aligned}    
    \end{equation}
    Then we have that $P_{\beta}\gamma_t(\vx) = \widehat{q}_{\vy}(\vx,t)$ for any $\vx \in \mathbb{R}^n$ and $t \in [0,T]$, \emph{i.e.} $P_{\beta}\gamma_t(\cdot)$ solves the following PDE:
    \begin{equation}
    \label{eqn: PDE of posterior diffusion app2}
    \begin{aligned}
    \frac{\partial}{\partial t}\widehat{q}_{\alpha,\vy} = &-\nabla_{\vx} \cdot \left(\left(\widehat{\mH}(\vx,t)-\alpha_tV(t)^2\nabla_{\vx}\mu_{\vy}\right)\widehat{q}_{\alpha,\vy}\right)
    +\tfrac{1}{2}V(t)^2\Delta_{x}\widehat{q}_{\alpha,\vy} \\
    &+\left(W_{\alpha,\vy}(\vx,t) - \int_{\mathbb{R}^n}W_{\alpha,\vy}(\vz,t)\widehat{q}_{\alpha,\vy}(\vz)\dif \vz\right)\widehat{q}_{\alpha,\vy}.
    \end{aligned}
    \end{equation}
\end{lemma}

The main idea behind our proof of the Lemma above is to derive the PDE governing the evolution of the joint distribution $\gamma_{t}(\vx,\beta)$ first, which then leads to a PDE for its weighted projection $P_{\beta}\gamma_{t}(\vx)$. Our derivation here is mainly based on the theory of semigroups. 
%\haoxuan{change latter - better notations for semigroup and infinitesimal operators should be $\mathcal{T}_{s,t}$ and $\mathcal{L}_t$ due to time-inhomogenity}

\begin{definition}[Two-Parameter Semigroup Operator]
    \label{def:semigroup}
    Given fixed time $s > 0$, consider a single particle $(\vx_t,\beta_t)$ with initial condition $(\vx_s,\beta_s) = (\vx^\ast,\beta^\ast)$ for any $t > s$. Then the corresponding semigroup operator $\gT^{(\vx,\beta)}_{s,t}$ is defined via
    \begin{equation}
        \gT^{(\vx,\beta)}_{s,t}\phi(\vx^\ast,\beta^\ast) := \mathbb{E}\left[\phi(\vx_t,\beta_t) \ | \ (\vx_s,\beta_s) = (\vx^\ast, \beta^\ast)\right],
        \label{eqn: defn of semigroup}
    \end{equation}
    where $\phi:\mathbb{R}^n \times \mathbb{R} \rightarrow \mathbb{R}$ above denotes an arbitrary test function. For the special case when $s=0$, we write  $\gT^{(\vx,\beta)}_{0,t} = \gT^{(\vx,\beta)}_{t}$. 
\end{definition}

\begin{definition}[Time-Dependent Infinitesimal Generator]
    \label{def:infinitesimal_generator}
    Let $\sI$ be the identity operator. Then for any fixed time $s > 0$ and suitable test function $\phi$, the infinitesimal generator $\Ls^{(\vx,\beta)}_{s}$ associated with the semigroup $\gT^{(\vx,\beta)}_{s,t}$ is defined by
    \begin{equation}
        \Ls^{(\vx,\beta)}_{s}\phi\left(\vx^\ast,\beta^\ast\right) := \lim_{\Delta s \rightarrow 0^{+}}\frac{1}{\Delta s}\left(\gT^{(\vx,\beta)}_{s, s+\Delta s}\phi(\vx^\ast,\beta^\ast) - \phi(\vx^\ast,\beta^\ast)\right).
        \label{eqn: defn of infinitesimal generator}
    \end{equation}
\end{definition}
Moreover, for any test function $\phi:\mathbb{R}^n \times \mathbb{R} \rightarrow \mathbb{R}$ and input $(\vx^\ast,\beta^\ast)$, we have
\begin{equation}
    \label{eqn: commutativity of semigroups}
    \begin{aligned}
    \gT^{(\vx,\beta)}_{t_1+t_2,t_1+t_2+t_3}\circ \gT^{(\vx,\beta)}_{t_1, t_1+t_2}\phi(\vx^\ast,\beta^\ast)
    &= \mathbb{E}\left[\phi(\vx_{t_1+t_2+t_3},\beta_{t_1+t_2+t_3}) \ | \ (\vx_{t_1},\beta_{t_1}) = (\vx^\ast, \beta^\ast)\right]\\
    &=\gT^{(\vx,\beta)}_{t_1,t_1+t_2+t_3}\phi(\vx^\ast,\beta^\ast) 
    \end{aligned}    
\end{equation}
demonstrating that $\gT^{(\vx,\beta)}_{t_1+t_2,t_1+t_2+t_3}\circ \gT^{(\vx,\beta)}_{t_1, t_1+t_2} = \gT^{(\vx,\beta)}_{t_1,t_1+t_2+t_3}$ for any time $t_1, t_2, t_3 > 0$. 

Furthermore, combining~\eqref{eqn: commutativity of semigroups} with the definition of the infinitesimal generator in~\eqref{eqn: defn of infinitesimal generator}, we can directly deduce the following equation for any time $0<s<t$, input $(\vx^\ast,\beta^\ast)$ and test function $\phi:\mathbb{R}^n \times \mathbb{R} \rightarrow \mathbb{R}$,
\begin{equation}
    \label{eqn: commutativity of semigroup and infinitesimal generator}
    \begin{aligned}
     \frac{\partial}{\partial t}\gT^{(\vx,\beta)}_{s,t}\phi &= \lim_{\Delta t \rightarrow 0^+}\frac{1}{\Delta t}\left(\gT^{(\vx,\beta)}_{s,t+\Delta t} - \gT^{(\vx,\beta)}_{s,t}\right)\phi\left(\vx^\ast,\beta^\ast\right) = \lim_{\Delta t \rightarrow 0^+}\frac{1}{\Delta t} \gT^{(\vx,\beta)}_{s,t} \circ \left(\gT^{(\vx,\beta)}_{t,t+\Delta t} -\sI\right)\phi\left(\vx^\ast,\beta^\ast\right) \\
     &= \gT^{(\vx,\beta)}_{s,t} \circ \left(\lim_{\Delta t \rightarrow 0^+}\frac{1}{\Delta t}\left(\gT^{(\vx,\beta)}_{t,t+\Delta t} -\sI\right)\right)\phi\left(\vx^\ast,\beta^\ast\right) = \left(\gT^{(\vx,\beta)}_{s,t}\circ \Ls^{(\vx,\beta)}_{t}\right)\phi \left(\vx^\ast,\beta^\ast\right)
     %&= \gT^{(\vx,\beta)}_{0,t}\lim_{\Delta t \rightarrow 0^{+}}\frac{1}{\Delta t}\left(\gT^{(\vx,\beta)}_{0,\Delta t}-\sI\right)\phi(\vx^\ast,\beta^\ast) \\
     %&= \lim_{\Delta t \rightarrow 0^{+}}\frac{1}{\Delta t}\gT^{(\vx,\beta)}_{\Delta t,t + \Delta t}\left(\gT^{(\vx,\beta)}_{0,\Delta t}-\sI\right)\phi(\vx^\ast,\beta^\ast)\\
     % =&\frac{\partial}{\partial t}\left(\gT^{(\vx,\beta)}_{t}\phi\right)(\vx^\ast,\beta^\ast)\\
     %&= \lim_{\Delta t \rightarrow 0^{+}}\frac{1}{\Delta t}\left(\gT^{(\vx,\beta)}_{t,t+\Delta t}-\sI\right)\gT^{(\vx,\beta)}_{0,t}\phi(\vx^\ast,\beta^\ast) \\ %&=\Ls^{(\vx,\beta)}_{t}\circ\gT^{(\vx,\beta)}_{0,t}\phi\left(\vx^\ast,\beta^\ast\right) = \Ls^{(\vx,\beta)}_{t}\circ\gT^{(\vx,\beta)}_{t}\phi\left(\vx^\ast,\beta^\ast\right),
    \end{aligned}    
\end{equation}
which is essentially the forward Kolgomorov equation expressed in terms of semigroups and infinitesimal generators. Moreover, for any $d \in \mathbb{Z}^{+}$ and two functions $\varphi^{(1)},\varphi^{(2)}:\mathbb{R}^{d} \rightarrow \mathbb{R}$, we use 
$$\left\langle\varphi^{(1)},\varphi^{(2)}\right\rangle_{L^2(\mathbb{R}^d)}:= \int_{\mathbb{R}^d}\varphi^{(1)}(\vx)\varphi^{(2)}(\vx)\dif\vx$$
to denote the inner product between $\varphi^{(1)}$ and $\varphi^{(2)}$. Should no confusion arise, we omit the subscript $L^2(\R^d)$ in all derivations below.

\begin{proposition}
The joint distribution $\gamma_t = \gamma_t(\vx, \beta)$ satisfies the following PDE:
    \begin{equation}
        \frac{\partial}{\partial t}\gamma_t = -\nabla_{\vx} \cdot \left(\mK_{\alpha,t}\gamma_t\right) - \frac{\partial}{\partial \beta}\left(b_{\alpha,t} \gamma_t\right)+\frac{1}{2}V(t)^2\Delta_{\vx}\gamma_t,
        \label{eqn: weak form of joint PDE}
    \end{equation}
    with initial condition $\gamma_0(\vx,\beta) =\widehat{q}_{\alpha,\vy}(\vx,0) \times \delta_{\beta = 1}$, where the two functions $\mK_{\alpha,t}:\mathbb{R}^d \rightarrow \mathbb{R}$ and $b_{\alpha,t}:\mathbb{R}^d \times \mathbb{R} \rightarrow \mathbb{R}$ correspond to the drift and reweighting terms in~\eqref{eqn: PDE soln via single particle} and~\eqref{eqn: PDE soln via single particle app}, \emph{i.e.},
    \begin{equation}
    \label{eqn: notation for the sde drifts}
    \begin{aligned}
    \mK_{\alpha,t}(\vx) &= \widehat{\mH}(\vx,t)-\alpha_tV(t)^2\nabla_{\vx}\mu_{\vy}(\vx),\\
    b_{\alpha,t}(\vx,\beta) &= \left(U_{\alpha,\vy}(\vx,t) - \alpha_t\widehat{\mH}(\vx,t)^\intercal\nabla_{\vx}\mu_{\vy}(\vx)-\alpha'_t\mu_{\vy}(\vx)\right)\beta\\
    &- \left(\int_{\mathbb{R}^n}\left(U_{\alpha,\vy}(\vx^\ast,t)  - \alpha_t\widehat{\mH}(\vx^\ast,t)^\intercal\nabla_{\vx}\mu_{\vy}(\vx^\ast)-\alpha'_t\mu_{\vy}(\vx)\right)\left(P_{\beta}\gamma_t\right)(\vx^\ast)\dif \vx^\ast\right)\beta.
    \end{aligned}
    \end{equation}
\end{proposition}

\begin{proof}
We note that our proof here is mainly based on the weak formulation of PDEs. Specifically, for any fixed time $t$ and test function $\varphi: \mathbb{R}^n \times \mathbb{R} \rightarrow \mathbb{R}$,, integrating the function $\gT^{(\vx,\beta)}_{t}\varphi$ over the initial joint distribution $\gamma_{0}(\vx,\beta)$ yields
    \begin{equation}
    \label{eqn: change of time in measure}
    \begin{aligned}
    \left\langle \gT^{(\vx,\beta)}_{t}\varphi,\gamma_0\right\rangle &= \int_{\mathbb{R}^n \times \mathbb{R}}\gT^{(\vx,\beta)}_{t}\varphi(\vx^\ast,\beta^\ast)\gamma_{0}\left(\vx^\ast,\beta^\ast\right)\dif\vx^\ast\dif\beta^\ast\\
    &= \int_{\mathbb{R}^n \times \mathbb{R}} \mathbb{E}\left[\varphi(\vx_t,\beta_t) \ | \ (\vx_0,\beta_0) = (\vx^\ast, \beta^\ast)\right]\gamma_{0}\left(\vx^\ast,\beta^\ast\right)\dif\vx^\ast\dif\beta^\ast\\
    &= \int_{\mathbb{R}^n \times \mathbb{R}}\varphi(\vx^\ast,\beta^\ast)\gamma_{t}(\vx^\ast,\beta^\ast)\dif\vx^\ast\dif\beta^\ast = \left\langle\varphi, \gamma_t\right\rangle.
    \end{aligned}    
    \end{equation}

    We integrate on both sides of~\eqref{eqn: commutativity of semigroup and infinitesimal generator} over the initial joint distribution $\gamma_{0}(\vx,\beta)$ and plug in~\eqref{eqn: change of time in measure}, which gives us that for any test function $\varphi:\mathbb{R}^n \times \mathbb{R} \rightarrow \mathbb{R}$,
    \begin{equation}
    \label{eqn: intermediate derivaion of weak form for joint PDE}
    \begin{aligned}
    \left\langle\varphi, \frac{\partial}{\partial t}\gamma_t\right\rangle &= \frac{\dif}{\dif t}\left\langle\varphi,\gamma_t\right\rangle = \frac{\dif}{\dif t}\left\langle\gT^{(\vx,\beta)}_{t}\varphi, \gamma_0\right\rangle = \left\langle\frac{\partial}{\partial t}\gT^{(\vx,\beta)}_{0,t}\varphi, \gamma_0\right\rangle \\
    &= \left\langle\gT^{(\vx,\beta)}_{0,t}\circ\Ls^{(\vx,\beta)}_t\varphi, \gamma_0\right\rangle = \left\langle\gT^{(\vx,\beta)}_{t}\circ\Ls^{(\vx,\beta)}_t\varphi, \gamma_0\right\rangle 
    = \left\langle\Ls^{(\vx,\beta)}_t\varphi, \gamma_t\right\rangle.  
    \end{aligned}    
    \end{equation}
    To further simplify the term on the RHS above, we need to compute the explicit form of the infinitesimal generator defined in~\eqref{eqn: defn of infinitesimal generator}. In fact, applying It\^{o}'s formula to the joint SDE~\eqref{eqn: PDE soln via single particle} yields the following identity for any test function $\varphi:\mathbb{R}^n \times \mathbb{R} \rightarrow \mathbb{R}$,
    \begin{equation}
    \label{eqn: Ito's formula applied to joint SDE}    
    \dif \varphi(\vx_t,\beta_t) = \left(\left(\nabla_{\vx}\varphi\right)^\intercal \mK_{\alpha,t} + \frac{\partial \varphi}{\partial \beta}b_{\alpha,t} + \frac{1}{2}V(t)^2\Tr\left(\nabla_{\vx}^2\varphi\right)\right)\dif t + V(t)\left(\left(\nabla_{\vx}\varphi\right)^\intercal\dif\vw_t\right),
    \end{equation}
    where $(\vw_t)_{t\geq 0}$ is a standard Brownian motion on $\mathbb{R}^n$ and the two functions. Taking expectation on both sides of~\eqref{eqn: Ito's formula applied to joint SDE} then yields the explicit expression of the infinitesimal generator for any test function $\varphi:\mathbb{R}^n \times \mathbb{R} \rightarrow \mathbb{R}$ as below:
    \begin{equation}
    \label{eqn: explicit expression of infinitesimal generator}
    \Ls^{(\vx,\beta)}_t\varphi = \left(\nabla_{\vx}\varphi\right)^\intercal\mK_{\alpha,t} + \frac{\partial \varphi}{\partial \beta}b_{\alpha,t} + \frac{1}{2}V(t)^2\Delta_{\vx}\varphi.
    \end{equation}
    Below we use $x_i$ and $\mK_{\alpha,t,i}$ to denote the $i$-th component of $\vx$ and $\mK_{\alpha,t}$ for any $i \in [n]$. By substituting~\eqref{eqn: explicit expression of infinitesimal generator} into the RHS of~\eqref{eqn: intermediate derivaion of weak form for joint PDE}, we obtain that for any test function $\varphi:\mathbb{R}^n \times \mathbb{R} \rightarrow \mathbb{R}$,
    \begin{equation}
    \label{eqn: derivation of weak form for joint density}
    \begin{aligned}
    \left\langle\Ls^{(\vx,\beta)}_t\varphi, \gamma_t\right\rangle =& \left\langle\left(\nabla_{\vx}\varphi\right)^\intercal\mK_{\alpha,t} + \frac{\partial \varphi}{\partial \beta}b_{\alpha,t} + \frac{1}{2}V(t)^2\Delta_{\vx}\varphi, \gamma_t\right\rangle\\
    =& \sum_{i=1}^{n}\left\langle\frac{\partial \varphi}{\partial x_i},\mK_{\alpha,t,i} \gamma_{t} \right\rangle + \left\langle\frac{\partial \varphi}{\partial \beta},b_{\alpha,t}\gamma_{t} \right\rangle + \frac{1}{2}V(t)^2\sum_{i=1}^{n}\left\langle\frac{\partial^2 \varphi}{\partial x_i^2}, \gamma_{t} \right\rangle\\
    =& \left\langle\varphi, -\sum_{i=1}^{n}\frac{\partial}{\partial x_i}\left(\mK_{\alpha,t,i} \gamma_{t}\right) - \frac{\partial}{\partial \beta}\left(b_{\alpha,t} \gamma_{t}\right) + \frac{1}{2}V(t)^2\Delta_{\vx}\gamma_{t} \right\rangle\\
    =& \left\langle\varphi, -\nabla_{\vx} \cdot \left(\mK_{\alpha,t}\gamma_t\right) - \frac{\partial}{\partial \beta}\left(b_{\alpha,t} \gamma_t\right)+\frac{1}{2}V(t)^2\Delta_{\vx}\gamma_t\right\rangle,
    \end{aligned}    
    \end{equation}
    where the second last equality above follows from integration by parts. Substituting the last expression in~\eqref{eqn: derivation of weak form for joint density} above into~\eqref{eqn: intermediate derivaion of weak form for joint PDE} then gives us the weak form of the PDE associated with the joint distribution $\gamma_t$ in~\eqref{eqn: weak form of joint PDE}, which concludes the proof.
\end{proof}

\begin{proof}[Proof of Lemma~\ref{lem: derivation of single particle dynamics}]
By defining 
$$\gamma^{\mP}_t(\vx) := P_{\beta}\gamma_t(\vx) = \int_{\mathbb{R}}\beta\gamma_{t}(\vx,\beta)\dif\beta$$
to be the weighted projection of $\gamma_t$, we then have that $\gamma^{\mP}_0(\vx) = \widehat{q}_{\alpha,\vy}(\vx,0)$. 
Below we proceed to derive the PDE govering the evolution of $\gamma^{\mP}_t$ based on~\eqref{eqn: weak form of joint PDE}. For any test function $\psi:\mathbb{R}^n \rightarrow \mathbb{R}$, taking $\varphi(\vx,\beta) = \beta\psi(\vx):\mathbb{R}^{n} \times \mathbb{R} \rightarrow \mathbb{R}$ in the weak form derived in~\eqref{eqn: intermediate derivaion of weak form for joint PDE} and~\eqref{eqn: derivation of weak form for joint density} yields
    \begin{equation}
    \label{eqn: intermediate weak form for PDE of projected measure}
    \begin{aligned}
    &\left\langle\psi, \frac{\partial}{\partial t}\gamma^{\mP}_t\right\rangle = \frac{\dif}{\dif t}\left\langle\psi,\gamma^{\mP}_t\right\rangle= \frac{\dif}{\dif t}\left(\int_{\mathbb{R}^n}\psi(\vx)\left(\int_{\mathbb{R}}\beta\gamma_{t}(\vx,\beta)\dif\beta\right)\dif\vx\right) \\
    =& \frac{\dif}{\dif t}\left\langle\varphi,\gamma_t\right\rangle = \left\langle\varphi, \frac{\partial}{\partial t}\gamma_t\right\rangle
    = \left\langle\varphi, -\nabla_{\vx} \cdot \left(\mK_{\alpha,t}\gamma_t\right) - \frac{\partial}{\partial \beta}\left(b_{\alpha,t} \gamma_t\right)+\frac{1}{2}V(t)^2\Delta_{\vx}\gamma_t\right\rangle\\
    =& -\left\langle \varphi, \nabla_{\vx} \cdot \left(\mK_{\alpha,t}\gamma_t\right) \right\rangle -\left\langle \varphi, \frac{\partial}{\partial \beta}\left(b_{\alpha,t} \gamma_t\right) \right\rangle + \frac{1}{2}V(t)^2\left\langle \varphi, \Delta_{\vx}\gamma_t \right\rangle.
    \end{aligned}    
    \end{equation}

    For the first and third terms in the last expression of~\eqref{eqn: intermediate weak form for PDE of projected measure}, we can further simplify them as follows
    \begin{equation}
    \label{eqn: derivation of first term}
    \begin{aligned}
    \left\langle \varphi, \nabla_{\vx} \cdot \left(\mK_{\alpha,t}\gamma_t\right) \right\rangle &= \int_{\mathbb{R}}\beta\left(\int_{\mathbb{R}^n}\psi(\vx)\left(\nabla_{\vx} \cdot\left(\mK_{\alpha,t}(\vx) \gamma_t(\vx,\beta)\right)\right)\dif\vx\right)\dif\beta\\
    &= \int_{\mathbb{R}^n}\psi(\vx)\left(\nabla_{\vx} \cdot\left(\mK_{\alpha,t}(\vx) \left(\int_{\mathbb{R}}\beta\gamma_t(\vx,\beta)\dif\beta\right)\right)\right)\dif\vx\\
    &= \int_{\mathbb{R}^n}\psi(\vx)\left(\nabla_{\vx} \cdot\left(\mK_{\alpha,t}(\vx)\gamma^{\mP}_t(\vx)\right)\right)\dif\vx \\   &=\left\langle\psi,\nabla_{\vx}\cdot(\mK_{\alpha,t}\gamma^{\mP}_t)\right\rangle
    \end{aligned}    
    \end{equation}
    and 
    \begin{equation}
        \label{eqn: derivation of third term}
        \begin{aligned}
            \left\langle \varphi, \Delta_{\vx}\gamma_t \right\rangle &= \int_{\mathbb{R}}\beta\left(\int_{\mathbb{R}^n}\psi(\vx)\left(\Delta_{\vx}\gamma_t(\vx,\beta)\right)\dif\vx\right)\dif\beta\\
            &= \int_{\mathbb{R}^n}\psi(\vx)\left(\Delta_{\vx}\left(\int_{\mathbb{R}}\beta\gamma_t(\vx,\beta)\dif\beta\right)\right)\dif\vx\\
            &= \int_{\mathbb{R}^n}\psi(\vx)\left(\Delta_{\vx}\gamma^{\mP}_t\right)(\vx)\dif\vx = \left\langle\psi,\Delta_{\vx}\gamma^{\mP}_t\right\rangle,
        \end{aligned}
    \end{equation}
    respectively. Moreover, for the second term in the last expression of~\eqref{eqn: intermediate weak form for PDE of projected measure}, we may plug in the expression of $b_{\alpha,t}$ and $W_{\alpha,\vy}(\vx,t)$ defined in~\eqref{eqn: notation for the sde drifts} and~\eqref{eqn: defn of W in appendix} above and apply integration by parts to deduce that
    \begin{equation}
    \label{eqn: derivation of reweighting term}
    \begin{aligned}
    &\left\langle \varphi, \frac{\partial}{\partial \beta}\left(b_{\alpha,t} \gamma_t\right) \right\rangle = -\left\langle \frac{\partial}{\partial \beta}\varphi, b_{\alpha,t}\gamma_t \right\rangle = -\left\langle \psi(\vx), b_t \gamma_t \right\rangle\\
    = &-\int_{\mathbb{R}}\psi(\vx)\gamma_t(\vx,\beta)\left(W_{\alpha,\vy}(\vx,t)-\int_{\mathbb{R}^n}W_{\alpha,\vy}(\vx^\ast,t)\gamma^{\mP}_t(\vx^\ast)\dif\vx^\ast\right)\beta\dif\vx\dif\beta\\
    = &-\int_{\mathbb{R}}\psi(\vx)\gamma^{\mP}_t(\vx)\left(W_{\alpha,\vy}(\vx,t)-\int_{\mathbb{R}^n}W_{\alpha,\vy}(\vx^\ast,t)\gamma^{\mP}_t(\vx^\ast)\dif\vx^\ast\right)\dif\vx\\
    = &-\left\langle\psi(\vx), \gamma^{\mP}_t(\vx)\left(W_{\alpha,\vy}(\vx,t)-\int_{\mathbb{R}^n}W_{\alpha,\vy}(\vx^\ast,t)\gamma^{\mP}_t(\vx^\ast)\dif\vx^\ast\right)\right\rangle.
    \end{aligned}    
    \end{equation}
    Substituting~\eqref{eqn: derivation of first term},~\eqref{eqn: derivation of third term}, and~\eqref{eqn: derivation of reweighting term} into~\eqref{eqn: intermediate weak form for PDE of projected measure} then gives us the weak form of the PDE governing the evolution of the projected measure $\gamma^{\mP}_t = P_{\beta}\gamma_t$, \emph{i.e.}, $\gamma^{\mP}_t(\vx) = P_{\beta}\gamma_t(\vx)$ satisfies the following PDE
    \begin{equation}
    \label{eqn: PDE of projected measure}
    \frac{\partial}{\partial t}\gamma^{\mP}_t = -\nabla_{\vx} \cdot \left(\mK_{\alpha,t}\gamma^{\mP}_t\right) +\frac{1}{2}V(t)^2\Delta_{\vx}\gamma^{\mP}_t + \left(W_{\alpha,\vy}(\vx,t)-\int_{\mathbb{R}^n}W_{\alpha,\vy}(\vx^\ast,t)\gamma^{\mP}_t(\vx^\ast)\dif\vx^\ast\right)\gamma^{\mP}_t,
    \end{equation}
    with initial condition $\gamma^{\mP}_0(\vx) =\widehat{q}_{\alpha,\vy}(\vx,0)$. This is exactly the PDE in~\eqref{eqn: PDE of posterior diffusion app2}, which concludes our proof.
\end{proof}

\begin{remark}[Comparison with Concurrent Work~\citep{skreta2025feynman}]
    We note that an alternative approach to derive the dynamics~\eqref{eqn: PDE soln via single particle} for a weighted particle from the PDE~\eqref{eqn: PDE of posterior diffusion} is to use the Feynman-Kac formula under the formulation of path integrals, as presented in the concurrent work~\cite[Appendix A]{skreta2025feynman}. Here we adopt the approach used for proving ~\cite[Lemma 1 and 10]{domingo2020mean}, which is mainly based on the idea of lifting the projected measure to the joint measure and the weak form of PDE solutions. 

    \label{rmk: difference between our dynamics and FK corrector}    
    % This remark further elaborates on the distinctions between the particle dynamics of our proposed AFDPS method and those of the Feynman-Kac (FK) Corrector introduced by Skreta et al.~\citep{skreta2025feynman}. The FK Corrector is selected for this detailed comparison because, similar to our proposed AFDPS method, it is an SMC-type algorithm that can be generalized to address a wide range of nonlinear inverse problems.

    We adapt the FK Corrector dynamics from~\cite[Proposition D.5]{skreta2025feynman} to provide a direct comparison with our dynamics of a weighted particle (derived from the PDE~\eqref{eqn: PDE of posterior diffusion} and presented as~\eqref{eqn: PDE soln via single particle}) for the setting of posterior sampling. This is achieved by setting the parameters in their notations as $\beta_t=\alpha_t$, the noise intensity $\sigma_t = V(t)^2$, and the reward function $r = -\mu_{\vy}$. The resulting drift and reweighting terms for both methods are juxtaposed in Table~\ref{tab:drift-reweighting-combined}.
    \begin{table}[htbp]
        \centering
        \caption{Drift and Reweighting Terms of AFDPS and FK Corrector}
        \begin{tabular}{c|c|c}
            \noalign{\hrule height 1.2pt}
            \textbf{Term} & \textbf{AFDPS (Ours)} & \textbf{FK Corrector} \\
            \noalign{\hrule height 1pt}
            Drift & 
            \makecell[l]{\\[-4pt]$-F(t)\vx + V(t)^2 \vphi_\theta(\vx,t)$\\[2pt]$ - \alpha_tV(t)^2\nabla_{\vx}\mu_{\vy}$\\[5pt]} & 
            $-F(t)\vx + V(t)^2 \vphi_\theta(\vx,t)$ \\
            \hline
            Reweighting 
            & 
            \makecell[l]{\\[-4pt]
            $\frac{1}{2}V(t)^2\left(\alpha^2_t\|\nabla_{\vx}\mu_{\vy}\|_2^2 - \alpha_t\Delta_{\vx}\mu_{\vy}\right)$\\[2pt]
            $+ \alpha_t\left(F(t)\vx - V(t)^2 \vphi_\theta(\vx,t)\right)^\intercal\nabla_{\vx}\mu_{\vy}-\alpha'_t\mu_{\vy}$\\[6pt]}
            & 
            \makecell[l]{\\[-4pt]
            $-\frac{1}{2}V(t)^2\left(\alpha^2_t\|\nabla_{\vx}\mu_{\vy}\|_2^2 - \alpha_t\Delta_{\vx}\mu_{\vy}\right)$\\[2pt]
            $+\alpha_tF(t)\vx^\intercal\nabla_{\vx}\mu_{\vy} -\alpha'_t\mu_{\vy}$\\[6pt]} \\
            \noalign{\hrule height 1.2pt}
        \end{tabular}
        \label{tab:drift-reweighting-combined}
    \end{table}

    It is noteworthy that if $V(t)=0$ (\emph{i.e.}, in the absence of the diffusion-based corrector $\vphi_\theta$ and the gradient guidance $\nabla_{\vx}\mu_{\vy}$), both AFDPS and the FK Corrector would simplify to the same ODE dynamics, with their drift terms reducing to $-F(t)\vx$. However, in the more general SDE case where $V(t) \neq 0$, the $-V(t)^2\nabla_{\vx}\mu_{\vy}$ term in our AFDPS drift marks a critical difference. Our empirical results, detailed in Section~\ref{sec: results of numerics}, demonstrate that this specific term plays a vital role in effectively guiding the sampler towards regions of high likelihood, thereby enhancing performance.  
    
    In fact, by using $Q_{\alpha,\vy}(\vx) := \cev{p}_t(\vx)e^{-\alpha_t\mu_{\vy}(\vx)}$ to denote the unnormalized posterior associated with the ground-truth backward SDE~\eqref{eqn: backward sde in standard dm} with $G(t) = V(t)$, we can directly differentiate $Q_{\alpha,\vy}$ with respect to $\vx$ to obtain that:
    \begin{equation}
    \label{eqn: derivation of grad of Qy}
    \begin{aligned}
    \nabla_{\vx}Q_{\alpha,\vy} &= \nabla_{\vx}\left(\cev{p}_t e^{-\alpha_t\mu_{\vy}}\right) = \left(\nabla_{\vx}\cev{p}_t\right)e^{-\alpha_t\mu_{\vy}} - \alpha_t\cev{p}_t e^{-\alpha_t\mu_{\vy}}\left(\nabla_{\vx}\mu_{\vy}\right)\\
    &= \cev{p}_t e^{-\alpha_t\mu_{\vy}}\left(\nabla_{\vx}\log \cev{p}_t - \alpha_t\nabla_{\vx}\mu_{\vy}\right) = Q_{\alpha,\vy}\left(\nabla_{\vx}\log \cev{p}_t - \alpha_t\nabla_{\vx}\mu_{\vy}\right)
    \end{aligned}
    \end{equation}
Moreover, a derivation similar to the proof of Lemma~\ref{lem: simplifying dynamics} yields that the PDE dynamics governing the evolution of $Q_{\vy}$ is given by 
    \begin{equation}
    \begin{aligned}
        \frac{\partial}{\partial t}Q_{\alpha,\vy} = &-\nabla_{\vx} \cdot \left(\left(\mH(\vx,t)-\alpha_tV(t)^2\nabla_{\vx}\mu_{\vy}\right)Q_{\alpha,\vy}\right)
        +\frac{1}{2}V(t)^2\Delta_{x}Q_{\alpha,\vy}\\
        &+\left(\frac{1}{2}V(t)^2\left(\alpha^2_t\|\nabla_{\vx}\mu_{\vy}\|_2^2-\alpha_t\Delta_{\vx}\mu_{\vy}\right) - \alpha_t\mH(\vx,t)^\intercal\nabla_{\vx}\mu_{\vy}-\alpha'_t\mu_{\vy}\right)Q_{\alpha,\vy}
    \end{aligned}    
    \end{equation}
where $\mH(\vx,t) = -F(t)\vx + V(t)^2\nabla_{\vx}\log \cev{p}_{t}(\vx)$ is essentially obtained by replacing the neural network-based approximation $\vphi_{\theta}(\vx,t)$ in the expression of $\widehat{\mH}(\vx,t)$ defined above with the true score function $\nabla_{\vx}\log \cev{p}_{t}(\vx)$. For any fixed scalar $\eta \in \mathbb{R}$, we may further decompose the term $\nabla_{\vx}\mu_{\vy}$ above as the sum of $\eta \nabla_{\vx}\mu_{\vy}$ and $(1-\eta)\nabla_{\vx}\mu_{\vy}$ and directly simplify the RHS above as follows:
    \begin{equation}
    \label{eqn: true unnormalized intermediate dynamics}
        \begin{aligned}
            \frac{\partial}{\partial t}Q_{\alpha,\vy} 
            = &-\nabla_{\vx} \cdot \left(\left(\mH(\vx,t)-\eta \alpha_tV(t)^2\nabla_{\vx}\mu_{\vy}\right)Q_{\alpha,\vy}\right)
            + (1-\eta)\alpha_tV(t)^2\nabla_{\vx} \cdot \left(Q_{\alpha,\vy}\nabla_{\vx}\mu_{\vy}\right)\\
            &+\frac{1}{2}V(t)^2\Delta_{\vx}Q_{\alpha,\vy} +\left(\frac{1}{2}V(t)^2\left(\alpha^2_t\|\nabla_{\vx}\mu_{\vy}\|_2^2-\alpha_t\Delta_{\vx}\mu_{\vy}\right) - \alpha_t\mH(\vx,t)^\intercal\nabla_{\vx}\mu_{\vy}-\alpha'_t\mu_{\vy}\right)Q_{\alpha,\vy}\\
            = &-\nabla_{\vx} \cdot \left(\left(\mH(\vx,t)-\eta \alpha_tV(t)^2\nabla_{\vx}\mu_{\vy}\right)Q_{\alpha,\vy}\right)
            + (1-\eta)\alpha_tV(t)^2\left(\nabla_{\vx}\mu_{\vy}\right)^\intercal \nabla_{\vx}Q_{\alpha,\vy}\\
            &+ (1-\eta)\alpha_tV(t)^2Q_{\alpha,\vy}\left(\Delta_{\vx}\mu_{\vy}\right) +\frac{1}{2}V(t)^2\Delta_{\vx}Q_{\alpha,\vy} \\
            &+\left(\frac{1}{2}V(t)^2\left(\alpha^2_t\|\nabla_{\vx}\mu_{\vy}\|_2^2-\alpha_t\Delta_{\vx}\mu_{\vy}\right) - \alpha_t\mH(\vx,t)^\intercal\nabla_{\vx}\mu_{\vy}-\alpha'_t\mu_{\vy}\right)Q_{\alpha,\vy}\\
            = &-\nabla_{\vx} \cdot \left(\left(\mH(\vx,t)-\eta\alpha_t -V(t)^2\nabla_{\vx}\mu_{\vy}\right)Q_{\alpha,\vy}\right)-\alpha'_t\mu_{\vy}Q_{\alpha,\vy}\\
            &+ (1-\eta)\alpha_tV(t)^2\left(\nabla_{\vx}\mu_{\vy}\right)^\intercal \left(\nabla_{\vx}\log \cev{p}_t - \alpha_t\nabla_{\vx}\mu_{\vy}\right)Q_{\alpha,\vy} + \frac{1}{2}V(t)^2\Delta_{\vx}Q_{\alpha,\vy} \\
            &+\left(\frac{1}{2}V(t)^2\alpha^2_t\|\nabla_{\vx}\mu_{\vy}\|_2^2+\left(\frac{1}{2}-\eta\right)\alpha_tV(t)^2\Delta_{\vx}\mu_{\vy} - \alpha_t\mH(\vx,t)^\intercal\nabla_{\vx}\mu_{\vy}-\alpha'_t\mu_{\vy}\right)Q_{\alpha,\vy}\\
            = &-\nabla_{\vx} \cdot \left(\left(\mH(\vx,t)-\eta \alpha_tV(t)^2\nabla_{\vx}\mu_{\vy}\right)Q_{\alpha,\vy}\right)+ \frac{1}{2}V(t)^2\Delta_{\vx}Q_{\alpha,\vy}-\alpha'_t\mu_{\vy}Q_{\alpha,\vy} \\
            &+\left(\eta -\frac{1}{2}\right)V(t)^2\left(\alpha^2_t\|\nabla_{\vx}\mu_{\vy}\|_2^2-\alpha_t\Delta_{\vx}\mu_{\vy}\right)Q_{\alpha,\vy}\\
            &+\alpha_t\left(F(t)\vx -\eta V(t)^2 \nabla_{\vx}\log \cev{p}_t\left(\vx\right)\right)^\intercal\left(\nabla_{\vx}\mu_{\vy}\right)Q_{\alpha,\vy}.
        \end{aligned}    
    \end{equation}
    where the second last equality above follows from plugging in~\eqref{eqn: derivation of grad of Qy}. 
    
    By replacing the true score function $\nabla_{\vx}\log \cev{p}_t(\vx)$ in the RHS above with the neural network-based estimator $\vphi_{\theta}(\vx,t)$, one then obtains the dynamics that can be used in practice. Specifically, for any fixed $\eta \in \mathbb{R}$, the drift term used in practice is given by
    \begin{equation}
    \label{eqn: general drift in practice}    
    -F(t)\vx + V(t)^2\vphi_{\theta}(\vx,t) - \eta \alpha_tV(t)^2\nabla_{\vx}\mu_{\vy},
    \end{equation}
    while the reweighting term used in practice is given by 
    \begin{equation}
    \label{eqn: reweighting in practice}
    \left(\eta - \frac{1}{2}\right)V(t)^2\left(\alpha_t\left\|\nabla_{\vx}\mu_{\vy}\right\|_2^2 - \alpha_t\Delta_{\vx}\mu_{\vy}\right) + \alpha_t\left(F(t)\vx -\eta V(t)^2 \vphi_{\theta}\left(\vx,t\right)\right)^\intercal\nabla_{\vx}\mu_{\vy}-\alpha'_t\mu_{\vy}Q_{\alpha,\vy} 
    \end{equation}
    By comparing the two terms above with Table~\ref{tab:drift-reweighting-combined}, we note that $\eta = 0$ yields the FK Corrector dynamics while $\eta = 1$ yields the AFDPS dynamics. Therefore, for more difficult nonlinear inverse problems, we may control the magnitude of the term $V(t)^2\nabla_{\vx}\mu_{\vy}$ by tuning the parameter $\eta$ in practice. This also conforms to strategies used in existing practical work on guidance like~\citep{dhariwal2021diffusion,ho2022classifier,bansal2023universal,song2023loss,he2023manifold,guo2024gradient,ye2024tfg}. Finally, it would be of independent question to mathematically analyze how the discrepancy between the true dynamics~\eqref{eqn: true unnormalized intermediate dynamics} and the practical dynamics given by~\eqref{eqn: general drift in practice} and~\eqref{eqn: reweighting in practice} depends on the parameter $\eta$ in future work.
\end{remark}

\section{Supplementary Proofs and Justifications for Section~\ref{sec: main result - theory}}
\label{sec: derivations for Sec 4}
In this section, we provide detailed proofs and justifications for claims listed in Section~\ref{sec: main result - theory}. 

\subsection{Proof of Theorem~\ref{thm: bound on true and approx posterior}}
\label{proof: posterior bound by prior error}
We begin by decomposing the total variation error via triangle's inequality. Specifically, here we slightly abuse the notation by taking $\widetilde{q}_{\alpha,\vy,t}(\cdot) := \widehat{q}_{\alpha,\vy}(\cdot,t)$ for any $t \in [0,T]$, \emph{i.e.}, $\widetilde{q}_{\alpha,\vy,t}$ satisfies the PDE~\eqref{eqn: PDE of posterior diffusion}:
\begin{equation}
\label{eqn: copy of posterior PDE}
\begin{aligned}
\frac{\partial}{\partial t}\widetilde{q}_{\alpha,\vy,t} = &-\nabla_{\vx} \cdot \left(\left(\widehat{\mH}(\vx,t)-\alpha_t V(t)^2\nabla_{\vx}\mu_{\vy}\right)\widetilde{q}_{\alpha,\vy,t}\right)+\tfrac{1}{2}V(t)^2\Delta_{x}\widetilde{q}_{\alpha,\vy,t}\\
&+\left(W_{\alpha,\vy}(\vx,t) - \int_{\mathbb{R}^n}W_{\alpha,\vy}(\vz,t)\widetilde{q}_{\alpha,\vy,t}(\vz)\dif \vz\right)\widetilde{q}_{\alpha,\vy,t}   
\end{aligned}
\end{equation}
with initial condition $\widetilde{q}_{\alpha,\vy,0}(\vx) = \widehat{q}_{\alpha,\vy}(\vx,0)$. A direct application of triangle's inequality yields
\begin{equation}
\label{eqn: TV triangle intermediate ineq}
\TV\left(\widehat{q}_{\alpha,\vy,T}, q_{\vy,0}\right) \leq \TV\left(\widetilde{q}_{\alpha,\vy,T}, q_{\vy,0}\right) + \TV\left(\widehat{q}_{\alpha,\vy,T}, \widetilde{q}_{\alpha,\vy,T}\right).    
\end{equation}
We note that our proof in this section can be divided into two parts, which provide upper bounds on the two terms on the RHS above respectively.

\subsubsection{Bounding the First Term on the RHS of~\eqref{eqn: TV triangle intermediate ineq}}
\label{subsubsec: first term of TV errors}
We start off with bounding the first term $\TV\left(\widetilde{q}_{\alpha,\vy,T}, q_{\vy,0}\right)$ on the RHS above, which requires the following two lemmas. Specifically, the first lemma below provides a quantitative bound on the discrepancy between two diffusion processes with different drift functions, while the second lemma describes the convergence of the forward process towards the target distribution when Gaussian noise is added. 
\begin{lemma}
    \label{lem: discrepancy between two diffusion processes}
    For any pair of diffusion processes $(\vx_t)_{t \in [0,T]}$ and $(\widetilde{\vx}_t)_{t \in [0,T]}$ on $\mathbb{R}^n$ defined as follows
    \begin{equation}
    \label{eqn: pair of diffusion processes in lemma}
    \begin{aligned}
    \dif \vx_t &= \vb(\vx_t,t)\dif t + c(t)\dif\vw_t\\
    \text{and}\quad \dif \widetilde{\vx}_t &= \widetilde{\vb}(\widetilde{\vx}_t,t)\dif t + c(t)\dif\vw_t
    \end{aligned}    
    \end{equation}
    where $\vb,\widetilde{\vb}:\mathbb{R}^n \times [0,T] \rightarrow \mathbb{R}^n$ are the two drift functions, $c:[0,T] \rightarrow \mathbb{R}^+$ and $(\vw_t)_{t \in [0,T]}$ is a standard Brownian motion. Let $\rho_t$ and $\widetilde{\rho}_t$ denote the distribution of $\vx_t$ and $\widetilde{\vx}_t$ respectively for any $t \in [0,T]$, then we have
    \begin{equation}
    \label{eqn: final bound on KL divergence}
    \KL(\rho_T \| \widetilde{\rho}_T) \leq \KL(\rho_0 \| \widetilde{\rho}_0) + \int_{0}^{T}\int_{\mathbb{R}^n}\frac{1}{2c(t)^2}\left\|\vb(\vx,t)-\widetilde{\vb}(\vx,t)\right\|_2^2\rho_t(\vx)\dif\vx\dif t.
    \end{equation}
\end{lemma}

\begin{proof}
    We remark that the proof of this lemma is essentially the same as the derivations in many previous works on the theoretical analysis of DMs and variants. Examples include, but are not limited to, ~\cite[Lemma C.1]{chen2023improved}, ~\cite[Lemma 2.22]{albergo2023stochastic}, and ~\cite[Lemma A.4]{wu2024principled}. For the sake of completeness, we include a detailed derivation here. 

    The main idea is to use the Fokker-Planck equations associated with the diffusion processes in~\eqref{eqn: pair of diffusion processes in lemma} and differentiate the KL divergence between the two evolving densities with respect to time. Specifically, we have that $\rho_t$ and $\widetilde{\rho}_t$ satisfy the following Fokker-Planck equations:
    \begin{equation}
    \label{eqn: pair of FKP equations in lemma}
    \begin{aligned}
    \frac{\partial}{\partial t}\rho_t &= -\nabla_{\vx} \cdot \left(\vb(\vx,t)\rho_t\right) + \frac{1}{2}c(t)^2\Delta_{\vx}\rho_t,\\
    \text{and}\quad \frac{\partial}{\partial t}\widetilde{\rho}_t &= -\nabla_{\vx} \cdot \left(\widetilde{\vb}(\vx,t)\widetilde{\rho}_t\right) + \frac{1}{2}c(t)^2\Delta_{\vx}\widetilde{\rho}_t.
    \end{aligned}    
    \end{equation}

    From the definition of the KL divergence 
    $$
        \KL(\rho_t \| \widetilde{\rho}_t) = \int_{\mathbb{R}^n}\log\frac{\rho_t(\vx)}{\widetilde{\rho}_t(\vx)}\rho_t(\vx)\dif\vx,
    $$ 
    we can differentiate it with respect to the time variable $t$, which yields 
    \begin{equation}
    \label{eqn: time derivative of KL in lemma}
    \begin{aligned}
    \frac{\dif}{\dif t}\KL(\rho_t \| \widetilde{\rho}_t) = &\int_{\mathbb{R}^n}\log\frac{\rho_t}{\widetilde{\rho}_t}\frac{\partial \rho_t}{\partial t}\dif\vx + \int_{\mathbb{R}^n}\left(\frac{\partial}{\partial t}\log\rho_t - \frac{\partial}{\partial t}\log\widetilde{\rho}_t\right)\rho_t\dif\vx\\
    = &\int_{\mathbb{R}^n}\log\frac{\rho_t}{\widetilde{\rho}_t}\frac{\partial \rho_t}{\partial t}\dif\vx + \int_{\mathbb{R}^n}\left(\frac{1}{\rho_t}\frac{\partial \rho_t}{\partial t} - \frac{1}{\widetilde{\rho}_t}\frac{\partial \widetilde{\rho}_t}{\partial t}\right)\rho_t\dif\vx\\
    = &\int_{\mathbb{R}^n}\log\frac{\rho_t}{\widetilde{\rho}_t}\frac{\partial \rho_t}{\partial t}\dif\vx - \int_{\mathbb{R}^n}\frac{\rho_t}{\widetilde{\rho}_t}\frac{\partial \widetilde\rho_t}{\partial t}\dif\vx
    \end{aligned}    
    \end{equation}

    For the first term in~\eqref{eqn: time derivative of KL in lemma} above, we plug in~\eqref{eqn: pair of FKP equations in lemma} and use integration by parts, which yields
    \begin{equation}
    \label{eqn: expression of first term in lemma}
    \begin{aligned}
    &\int_{\mathbb{R}^n}\log\frac{\rho_t}{\widetilde{\rho}_t}\frac{\partial \rho_t}{\partial t}\dif\vx\\
    = &-\int_{\mathbb{R}^n}\left(\log\rho_t - \log\widetilde{\rho}_t\right)\nabla_{\vx}\cdot\left(\left(\vb - \frac{c(t)^2}{2}\nabla_{\vx}\log\rho_t\right)\rho_t\right)\dif\vx\\
    = &\int_{\mathbb{R}^n}\left(\nabla_{\vx}\log\rho_t - \nabla_{\vx}\log\widetilde{\rho}_t\right)^\intercal\left(\vb - \frac{c(t)^2}{2}\nabla_{\vx}\log\rho_t\right)\rho_t\dif\vx.
    \end{aligned}    
    \end{equation}

    To simplify the second term in~\eqref{eqn: time derivative of KL in lemma}, we plug in~\eqref{eqn: pair of FKP equations in lemma} apply integration by parts again to obtain that 
    \begin{equation}
        \label{eqn: expression of second term in lemma}
        \begin{aligned}
        &-\int_{\mathbb{R}^n}\frac{\rho_t}{\widetilde{\rho}_t}\frac{\partial \widetilde\rho_t}{\partial t}\dif\vx = \int_{\mathbb{R}^n}\frac{\rho_t}{\widetilde{\rho}_t}\nabla_{\vx}\cdot\left(\left(\widetilde{\vb} - \frac{c(t)^2}{2}\nabla_{\vx}\log\widetilde{\rho}_t\right)\widetilde{\rho}_t\right)\dif\vx\\
        = &\int_{\mathbb{R}^n}\left[\nabla_{\vx}\cdot\left(\widetilde{\vb} - \frac{c(t)^2}{2}\nabla_{\vx}\log\widetilde{\rho}_t\right)\rho_t + \left(\widetilde{\vb} - \frac{c(t)^2}{2}\nabla_{\vx}\log\widetilde{\rho}_t\right)^\intercal\frac{\nabla_{\vx}\widetilde{\rho}_t}{\widetilde{\rho}_t}\rho_t\right]\dif\vx\\
        = &\int_{\mathbb{R}^n}\left(\widetilde{\vb} - \frac{c(t)^2}{2}\nabla_{\vx}\log\widetilde{\rho}_t\right)^\intercal\left(\nabla_{\vx}\log\widetilde{\rho}_{t} - \nabla_{\vx}\log\rho_{t}\right)\rho_t\dif\vx
        \end{aligned}    
    \end{equation}

    Furthermore, substituting and into then yields
    \begin{equation}
    \label{eqn: time derivative of KL final expression in lemma}
    \begin{aligned}
    \frac{\dif}{\dif t}\KL(\rho_t \| \widetilde{\rho}_t) = &\int_{\mathbb{R}^n}\left(\nabla_{\vx}\log\rho_t - \nabla_{\vx}\log\widetilde{\rho}_t\right)^\intercal\left(\vb - \frac{c(t)^2}{2}\nabla_{\vx}\log\rho_t\right)\rho_t\dif\vx\\
    +&\int_{\mathbb{R}^n}\left(\nabla_{\vx}\log\widetilde{\rho}_t - \nabla_{\vx}\log\rho_t\right)^\intercal\left(\widetilde{\vb} - \frac{c(t)^2}{2}\nabla_{\vx}\log\widetilde{\rho}_t\right)\rho_t\dif\vx\\
    = &-\frac{c(t)^2}{2}\int_{\mathbb{R}^n}\left\|\nabla_{\vx}\log\widetilde{\rho}_t - \nabla_{\vx}\log\rho_t\right\|_2^2\rho_t\dif\vx\\
    + &\int_{\mathbb{R}^n}\left(\vb-\widetilde{\vb}\right)^\intercal\left(\nabla_{\vx}\log\rho_t-\nabla_{\vx}\log\widetilde{\rho}_t\right)\rho_t\dif\vx\\
    \leq &\frac{1}{2c(t)^2}\int_{\mathbb{R}^n}\left\|\vb-\widetilde{\vb}\right\|_2^2\rho_t\dif\vx = \frac{1}{2c(t)^2}\int_{\mathbb{R}^n}\left\|\vb(\vx,t)-\widetilde{\vb}(\vx,t)\right\|_2^2\rho_t(\vx)\dif\vx
    \end{aligned}    
    \end{equation}
    where the last inequality follows from the AM-GM inequality, \emph{i.e.} $\vx^\intercal\vy \leq \frac{1}{2c(t)^2}\|\vx\|_2^2 + \frac{c(t)^2}{2}\|\vy\|_2^2$ for any vectors $\vx,\vy \in\mathbb{R}^n$ and $t \in [0,T]$. 

    Integrating~\eqref{eqn: time derivative of KL final expression in lemma} from $t=0$ to $t=T$ then yields~\eqref{eqn: final bound on KL divergence}, which concludes our proof. 
\end{proof}
\begin{lemma}
\label{lem: convergence of the forward process}
For any distribution $p$ on $\mathbb{R}^n$ with bounded second moment $m_2^2$, \emph{i.e.}, $\mathbb{E}_{\vx \sim p}[\|\vx\|_{2}^2] \leq m_2^2$, we have $\KL\left(p \ast \gN(\boldsymbol{0},\sigma^2\mI_n) \| \gN(\boldsymbol{0},\sigma^2\mI_n)\right) \leq \frac{m_2^2}{2\sigma^2}$, where $(p \ast q)(\vx) := \int_{\mathbb{R}^n}p(\vy)q(\vx-\vy)\dif\vy$ denotes the convolution of the two probability distributions $p,q$. 
\end{lemma}
\begin{proof}
We remark that this is the same as~\cite[Lemma 10]{wang2024evaluating}, where a complete proof is already provided. 
\end{proof}

%Theorem~\ref{thm: bound on true and approx posterior}. 
%\begin{proof}[Proof of Theorem~\ref{thm: bound on true and approx posterior}]

With Lemma~\ref{lem: discrepancy between two diffusion processes} and Lemma~\ref{lem: convergence of the forward process} listed above, we then proceed to bound the term $\TV\left(\widetilde{q}_{\alpha,\vy,T}, q_{\vy,0}\right)$. Consider the backward process associated with the true score function under the EDM framework, which can be formally written as
\begin{equation}
\label{eqn: true EDM SDE}
\dif \cev{\vx}_t = \left[-\frac{\dot{s}(t)}{s(t)}\cev{\vx}_t + 2s(t)^2\dot{\sigma}(t)\sigma(t)\nabla \log \cev{p}_t(\cev{\vx}_t)\right]\dif t + s(t)\sqrt{2\dot{\sigma}(t)\sigma(t)}\dif \vw_t.    
\end{equation}
with initial condition 
$$
    \cev{\vx}_0 \sim \cev{p}_0 = p_{T} = p_0 \ast \gN(\boldsymbol{0}, T^2\mI_n),
$$ 
where the last identity follows from results derived in Appendix B.1 in the paper~\citep{karras2022elucidating} that proposes the EDM framework as well as our particular choices of the scaling functions $s(t)=1$ and $\sigma(t)=t$. 

Then we consider applying Lemma~\ref{lem: discrepancy between two diffusion processes} to compare the two diffusion processes $(\cev{x}_t)_{t \in [0,T]}$ and $(\widehat{\cev{\vx}}_t)_{t \in [0,T]}$ defined in~\eqref{eqn: true EDM SDE} and~\eqref{eqn: EDM SDE} respectively. 

By setting $c(t) = s(t)\sqrt{2\dot{\sigma}(t)\sigma(t)} = \sqrt{2t},$ 
$$\vb(\vx,t) = -\frac{\dot{s}(t)}{s(t)}\vx + 2s(t)^2\dot{\sigma}(t)\sigma(t)\nabla \log \cev{p}_t(\vx) = 2t\nabla \log \cev{p}_t(\vx)$$ and $$\widetilde{\vb}(\vx,t) = -\frac{\dot{s}(t)}{s(t)}\vx + 2s(t)^2\dot{\sigma}(t)\sigma(t)\vphi_\theta(\vx,t) = 2t\vphi_\theta(\vx,t),$$ 
we have
\begin{equation}
\begin{aligned}
\KL(p_0 \| \widehat{\cev{p}}_T) &= \KL(\cev{p}_T \| \widehat{\cev{p}}_T)\\
&\leq \KL(\cev{p}_0 \| \widehat{\cev{p}}_0) + \int_{0}^{T}\int_{\mathbb{R}^n}\frac{1}{4t}\left\|2t\left(\vphi_{\theta}(\vx,t)-\nabla_{\vx}\log \cev{p}_{t}(\vx)\right)\right\|_2^2\cev{p}_t(\vx)\dif\vx\dif t\\
&= \KL\left(p_0 \ast \gN(\boldsymbol{0},T^2\mI_n) \| \gN(\boldsymbol{0},T^2\mI_n)\right) \\
&+ \int_{0}^{T}\int_{\mathbb{R}^n}t\left\|\vphi_{\theta}(\vx,t)-\nabla_{\vx}\log \cev{p}_{t}(\vx)\right\|_2^2\cev{p}_t(\vx)\dif\vx\dif t \leq \frac{m_2^2}{2T^2} + \frac{1}{2}T^2\epsilon_{\vs}^2,
\end{aligned}    
\end{equation}
where the second lest inequality above follows from Lemma~\ref{lem: discrepancy between two diffusion processes} and the last inequality follows from Assumption~\ref{assump: bounded second moment}, Assumption~\ref{assump: score matching training error} and Lemma~\ref{lem: convergence of the forward process}. 

Applying Pinsker's inequality helps us further bound the TV divergence between $p_0$ and $\cev{\widehat{p}}_T$ as follows
\begin{equation}
\label{eqn: bound on TV between prior dist}
\TV\left(\widehat{\cev{p}}_T,p_0\right)=\TV\left(p_0,\widehat{\cev{p}}_T\right) \leq \sqrt{\frac{1}{2}\KL(p_0 \| \widehat{\cev{p}}_T)} \leq \frac{1}{2}\sqrt{\frac{m_2^2}{T^2} + T^2\epsilon_{\vs}^2}.
\end{equation}

Based on the bounds on the distance between the two prior distributions above, we proceed to bound the distance between the two associated posterior distributions. From our definition of $\widetilde{q}_{\alpha,\vy,t}$ in equation~\eqref{eqn: copy of posterior PDE} above, we have that
$$\widetilde{q}_{\alpha,\vy,T}(\vx) = \widehat{q}_{\alpha,\vy}(\vx,T) \propto \widehat{\cev{p}}_T(\vx)e^{-\alpha_t\mu_{\vy}(\vx)} =\widehat{\cev{p}}_T(\vx)e^{-\mu_{\vy}(\vx)} \quad \text{ and } \quad q_{\vy,0}(\vx) \propto p_0(x)e^{-\mu_{\vy}(\vx)},$$ 
By using
$$\widehat{Z}(\vy) := \int_{\mathbb{R}^n}\widehat{\cev{p}}_T(\vx)e^{-\mu_{\vy}(\vx)}\dif \vx \quad \text{ and } \quad Z(\vy) := \int_{\mathbb{R}^n}p_0(x)e^{-\mu_{\vy}(\vx)}\dif\vx$$ to denote the two corresponding normalizing constants, we can further deduce that
\begin{equation}
\label{eqn: bound on difference of normalizing constants}
\left|\widehat{Z}(\vy) -Z(\vy)\right| = \left|\int_{\mathbb{R}^n}e^{-\mu_{\vy}(\vx)}\left(\widehat{\cev{p}}_T(\vx) - p_0(\vx)\right)\dif\vx\right| \leq 2e^{-C^{(1)}_{\vy}}\TV\left(\widehat{\cev{p}}_T,p_0\right)    
\end{equation}
where the inequality above follows from Assumption~\ref{assump: regularity and boundedness of log-likelihood}. Then we can use the bound on the difference between the normalizing constants above to further obtain that 
\begin{equation*}
\begin{aligned}
\TV(\widetilde{q}_{\alpha,\vy,T},q_{\vy,0}) &= \frac{1}{2}\int_{\mathbb{R}^n}\left|\frac{1}{\widehat{Z}(\vy)}\widehat{\cev{p}}_T(\vx)e^{-\mu_{\vy}(\vx)} - \frac{1}{Z(\vy)}p_0(\vx)e^{-\mu_{\vy}(\vx)}\right|\dif\vx \\
&\leq \frac{1}{2}\int_{\mathbb{R}^n}\left|\frac{1}{\widehat{Z}(\vy)}\widehat{\cev{p}}_T(\vx)e^{-\mu_{\vy}(\vx)} - \frac{1}{Z(\vy)}\widehat{\cev{p}}_T(\vx)e^{-\mu_{\vy}(\vx)}\right|\dif\vx \\
&+ \frac{1}{2}\int_{\mathbb{R}^n}\left|\frac{1}{Z(\vy)}\widehat{\cev{p}}_T(\vx)e^{-\mu_{\vy}(\vx)} - \frac{1}{Z(\vy)}p_0(\vx)e^{-\mu_{\vy}(\vx)}\right|\dif\vx \\
&= \frac{|Z(\vy)-\widehat{Z}(\vy)|}{2Z(\vy)\widehat{Z}(\vy)}\left(\int_{\mathbb{R}^n}\widehat{\cev{p}}_T(\vx)e^{-\mu_{\vy}(\vx)}\dif\vx\right) \\
&+ \frac{1}{2Z(\vy)}\left|\int_{\mathbb{R}^n}e^{-\mu_{\vy}(\vx)}\left(\widehat{\cev{p}}_T(\vx) - p_0(\vx)\right)\dif\vx\right|\\
&\leq \frac{1}{2Z(\vy)}\left(\left|\widehat{Z}(\vy) - Z(\vy)\right| + 2e^{-C^{(1)}_{\vy}}\TV\left(\widehat{\cev{p}}_T,p_0\right)\right)\\
&\leq \frac{2e^{-C^{(1)}_{\vy}}}{Z(\vy)}\TV\left(\widehat{\cev{p}}_T,p_0\right) \leq  \frac{e^{-C^{(1)}_{\vy}}}{Z(\vy)}\sqrt{\frac{m_2^2}{T^2} + T^2\epsilon_{\vs}^2},
\end{aligned}
\end{equation*}
where the first inequality above follows from triangle inequality, the second inequality above follows from Assumption~\ref{assump: regularity and boundedness of log-likelihood}, the third inequality above follows from~\eqref{eqn: bound on difference of normalizing constants} and the last inequality above follows from~\eqref{eqn: bound on TV between prior dist}. 

By setting 
$$
    C^{(2)}_{\vy}:= \frac{e^{-C^{(1)}_{\vy}}}{Z(\vy)}
$$ 
in the last expression above, which is some constant that only depends on $\vy$, we finally obtain the following upper bound for the first term on the RHS of~\eqref{eqn: TV triangle intermediate ineq}:
\begin{equation}
\label{eqn: upper bound on the fist term of C2}
\TV\left(\widetilde{q}_{\alpha,\vy,T}, q_{\vy,0}\right) \leq C^{(2)}_{\vy}\sqrt{\frac{m_2^2}{T^2} + T^2\epsilon_{\vs}^2}. 
\end{equation}

\subsubsection{Bounding the Second Term on the RHS of~\eqref{eqn: TV triangle intermediate ineq}}
\label{subsubsec: second term of TV errors}
Then we proceed to upper bound the second term $\TV\left(\widehat{q}_{\alpha,\vy,T}, \widetilde{q}_{\alpha,\vy,T}\right)$. Following the same set of notations used in Appendix~\ref{sec: derivations for Sec 3} above, we define $\mK(\vx,t) := \widehat{\mH}(\vx,t)-\alpha_tV(t)^2\nabla_{\vx}\mu_{\vy}(\vx)$ and $\zeta(t):=\frac{1}{2}V(t)^2$. Moreover, we use $\mathcal{P}_t$ to denote the following time-dependent operator for any test function $\phi:\mathbb{R}^n \rightarrow \mathbb{R}$,
\begin{equation}
(\mathcal{P}_t\phi)(\vx):= \mK(\vx,t)^\intercal\nabla_{\vx}\phi(\vx) + \zeta(t)\Delta_{\vx}\phi(\vx)    
\end{equation}
A direct computation yields that the associated adjoint operator $\mathcal{P}^\ast_t$ is exactly the time-dependent infinitesimal generator given by the sum of the drift and diffusion term on the RHS of~\eqref{eqn: copy of posterior PDE}, \emph{i.e.},
\begin{equation}
\label{eqn: defn of langevin type generator}
\left(\mathcal{P}^\ast_{t}\phi\right)(\vx): = -\nabla_{\vx} \cdot \left(\mK(\vx,t) \phi(\vx)\right) + \zeta(t) \Delta_{\vx}\phi(\vx),  
\end{equation}
for any test function $\phi:\mathbb{R}^n \rightarrow \mathbb{R}$. Furthermore, we use $\lambda :=\lambda_t(\phi)$ to denote the functional formed by the linear term on the RHS of~\eqref{eqn: TV triangle intermediate ineq}, \emph{i.e.}, 
\begin{equation}
\label{eqn: functional of the linear term}
\lambda_t(\phi) := \left(W_{\alpha,\vy}(\vx,t) - \int_{\mathbb{R}^n}W_{\alpha,\vy}(\vz,t)\phi(\vz)\dif \vz\right)\phi(\vx) = \left(tI(\vx,t) - t\int_{\mathbb{R}^n}I(\vz,t)\phi(\vz)\dif\vz\right)\phi(\vx)
\end{equation}
for any test function $\phi:\mathbb{R}^n \rightarrow \mathbb{R}$, where the last identity above follows from our definition of $I(\vx,t)$ given in Assumption~\ref{assump: boundedness and Lipschitz of I} and the special choice that $\sigma(t) = t$. Then we may use the notations introduced above to rewrite the PDE~\eqref{eqn: copy of posterior PDE} as follows:
\begin{equation}
\label{eqn: posterior PDE of revised form}
\frac{\partial}{\partial t}q(\vx,t) = (\mathcal{P}^\ast_t q)(\vx) + \lambda_t(q)
\end{equation}
Based on the time-dependent infinitesimal generator~\eqref{eqn: defn of langevin type generator} and the PDE~\eqref{eqn: posterior PDE of revised form} above, we may further define the two-parameter semigroup operator $\mathcal{U}_{s,t}$ by
\begin{equation}
(\mathcal{U}_{s,t}\phi)(\vx) = \mathbb{E}\left[\phi(X_t) | X_s = \vx\right]    
\end{equation}
for any time interval $[s,t]$ and test function $\phi:\mathbb{R}^n \rightarrow \mathbb{R}$, where $(X_\tau)_{\tau \geq s}$ is driven by the following time-inhomogeneous SDE. 
\begin{equation}
\dif X_\tau = -\mK(X_{\tau},\tau)\dif\tau + \sqrt{2\zeta(t)}\dif W_{\tau}, X_s = \vx.
\end{equation}
Moreover, based on the theory of semigroups and PDEs, we have an alternative interpretation of $\mathcal{U}_{s,t}$. Specifically, for any test function $\phi:\mathbb{R}^n \rightarrow \mathbb{R}$, we have that $(\mathcal{U}_{s,t}\phi)(\vx) = u(\vx,t)$ is exactly the solution of the following forward PDE:
\begin{equation}
\label{eqn: semigroup PDE}
\frac{\partial}{\partial \tau}u(\vx,\tau) = -\nabla_{\vx} \cdot(\mK(\vx,\tau)u(\vx,\tau)) +\zeta(\tau)\Delta_{\vx}u(\vx,\tau) = (\mathcal{P}^\ast_{\tau}u)(\vx) \text{ for } \tau \in (s,t], \ u(\vx,s) = \phi(\vx).
\end{equation}
From the PDE-based interpretation above, we also have that the adjoint two-parameter semigroup operator $\mathcal{U}^\ast_{s,t}$ of $\mathcal{U}_{s,t}$ can be defined via the backward PDE associated with $\mathcal{P}_t$. Specifically, for any time interval $[s,t]$ and test function $\phi:\mathbb{R}^n \rightarrow \mathbb{R}$, we have that $(\mathcal{U}^\ast_{s,t}\phi)(\vx) = v(\vx,s)$ is exactly the solution of the following backward PDE:
\begin{equation}
\label{eqn: adjoint semigroup PDE}
\frac{\partial}{\partial \tau}v(\vx,\tau) = \mK(\vx,\tau)^\intercal \nabla_{\vx}v(\vx,\tau) +\zeta(\tau)\Delta_{\vx}v(\vx,\tau) = (\mathcal{P}_{\tau}v)(\vx) \text{ for } \tau \in [s,t), \ v(\vx,t) = \phi(\vx). 
\end{equation}
Furthermore, we note that the two target distributions $\widehat{q}_{\alpha,\vy,T}$ and $\widetilde{q}_{\alpha,\vy,T}$ are essentially solutions to PDEs of the same form~\eqref{eqn: posterior PDE of revised form} but with two different initial conditions $\widehat{q}_{\alpha,\vy,0}(\cdot)$ and $\widetilde{q}_{\alpha,\vy,T}(\cdot) = \widehat{q}_{\alpha,\vy}(\cdot,0)$. Therefore, in order to upper bound the term, here we only need to prove the stability of PDE~\eqref{eqn: posterior PDE of revised form} above with respect to the initial conditions. Before proving the generic stability argument, we need to prove the following two lemmas beforehand. Specifically, the first lemma characterizes solution to the PDE~\eqref{eqn: posterior PDE of revised form} based on the generator $\mathcal{P}^\ast_t$ and initial condition, while the second lemma shows the contractiveness of the two-parameter semigroup $\mathcal{U}_{s,t}$ with respect to the $L^1$ norm. 
\begin{lemma}[Duhamel's Principle/Variation of Parameters Formula]
A generic formula of the solution $q(t,\vx)=q_t(\vx)$ to the PDE~\eqref{eqn: posterior PDE of revised form} with initial condition $q(0,\vx) = q_{0}(\vx)$ can be written as follows:
\begin{equation}
\label{eqn: variation of parameter formula}
q(\vx,t) = \mathcal{U}_{0,t}q_0 + \int_{0}^{t}(\mathcal{U}_{s,t}\lambda_s(q_s))\dif s    
\end{equation}
\end{lemma}
\begin{proof}
We note that the generic formula above directly follows from Duhamel's Principle/Variation of parameters formula. For the sake of completeness, we provide a derivation below. In fact, using the fact that $\mathcal{U}_{s,t}\phi$ satisfies the PDE~\eqref{eqn: semigroup PDE} with initial condition $u(s,\vx) = \phi(\vx)$ for any $\phi$, we can differentiate with respect to $t$ on both sides of~\eqref{eqn: variation of parameter formula} to deduce that 
\begin{equation}
\begin{aligned}
\frac{\partial}{\partial t}q(\vx,t) &= \frac{\partial}{\partial t}\left(\mathcal{U}_{0,t}q_0\right) + \frac{\dif}{\dif t}\left(\int_{0}^{t}(\mathcal{U}_{s,t}\lambda_s(q_s))\dif s\right) = \frac{\partial}{\partial t}\left(\mathcal{U}_{0,t}q_0\right) +\mathcal{U}_{t,t}\lambda_t(q_t) + \int_{0}^{t}\frac{\partial}{\partial t}(\mathcal{U}_{s,t}\lambda_s(q_s))\dif s\\
&= \left(\mathcal{P}^\ast_{t}\mathcal{U}_{0,t}q_0\right)(\vx) +\lambda_t(q(\vx,t)) +\int_{0}^{t}\left(\mathcal{P}^\ast_t\mathcal{U}_{s,t}\lambda_s(q_s)\right)\dif s \\
&= \mathcal{P}^\ast_t\left(\mathcal{U}_{0,t}q_0(\vx) + \int_{0}^{t}\mathcal{U}_{s,t}\lambda_s(q_s)\dif s\right) + \lambda_t(q(\vx,t)) = (\mathcal{P}^\ast_t q)(\vx) + \lambda_t(q)
\end{aligned}    
\end{equation}
which is exactly the PDE~\eqref{eqn: posterior PDE of revised form}, as desired.  
\end{proof}
\begin{lemma}[Contractiveness of the two-parameter semigroup operator]   
\label{lem: contractiveness of semigroup operator}
For any time interval $[s,t] \subset [0,T]$ and test function $\phi: \mathbb{R}^n \rightarrow \mathbb{R}$, we have $\left\|\mathcal{U}_{s,t}\phi\right\|_{L^1} \leq \left\|\phi\right\|_{L^1}$.
\end{lemma}
\begin{proof}
We note that it suffices to show that the operator norm $\left\|\mathcal{U}_{s,t}\right\|_{L^1 \rightarrow L^1}$ is bounded by $1$. Given that the dual of the $L^1$ norm is exactly the $L^\infty$ norm, we have that 
\begin{equation}
\begin{aligned}
\left\|\mathcal{U}_{s,t}\right\|_{L^1 \rightarrow L^1} &= \sup_{\left\|\phi\right\|_{L^1} \leq 1}\left\|\mathcal{U}_{s,t}\phi\right\|_{L^1} = \sup_{\|\phi\|_{L^1} \leq 1}\sup_{\|\psi\|_{L^\infty} \leq 1}\langle \psi, \mathcal{U}_{s,t}\phi \rangle_{L^2}\\
&= \sup_{\|\psi\|_{L^\infty} \leq 1}\sup_{\|\phi\|_{L^1} \leq 1}\langle \phi, \mathcal{U}^\ast_{s,t}\psi \rangle_{L^2} = \sup_{\left\|\psi\right\|_{L^\infty} \leq 1}\left\|\mathcal{U}^\ast_{s,t}\psi\right\|_{L^\infty} = \left\|\mathcal{U}^\ast_{s,t}\right\|_{L^\infty \rightarrow L^\infty}
\end{aligned}
\end{equation}
where $\langle \cdot, \cdot \rangle_{L^2}$ above denotes the inner product associated with the $L^2$ norm, \emph{i.e.}, $\langle\psi,\phi\rangle_{L^2} = \int_{\mathbb{R}^n}\psi(\vx)\phi(\vx)\dif\vx$ for any test functions $\psi,\phi:\mathbb{R}^n \rightarrow \mathbb{R}$. Therefore, here we only need to show that
\begin{equation}
\label{eqn: bound on adjoint operator}
\left\|\mathcal{U}^\ast_{s,t}\right\|_{L^\infty \rightarrow L^\infty} \leq 1,    
\end{equation}
\emph{i.e.}, $\left\|\mathcal{U}^\ast_{s,t}\phi\right\|_{L^\infty} \leq \|\phi\|_{L^\infty}$ for any bounded test function $\phi:\mathbb{R}^n \rightarrow \mathbb{R}$. We note that this essentially follows from applying maximum principle to the backward PDE~\eqref{eqn: adjoint semigroup PDE}. For a detailed discussion on the maximum principle, the readers may refer to~\citep{evans2022partial}. Specifically, it suffices to show that the solution $v(\vx,\tau)$ to the PDE~\eqref{eqn: adjoint semigroup PDE} satisfies
\begin{equation}
\label{eqn: backward PDE maximum principles}
\begin{aligned}
\sup_{\tau \in [s,t], \ \vx \in \mathbb{R}^n} v(\vx,\tau) &\leq \sup_{\vx \in \mathbb{R}^n} v(\vx,t), \\
\inf_{\tau \in [s,t], \ \vx \in \mathbb{R}^n} v(\vx,\tau) &\geq \inf_{\vx \in \mathbb{R}^n} v(\vx,t),      
\end{aligned}  
\end{equation}
when the initial function $v(\vx,t)$ is bounded with respect to the $L^\infty$ norm. We begin by proving the first inequality regarding the supremum. For any fixed $\omega>0$ that can be arbitrarily small, we consider the perturbed solution $v^{\omega}(\vx,\tau):= v(\vx,\tau) - \omega \tau$. We will first show that the supremum of $v^{\omega}$ must be attained at the boundary when $\tau =t$, \emph{i.e.}, 
\begin{equation}
\label{eqn: maximum principle sup for perturbed solution}
\sup_{\tau \in [s,t], \ \vx \in \mathbb{R}^n}v^\omega(\vx,\tau) \leq \sup_{\vx \in \mathbb{R}^n}v^\omega(\vx,t).   
\end{equation}
A direct computation yields that $v^\omega$ satisfies the following modified backward PDE:
\begin{equation}
\label{eqn: modified PDE}
\begin{aligned}
\frac{\partial}{\partial \tau}v^{\omega}(\vx,\tau) &= \frac{\partial}{\partial \tau}v(\vx,\tau)-\omega = \left(\mathcal{P}_{\tau}v\right)(\vx)-\omega = \mK(\vx,\tau)^\intercal\nabla_{\vx}v(\vx,\tau) + \zeta(\tau)\Delta_{\vx}v(\vx,\tau)-\omega \\
&= \mK(\vx,\tau)^\intercal\nabla_{\vx}v^\omega(\vx,\tau) + \zeta(\tau)\Delta_{\vx}v^\omega(\vx,\tau) -\omega = \left(\mathcal{P}_{\tau}v^\omega\right)(\vx) -\omega.
\end{aligned}   
\end{equation}
For the sake of contradiction, assume that $\sup_{\tau \in [s,t], \ \vx \in \mathbb{R}^n}v^\omega(\vx,\tau) > \sup_{\vx \in \mathbb{R}^n}v^\omega(\vx,t)$ \emph{i.e.}, the supremum of $v^\omega$ over $(\vx,\tau) \in \mathbb{R}^n \times [s,t]$ is attained at some point $(x^\ast,\tau^\ast) \in \mathbb{R}^n \times [s,t)$. Since $(\vx^\ast,\tau^\ast)$ is also a local maximum, we must have
\begin{equation}
\frac{\partial}{\partial \tau}v^\omega\left(\vx^\ast,\tau^\ast\right) \geq 0, \ \nabla_{\vx}v^\omega\left(\vx^\ast,\tau^\ast\right) = \boldsymbol{0}, \ \Delta_{\vx}v^\omega\left(\vx^\ast,\tau^\ast\right) \leq 0.
\end{equation} 
Substituting the inequalities and equality above into the modified PDE~\eqref{eqn: modified PDE} then yields 
\begin{equation}
0 \leq \frac{\partial}{\partial \tau}v^\omega\left(\vx^\ast,\tau^\ast\right) = \zeta(\tau)\Delta_{\vx}v^\omega\left(\vx^\ast,\tau^\ast\right) -\omega \leq -\omega < 0,   
\end{equation}
which leads to a contradiction. Therefore, the assumption is wrong and~\eqref{eqn: maximum principle sup for perturbed solution} is proved. Based on the definition of $v$ and $v^\omega$, we have the following inequality
\begin{equation}
v^\omega(\vx,\tau) \leq v(\vx,\tau) = v^\omega(\vx,\tau) + \omega\tau \leq v^{\omega}(\vx,\tau) + \omega T, 
\end{equation}
for any $(\vx,\tau) \in \mathbb{R}^n \times [s,t]$. Taking supremum with respect to $(\vx,\tau)$ in the inequality above and plugging in~\eqref{eqn: maximum principle sup for perturbed solution} then imply
\begin{equation*}
\sup_{\tau \in [s,t], \ \vx\in\mathbb{R}^n}v(\vx,\tau) \leq \sup_{\tau \in [s,t], \ \vx \in \mathbb{R}^n}v^\omega(\vx,\tau) + \omega T \leq \sup_{\vx \in \mathbb{R}^n}v^\omega(\vx,t) + \omega T \leq \sup_{\vx \in \mathbb{R}^n}v(\vx,t)+\omega T.
\end{equation*}
Taking the limit $\omega \rightarrow 0^+$ in the inequality above then implies $\sup_{\tau \in [s,t], \ \vx \in \mathbb{R}^n}v(\vx,\tau) \leq \sup_{\vx \in \mathbb{R}^n}v(\vx,t)$, which completes our proof of the first inequality in~\eqref{eqn: backward PDE maximum principles} above. For the second inequality regarding the infimum, we note that $\sup(-v) = -\inf v$. Hence, the second inequality can be proved in the same way by considering $-v$ and the corresponding perturbation $(-v)^\omega := -v +\omega \tau$. This concludes our proof of~\eqref{eqn: bound on adjoint operator} and Lemma~\ref{lem: contractiveness of semigroup operator}.    
\end{proof}
With the two lemmas above, we may now prove stability of the PDE~\eqref{eqn: semigroup PDE}. In general, consider any two solutions $q^{(i)}(\vx,t) \ (i =1,2)$ to~\eqref{eqn: semigroup PDE} with two different initial conditions $q^{(1)}(\vx,0) =q^{(1)}_0(\vx)$ and $q^{(2)}(\vx,0) =q^{(2)}_0(\vx)$, \emph{i.e.},
\begin{equation}
\label{eqn: two PDEs for stability}
\frac{\partial}{\partial t}q^{(i)}(\vx,t) = \left(\mathcal{P}^\ast_t q^{(i)}\right)(\vx) + \lambda_t\left(q^{(i)}\right), (i=1,2)    
\end{equation}
Applying the formula~\eqref{eqn: variation of parameter formula} derived in Lemma above, we also have that the two solutions satisfy 
\begin{equation}
\label{eqn: Duhamel principle for the two solutions}
q^{(i)}(\vx,t) = \mathcal{U}_{0,t}q^{(i)}_0 + \int_{0}^{t}\mathcal{U}_{s,t}\lambda_s\left(q^{(i)}_s\right)\dif s \ (i=1,2).     
\end{equation}
We further use $r_t(\vx):= q^{(1)}(\vx,t) - q^{(2)}(\vx,t)$ to denote the difference between the two solutions. Then it suffices to show that the total variation
\begin{equation}
\TV\left(q^{(1)}(\cdot,t), q^{(2)}(\cdot,t)\right) = \frac{1}{2}\left\|q^{(1)}(\cdot,t)-q^{(2)}(\cdot,t)\right\|_{L^1} = \frac{1}{2}\left\|r_t\right\|_{L^1},    
\end{equation}
can be upper bounded by the initial error $\TV\left(q^{(1)}(\cdot,0), q^{(2)}(\cdot,0)\right) = \frac{1}{2}\left\|r_0\right\|_{L^1}$ up to some constant depending on $\vy$ and $t$. In fact, taking the difference of the two equations in~\eqref{eqn: Duhamel principle for the two solutions} and applying the linearity of the two-parameter semigroup operator $\mathcal{U}_{s,t}$, we have that  
\begin{equation}
\label{eqn: gronwall ineq intermediate}
\begin{aligned}
\left\|r_t\right\|_{L^1} &= \left\|\mathcal{U}_{0,t}\left(q^{(1)}_0 - q^{(2)}_0\right) + \int_{0}^{t}\mathcal{U}_{s,t}\left(\lambda_s\left(q^{(1)}_s\right) - \lambda_s\left(q^{(2)}_s\right)\right)\dif s\right\|_{L^1}\\
&\leq \left\|\mathcal{U}_{0,t}\left(q^{(1)}_0 - q^{(2)}_0\right)\right\|_{L^1} + \int_{0}^{t}\left\|\mathcal{U}_{s,t}\left(\lambda_s\left(q^{(1)}_s\right) - \lambda_s\left(q^{(2)}_s\right)\right)\right\|_{L^1}\dif s\\
&\leq \left\|q^{(1)}_0 - q^{(2)}_0\right\|_{L^1} + \int_{0}^{t}\left\|\lambda_s\left(q^{(1)}_s\right) - \lambda_s\left(q^{(2)}_s\right)\right\|_{L^1}\dif s,
\end{aligned}    
\end{equation}
where the first and second inequality above follows from triangle inequality and contractiveness of the semigroup operator proved in Lemma~\ref{lem: contractiveness of semigroup operator} above, respectively. Moreover, plugging in the expression of $\lambda$ defined in~\eqref{eqn: functional of the linear term} above yields
\begin{equation}
\label{eqn: triangle ineq intermediate step}
\begin{aligned}
\left\|\lambda_s\left(q^{(1)}_s\right) - \lambda_s\left(q^{(2)}_s\right)\right\|_{L^1} &= s\left\|I(\vx,s)\left(q^{(1)}_s-q^{(2)}_s\right)-\langle I(\cdot,s),q^{(1)}_s\rangle_{L^2}q^{(1)}_s + \langle I(\cdot,s),q^{(2)}_s\rangle_{L^2}q^{(2)}_s\right\|_{L^1}\\
&\leq s\left\|I(\vx,s)\left(q^{(1)}_s - q^{(2)}_s\right)\right\|_{L^1} + s\left\|\left(\langle I(\cdot,s), q^{(1)}_s-q^{(2)}_s\rangle_{L^2}\right)q^{(1)}_s\right\|_{L^1}\\
&+ s\left\|\left(\langle I(\cdot,s),q^{(2)}_s\rangle_{L^2}\right)\left(q^{(1)}_s-q^{(2)}_s\right)\right\|_{L^1}.
\end{aligned}    
\end{equation}
Furthermore, applying the upper bound on $I(\vx,s)$ in Assumption~\ref{assump: boundedness and Lipschitz of I} then yields that for any $\vx$ and $s \in [0,t]$,
\begin{equation}
\begin{aligned}
\left\|I(\vx,s)\left(q^{(1)}_s - q^{(2)}_s\right)\right\|_{L^1} &\leq B_{\vy}\left\|q^{(1)}_s - q^{(2)}_s\right\|_{L^1},\\
\left\|\left(\langle I(\cdot,s), q^{(1)}_s-q^{(2)}_s\rangle_{L^2}\right)q^{(1)}_s\right\|_{L^1} &\leq \left(\int_{\mathbb{R}^n}\left|I(\vx,s)\right|\left|q^{(1)}_s(\vx)-q_s^{(2)}(\vx)\right|\dif\vx\right)\left(\int_{\mathbb{R}^n}q^{(1)}_s(\vx)\dif\vx\right) \\
&\leq B_{\vy}\int_{\mathbb{R}^n}\left|q^{(1)}_s(\vx) - q^{(2)}_s(\vx)\right|\dif\vx = B_{\vy}\left\|q^{(1)}_s - q^{(2)}_s\right\|_{L^1},\\
\left\|\left(\langle I(\cdot,s),q^{(2)}_s\rangle_{L^2}\right)\left(q^{(1)}_s-q^{(2)}_s\right)\right\|_{L^1} &\leq \left|\langle I(\cdot,s),q^{(2)}_s\rangle_{L^2}\right|\left\|q^{(1)}_s-q^{(2)}_s\right\|_{L^1} \\
&= \left(\int_{\mathbb{R}^n}\left|I(\vx,s)\right|q^{(2)}_s(\vx)\dif\vx\right)|\left\|q^{(1)}_s-q^{(2)}_s\right\|_{L^1}\\
&\leq B_{\vy}\left(\int_{\mathbb{R}^n}q^{(2)}_s(\vx)\dif\vx\right)\left\|q^{(1)}_s-q^{(2)}_s\right\|_{L^1} = B_{\vy}\left\|q^{(1)}_s-q^{(2)}_s\right\|_{L^1}.
\end{aligned} 
\end{equation}
Substituting the three bounds derived above into~\eqref{eqn: triangle ineq intermediate step} and~\eqref{eqn: gronwall ineq intermediate} then gives us that 
\begin{equation}
\left\|r_t\right\|_{L^1} \leq \left\|r_0\right\|_{L^1} + \int_{0}^{t}3B_{\vy} s\left\|q^{(1)}_{s} - q^{(2)}_s\right\|_{L^1}\dif s \leq \int_{0}^{t}3B_{\vy} s\left\|r_s\right\|\dif s + \left\|r_0\right\|_{L^1}.    
\end{equation}
Applying Gr\"{o}nwall's inequality then implies that $\left\|r_t\right\|_{L^1} \leq \left\|r_0\right\|_{L^1}e^{\int_{0}^{t}3B_{\vy}s\dif s} = \left\|r_0\right\|_{L^1}e^{\frac{3}{2}B_{\vy}t^2}$, \emph{i.e.},
\begin{equation}
\TV\left(q^{(1)}(\cdot,t), q^{(2)}(\cdot,t)\right) = \frac{1}{2}\left\|r_t\right\|_{L^1} \leq \frac{1}{2}\left\|r_0\right\|_{L^1}e^{\frac{3}{2}B_{\vy}t^2} = e^{\frac{3}{2}B_{\vy}t^2}\TV\left(q^{(1)}(\cdot,0), q^{(2)}(\cdot,0)\right).    
\end{equation}
By taking $q^{(1)}(\cdot,t) := \widehat{q}_{\alpha,\vy,t}$ and $q^{(2)}(\cdot,t) := \widetilde{q}_{\alpha,\vy,t}$ in the stability inequality above, we obtain that
\begin{equation}
\TV\left(\widehat{q}_{\alpha,\vy,T},\widetilde{q}_{\alpha,\vy,T}\right) \leq e^{\frac{3}{2}B_{\vy}T^2}\TV\left(\widehat{q}_{\alpha,\vy,0},\widetilde{q}_{\alpha,\vy,0}\right) =: C^{(3)}_{\vy,T}\TV\left(\widehat{q}_{\alpha,\vy,0}, \widehat{q}_{\alpha,\vy}(\cdot,0)\right) = C^{(3)}_{\vy,T}\epsilon_{I},   
\end{equation}
where $\epsilon_I$ is the initial error introduced before Theorem~\ref{thm: bound on true and approx posterior} above and $C^{(3)}_{\vy,T} := e^{\frac{3}{2}B_{\vy}T^2}$ is some constant depending on $\vy, T$ only. This concludes our proof of the upper bound for the second term on the RHS of~\eqref{eqn: TV triangle intermediate ineq}.

Finally, substituting the two upper bounds proved in Appendix~\ref{subsubsec: first term of TV errors} and~\ref{subsubsec: second term of TV errors} above into~\eqref{eqn: TV triangle intermediate ineq} yields
$$\TV\left(\widehat{q}_{\alpha,\vy,T}, q_{\vy,0}\right) \leq C^{(2)}_{\vy}\sqrt{\frac{m_2^2}{T^2} + T^2\epsilon_{\vs}^2} + C^{(3)}_{\vy,T}\epsilon_{I},$$
which concludes our proof of Theorem~\ref{thm: bound on true and approx posterior}. In particular, for the case when the initial sampling error $\epsilon_I =0$, balancing the two terms in the last expression above also yields $\frac{m_2^2}{T^2} = T^2\epsilon_{\vs}^2$, \emph{i.e.}, $T = \sqrt{\frac{m_2}{\epsilon_s}}$ gives us the optimal upper bound 
$$
\TV(\widehat{q}_{\alpha,\vy,T},q_{\vy,0}) \leq C^{(2)}_{\vy}\sqrt{m_2\epsilon_{\vs}},
$$ 
which is proportional to the square root of the score matching error defined in Assumption~\ref{assump: score matching training error}.

%\end{proof}

\subsection{Proof of Theorem~\ref{thm: many-particle limit}}
\label{proof: mean-field limit}

Our proof of Theorem~\ref{thm: many-particle limit} is mainly based on arguments from propagation of chaos~\citep{sznitman1991topics,lacker2018mean}. Recall that 
$$
    \gamma^{N}_\tau(\vx,\beta) = \frac{1}{N}\sum_{i=1}^{N}\delta_{(\vx^{(i)}_\tau,\beta^{(i)}_\tau)}
$$ 
denotes the joint measured formed by the $N$ weighted particles $\left\{\left(\vx^{(i)}_\tau, \beta^{(i)}_\tau\right)\right\}_{i=1}^{N}$ given by~\eqref{eqn: PDE soln via ensemble of particles} and $\gamma_{\tau}$ is the joint probability distribution of the single weighted particle $(\vx_\tau,\beta_\tau)$ satisfying~\eqref{eqn: PDE soln via single particle}. 

Now we consider an auxiliary system of $N$ weighted particles $\{(\widetilde{\vx}^{(i)}_t,\widetilde{\beta}^{(i)}_t)\}_{i=1}^{N}$ sampled identically and independently from the single particle dynamics~\eqref{eqn: PDE soln via single particle}, \emph{i.e.}, 
\begin{equation}
\label{eqn: copies of N single particles} 
\begin{cases}
\dif \widetilde{\vx}^{(i)}_t &= \left(\widehat{\mH}(\widetilde{\vx}^{(i)}_t,t)-\alpha_tV(t)^2\nabla_{\vx}\mu_{\vy}(\widetilde{\vx}^{(i)}_t)\right)\dif t + V(t)\dif \vw^{(i)}_t,\\
\dif \widetilde{\beta}^{(i)}_t &= \left(U_{\alpha,\vy}(\widetilde{\vx}^{(i)}_t,t) - \alpha_t\widehat{\mH}(\widetilde{\vx}^{(i)}_t,t)^\intercal\nabla_{\vx}\mu_{\vy}(\widetilde{\vx}^{(i)}_t) - \alpha'_t \mu_{\vy}(\vx^{(i)}_t)\right)\widetilde{\beta}^{(i)}_t \dif t \\
&- \left(\displaystyle\int_{\mathbb{R}^n}\left(U_{\alpha,\vy}(\vx,t) - \alpha_t\widehat{\mH}(\vx,t)^\intercal\nabla_{\vx}\mu_{\vy}(\vx)-\alpha'_t \mu_{\vy}(\vx)\right)\left(P_{\beta}\gamma_t\right)(\vx)\dif \vx\right)\widetilde{\beta}^{(i)}_t \dif t.    
\end{cases}
\end{equation}
We note that the initial conditions of~\eqref{eqn: copies of N single particles} are given by $\widetilde{\vx}^{(i)}_0 = \vx^{(i)}_0 \sim \widehat{q}_{\alpha,\vy}(\cdot,0)$ and $\widetilde{\beta}^{(i)}_0 = 1$ for $i \in [N]$. Moreover, we have that the $\left(\vw^{(i)}_t\right)_{t \in [0,T]}$ is the same standard Brownian motion used in~\eqref{eqn: PDE soln via ensemble of particles} for any $i \in [N]$, which implies $\vx^{(i)}_t \equiv \widetilde{\vx}^{(i)}_t$ for any $i \in [N]$ and $t \in [0,T]$. Then we consider the joint empirical measure 
\begin{equation}
\label{eqn: defn of wasserstein 2 distance}
\widetilde{\gamma}^{N}_t(\vx,\beta) = \frac{1}{N}\sum_{i=1}^{N}\delta_{(\widetilde{\vx}^{(i)}_t,\widetilde{\beta}^{(i)}_t)}    
\end{equation}
formed by the $N$ weighted particles $\left\{\left(\widetilde{\vx}^{(i)}_t, \widetilde{\beta}^{(i)}_t\right)\right\}_{i=1}^{N}$ given by~\eqref{eqn: copies of N single particles}. 

Before we proceed, we establish the following upper bound:
\begin{lemma}
    The following upper bound on the absolute values of the weights holds for any $t \in [0,T]$
    \begin{equation}
    \label{eqn: upper bound on abs values of weights}
    \max_{i \in [N]}\left\{\left|\beta_t\right|,\left|\beta^{(i)}_t\right|,\left|\widetilde{\beta}^{(i)}_t\right|\right\} \leq \exp\left(B_{\vy}t^2\right).
    \end{equation}
\end{lemma}
\begin{proof}
    Below we will only prove the upper bound in~\eqref{eqn: upper bound on abs values of weights} above for the weight $\widetilde{\beta}^{(i)}_t$ governed by~\eqref{eqn: copies of N single particles}, as the same upper bound for $\beta^{(i)}_t$ governed by~\eqref{eqn: PDE soln via ensemble of particles} and $\beta_t$ governed by~\eqref{eqn: PDE soln via single particle} can be proved via the same procedure. By integrating from $0$ to $t$ on both sides of~\eqref{eqn: copies of N single particles} and applying the bound on $I$ provided in the statement of Theorem~\ref{thm: many-particle limit}, we have that
    \begin{equation}
    \label{eqn: integral bound on weight}
    \begin{aligned}
    \left|\widetilde{\beta}^{(i)}_t\right| &\leq \left|\int_{0}^{t}\left(U_{\alpha,\vy}\left(\widetilde{\vx}^{(i)}_{\tau},\tau\right) - \alpha_\tau\widehat{\mH}\left(\widetilde{\vx}^{(i)}_{\tau},\tau\right)^\intercal\nabla_{\vx}\mu_{\vy}\left(\widetilde{\vx}^{(i)}_{\tau}\right)-\alpha'_\tau\mu_{\vy}\left(\widetilde{\vx}^{(i)}_{\tau}\right)\right)\widetilde{\beta}^{(i)}_{\tau}\dif\tau\right| + \widetilde{\beta}^{(i)}_0\\
    &+\left|\int_{0}^{t}\left(\int_{\mathbb{R}^n}\left(U_{\alpha,\vy}(\vx,\tau) - \alpha_\tau\widehat{\mH}(\vx,\tau)^\intercal\nabla_{\vx}\mu_{\vy}(\vx) -\alpha'_\tau \mu_{\vy}\left(\vx\right)\right)\left(P_{\beta}\gamma_\tau\right)(\vx)\dif \vx\right)\widetilde{\beta}^{(i)}_{\tau}\dif\tau\right|\\
    &\leq \int_{0}^{t}\left|\tau I\left(\widetilde{\vx}^{(i)}_{\tau},\tau\right)\widetilde{\beta}^{(i)}_{\tau}\right|\dif\tau + \int_{0}^{t}\left|\tau\left(\int_{\mathbb{R}^n}I\left(\vx,\tau\right)\left(P_{\beta}\gamma_\tau\right)(\vx)\dif \vx\right)\widetilde{\beta}^{(i)}_{\tau}\right|\dif\tau + 1\\
    &\leq 2\int_{0}^{t}B_{\vy}\tau\left|\widetilde{\beta}^{(i)}_{\tau}\right|\dif\tau  + 1,
    \end{aligned}    
    \end{equation}
    where the last inequality above follows from the assumed upper bound on $I$ and Lemma~\ref{lem: derivation of single particle dynamics}, which shows that the weighted projection $P_{\beta}\gamma_{t}$ is a probability measure on $\mathbb{R}^n$ for any $t \in [0,T]$. Applying Gronwall's inequality to~\eqref{eqn: integral bound on weight} then yields the upper bound in~\eqref{eqn: upper bound on abs values of weights}, as desired.
\end{proof}

\begin{proof}[Proof of Theorem~\ref{thm: many-particle limit}]

    By recalling the definition of the Wasserstein-2 distance as follows
    \begin{equation}
    \label{eqn: defn of wasserstein-2 distance}
    \gW_{2}^2(\mu,\nu) := \inf_{\Gamma \in \Pi(\mu,\nu)}\left(\int_{\mathbb{R}^d \times \mathbb{R}^d}\|\vx-\vy\|_2^2\Gamma(\vx,\vy)\dif\vx\dif\vy\right),
    \end{equation}
    where $\Pi(\mu,\nu)$ denotes the set of couplings between any two distributions $\mu,\nu$ on $\mathbb{R}^d$ for fixed $d$, we can apply triangle inequality 
    $$
        \|a+b\|_2^2 \leq 2\left(\|a\|_2^2 + \|b\|_2^2\right), 
    $$ 
    and take expectation on both sides to deduce that for any fixed $t \in [0,T]$ and $\tau \in [0,t]$, the following inequality
    \begin{equation}
    \label{eqn: wasserstein triangle inequality}
    \E\left[\gW^2_2(\gamma^{N}_\tau, \gamma_\tau)\right] \leq 2\E\left[\gW^2_2(\gamma^{N}_\tau, \widetilde{\gamma}^{N}_\tau)\right] + 2\E\left[\gW^2_2(\widetilde{\gamma}^{N}_\tau, \gamma_\tau)\right]
    \end{equation}
    holds for any $N$. 
    
    Taking supremum with respect to $\tau \in [0,t]$ on both sides above then yields 
    \begin{equation}
    \label{eqn: wasserstein triangle inequality sup version}
    \sup_{\tau \in [0,t]}\E\left[\gW^2_2(\gamma^{N}_\tau, \gamma_\tau)\right] \leq 2\sup_{\tau \in [0,t]}\E\left[\gW^2_2(\gamma^{N}_\tau, \widetilde{\gamma}^{N}_\tau)\right] + 2\sup_{\tau \in [0,t]}\E\left[\gW^2_2(\widetilde{\gamma}^{N}_\tau, \gamma_\tau)\right].
    \end{equation}

    We then need to bound the two terms on the RHS of~\eqref{eqn: wasserstein triangle inequality}. 
    
    For the first term in~\eqref{eqn: wasserstein triangle inequality}, we note that the empirical measure 
    $$
        \frac{1}{N}\sum_{i=1}^{N}\delta_{(\vx^{(i)}_\tau,\beta^{(i)}_\tau),(\widetilde{\vx}^{(i)}_\tau,\widetilde{\beta}^{(i)}_\tau)}
    $$ 
    defined on $\mathbb{R}^{n+1} \times \mathbb{R}^{n+1}$ is a coupling between $\gamma^N_\tau$ and $\widetilde{\gamma}^N_\tau$ for any time $\tau \in [0,t]$. Setting $\Gamma$ in~\eqref{eqn: defn of wasserstein 2 distance} to be such a coupling then gives us the following upper bound on the expected Wasserstein-2 distance:
    \begin{equation}
        \label{eqn: coupling bound}
        \begin{aligned}
        \sup_{\tau \in [0,t]}\E\left[\gW^2_2(\gamma^{N}_\tau, \widetilde{\gamma}^{N}_\tau)\right] &\leq \sup_{\tau \in [0,t]}\left(\frac{1}{N}\sum_{i=1}^{N}\E\left[\left\|\vx^{(i)}_t - \widetilde{\vx}^{(i)}_t\right\|_2^2 + \left|\beta^{(i)}_t - \widetilde{\beta}^{(i)}_t\right|^2\right]\right)\\
        &= \frac{1}{N}\sum_{i=1}^{N}\sup_{\tau \in [0,t]}\E\left[\left|\beta^{(i)}_t - \widetilde{\beta}^{(i)}_t\right|^2\right].     
        \end{aligned}  
    \end{equation}
    where the equality above follows from the observation $\vx^{(i)}_t \equiv \widetilde{\vx}^{(i)}_t$ for any $i \in [N]$ and $t \in [0,T]$. 
    
    Below, we use 
    $$
        L(\vx,t):=U_{\alpha,\vy}(\vx,t) - \alpha_t\widehat{\mH}(\vx,t)^\intercal\nabla_{\vx}\mu_{\vy}(\vx)-\alpha'_t\mu_{\vy}(\vx)
    $$ 
    to denote the drift function appearing in the dynamics~\eqref{eqn: PDE soln via ensemble of particles} and~\eqref{eqn: copies of N single particles}. By plugging in the choices $s(t) = 1, \sigma(t) = t$ stated in Theorem~\ref{thm: many-particle limit}, we then have 
    \begin{equation}
    \label{eqn: computation of drift function under particular choices}
    L =U_{\alpha,\vy} - \alpha_t\widehat{\mH}^\intercal\nabla_{\vx}\mu_{\vy} -\alpha'_t\mu_{\vy} = t\left(\alpha^2_t\left\|\nabla_{\vx}\mu_{\vy}\right\|_2^2 - \alpha_t\Delta_{\vx}\mu_{\vy}\right) - 2t\alpha_t\vphi^\intercal_{\theta}\nabla_{\vx}\mu_{\vy} -\alpha'_t\mu_{\vy} = tI,
    \end{equation}
    where $I = I(\vx,t)$ is defined in the statement of Theorem~\ref{thm: many-particle limit}.

    Now we return to bound the RHS of~\eqref{eqn: coupling bound}. By taking the difference between the two dynamics~\eqref{eqn: PDE soln via ensemble of particles} and~\eqref{eqn: copies of N single particles} and applying triangle inequality, we then plug in $\widetilde{\vx}^{(i)}_t \equiv \vx^{(i)}_t$ to obtain the following decomposed upper bound for any $\tau' \in [0,\tau]$ with fixed $\tau \in [0,t]$ and $i \in [N]$:
    %By recalling the notation $\nu^{N}_{\vy,t} = \frac{1}{N}\sum_{i=1}^{N}\beta^{(i)}_t \delta_{\vx^{(i)}_t} = P_{\beta}\gamma^{N}_t \ (t \in [0,T])$ used for denoting the projected empirical measure formed via weighted particles governed by~\eqref{eqn: PDE soln via ensemble of particles}, we define $\widetilde{\nu}^{N}_{\vy,t} = \frac{1}{N}\sum_{i=1}^{N}\widetilde{\beta}^{(i)}_t \delta_{\widetilde{\vx}^{(i)}_t} = P_{\beta}\widetilde{\gamma}^{N}_t$ to be the projected empirical measure formed via weighted particles governed by~\eqref{eqn: copies of N single particles}. $P_{\beta}\gamma_t(\cdot) = \widehat{q}_{\vy}(\cdot,t)$ proved in Lemma~\ref{lem: derivation of single particle dynamics} and 
    \begin{equation}
    \label{eqn: bound on difference between derivatives}
    \begin{aligned}
    \left|\frac{\dif}{\dif \tau'}\beta^{(i)}_{\tau'} - \frac{\dif}{\dif \tau'}\widetilde{\beta}^{(i)}_{\tau'}\right| &\leq \left|L\left(\vx^{(i)}_{\tau'},\tau'\right)\left(\beta^{(i)}_{\tau'} - \widetilde{\beta}^{(i)}_{\tau'}\right)\right|\\
    &+\left|\left(\int_{\mathbb{R}^n}L\left(\vx,\tau'\right)\left(P_{\beta}\gamma_{\tau'}\right)(\vx)\dif \vx\right)\left(\beta^{(i)}_{\tau'} - \widetilde{\beta}^{(i)}_{\tau'}\right)\right|\\
    &+\left|\int_{\mathbb{R}^{n+1}}\beta L\left(\vx,\tau'\right)\left(\gamma^N_{\tau'}(\vx,\beta) - \gamma_{\tau'}(\vx,\beta)\right)\dif\vx\dif\beta\right|\left|\beta^{(i)}_{\tau'}\right|\\
    &\leq  2B_{\vy}\tau'\left|\beta^{(i)}_{\tau'} - \widetilde{\beta}^{(i)}_{\tau'}\right|\\ &+\left|\int_{\mathbb{R}^{n+1}}\beta I\left(\vx,\tau'\right)\left(\gamma^N_{\tau'}(\vx,\beta) - \gamma_{\tau'}(\vx,\beta)\right)\dif\vx\dif\beta\right|\tau'\exp\left(B_{\vy}\tau'^2\right),
    \end{aligned}    
    \end{equation}
    where the last inequality above follows from~\eqref{eqn: upper bound on abs values of weights} and assumed upper bound on the function $I=\frac{1}{t}L$. 
    
    Furthermore, we recall the following property of the Wasserstein distances $\gW_1$ and $\gW_2$:
    \begin{equation}
    \label{eqn: bound on Wasserstein-1 distance}
    \gW_{1}(\mu,\nu) := \sup_{g:\mathbb{R}^n \rightarrow \mathbb{R}, \ \mathrm{Lip}(g) \leq 1}\int_{\mathbb{R}^n}g(\vx)\left(\mu(\vx)-\nu(\vx)\right)\dif\vx \leq \gW_{2}(\mu,\nu).
    \end{equation}

    From the assumed upper bound on $\mathrm{Lip}\left(I\right)$ given in Theorem~\ref{thm: many-particle limit}, we have $\mathrm{Lip}\left(\frac{1}{B_{\vy}}I\right) \leq 1$. Setting $$g(\vx,\beta):=\frac{\beta I(\vx,\tau')}{e^{B_{\vy}\tau'^2}B_{\vy}},$$ $\mu:=\gamma^N_{\tau'}$, and $\nu:=\gamma_{\tau'}$ in~\eqref{eqn: bound on Wasserstein-1 distance} above for any $\tau' \in [0,\tau]$ then implies
    \begin{equation}
    \label{eqn: bound on second part via wasserstein 1-2}
    \begin{aligned}
    \left|\int_{\mathbb{R}^{n}}\beta I\left(\vx,\tau'\right)\left(\gamma^N_{\tau'}(\vx,\beta) - \gamma_{\tau'}(\vx,\beta)\right)\dif\vx\dif\beta\right| &\leq B_{\vy}e^{B_{\vy}\tau'^2}\gW_1\left(\gamma^N_{\tau'},\gamma_{\tau'}\right) \\
    &\leq B_{\vy}e^{B_{\vy}\tau'^2}\gW_2\left(\gamma^N_{\tau'},\gamma_{\tau'}\right).
    \end{aligned}    
    \end{equation}
    Substituting~\eqref{eqn: bound on second part via wasserstein 1-2} into~\eqref{eqn: bound on difference between derivatives}, squaring on both sides and applying AM-GM inequality indicate that for any $\tau' \in [0,\tau]$ and $i \in [N]$:
    \begin{equation}
    \begin{aligned}
    \left|\frac{\dif}{\dif \tau'}\beta^{(i)}_{\tau'} - \frac{\dif}{\dif \tau'}\widetilde{\beta}^{(i)}_{\tau'}\right|^2 &\leq  \left(2B_{\vy}\tau'\left|\beta^{(i)}_{\tau'} - \widetilde{\beta}^{(i)}_{\tau'}\right| + B_{\vy}\tau'\exp\left(2B_{\vy}\tau'^2\right)\gW_2\left(\gamma^N_{\tau'},\gamma_{\tau'}\right)\right)^2 \\
    &\leq 8B_{\vy}^2\tau'^2\left|\beta^{(i)}_{\tau'} - \widetilde{\beta}^{(i)}_{\tau'}\right|^2 + 2B_{\vy}^2\tau'^2\exp\left(4B_{\vy}\tau'^2\right)\gW^2_2\left(\gamma^N_{\tau'},\gamma_{\tau'}\right).
    \end{aligned}    
    \end{equation}

    Integrating from $\tau'=0$ to $\tau'=\tau$ on both sides above and applying Cauchy-Schwarz inequality imply that for any $\tau \in [0,t]$ and $i \in [N]$:
    \begin{equation}
    \label{eqn: intermediate upper bound on difference between weights}
    \begin{aligned}
    \left|\beta^{(i)}_{\tau} - \widetilde{\beta}^{(i)}_{\tau}\right|^2 &= \left|\int_{0}^{\tau}\left(\frac{\dif}{\dif \tau'}\beta^{(i)}_{\tau'} - \frac{\dif}{\dif \tau'}\widetilde{\beta}^{(i)}_{\tau'}\right)\dif\tau'\right|^2 \leq \tau\left(\int_{0}^{\tau}\left|\frac{\dif}{\dif \tau'}\beta^{(i)}_{\tau'} - \frac{\dif}{\dif \tau'}\widetilde{\beta}^{(i)}_{\tau'}\right|^2\dif\tau'\right)\\
    &\leq 8B_{\vy}^2\tau\int_{0}^{\tau}\tau'^2\left|\beta^{(i)}_{\tau'} - \widetilde{\beta}^{(i)}_{\tau'}\right|^2\dif\tau' \\
    &+ 2\tau\int_{0}^{\tau}B_{\vy}^2\tau'^2\exp\left(4B_{\vy}\tau'^2\right)\gW^2_2\left(\gamma^N_{\tau'},\gamma_{\tau'}\right)\dif\tau'.\\
    \end{aligned}    
    \end{equation}

    Applying Gronwall's inequality to the function $\frac{1}{\tau'}\left|\beta^{(i)}_{\tau'} - \widetilde{\beta}^{(i)}_{\tau'}\right|^2$ in~\eqref{eqn: intermediate upper bound on difference between weights} above then yields 
    \begin{equation}
    \label{eqn: Gronwall for divided function}
    \begin{aligned}
    \frac{1}{\tau}\left|\beta^{(i)}_{\tau} - \widetilde{\beta}^{(i)}_{\tau}\right|^2 &\leq 2\left(\int_{0}^{\tau}B_{\vy}^2\tau'^2\exp\left(4B_{\vy}\tau'^2\right)\gW^2_2\left(\gamma^N_{\tau'},\gamma_{\tau'}\right)\dif\tau'\right)e^{\int_{0}^{\tau}8B_{\vy}^2\tau'^3\dif\tau'}.
    \end{aligned}    
    \end{equation}

    Then we multiply $\tau$ and take the expectation on both sides of~\eqref{eqn: Gronwall for divided function}. A direct application of Fubini's Theorem then indicates that for any $i \in [N]$ and $\tau \in [0,t]$:
    \begin{equation}
    \label{eqn: bound on expectation via integral}
    \begin{aligned}
    \E\left[\left|\beta^{(i)}_{\tau} - \widetilde{\beta}^{(i)}_{\tau}\right|^2\right] 
    &\leq 2\tau e^{2B_{\vy}^2\tau^4}\int_{0}^{\tau}B_{\vy}^2\tau'^2\exp\left(4B_{\vy}\tau'^2\right)\E\left[\gW^2_2\left(\gamma^N_{\tau'},\gamma_{\tau'}\right)\right]\dif\tau'\\
    &\leq 2B_{\vy}^2\tau^3e^{2B_{\vy}^2\tau^4 + 4B_{\vy}\tau^2}\int_{0}^{\tau}\sup_{\tau''\in[0,\tau']}\E\left[\gW^2_2\left(\gamma^N_{\tau''},\gamma_{\tau''}\right)\right]\dif\tau'.
    \end{aligned}    
    \end{equation}

    Taking supremum with respect to $\tau \in [0,t]$ on both sides of~\eqref{eqn: bound on expectation via integral} further implies 
    \begin{equation}
    \label{eqn: bound on sup expectation via integral}
    \begin{aligned}
    \sup_{\tau \in [0,t]}\E\left[\left|\beta^{(i)}_{\tau} - \widetilde{\beta}^{(i)}_{\tau}\right|^2\right] \leq 2B_{\vy}^2t^3e^{2B_{\vy}^2t^4 + 4B_{\vy}t^2}\int_{0}^{t}\sup_{\tau' \in[0,\tau]}\E\left[\gW^2_2\left(\gamma^N_{\tau'},\gamma_{\tau'}\right)\right]\dif\tau,    
    \end{aligned}    
    \end{equation}
    for any $i \in [N]$ and $t \in [0,T]$. 
    
    Substituting~\eqref{eqn: bound on sup expectation via integral} above into~\eqref{eqn: coupling bound} and then~\eqref{eqn: wasserstein triangle inequality sup version} indicates
    \begin{equation}
    \label{eqn: pre gronwall for the final bound}
    \begin{aligned}
    \sup_{\tau \in [0,t]}\E\left[\gW^2_2(\gamma^{N}_\tau, \gamma_\tau)\right] \leq &4B_{\vy}^2t^3e^{2B_{\vy}^2t^4 + 4B_{\vy}t^2}\int_{0}^{t}\sup_{\tau' \in[0,\tau]}\E\left[\gW^2_2\left(\gamma^N_{\tau'},\gamma_{\tau'}\right)\right]\dif\tau \\
    + &2\sup_{\tau \in [0,t]}\E\left[\gW^2_2(\widetilde{\gamma}^{N}_\tau, \gamma_\tau)\right]\\
    \leq &\int_{0}^{t}4B_{\vy}^2T^3e^{2B_{\vy}^2T^4 + 4B_{\vy}T^2}\sup_{\tau' \in[0,\tau]}\E\left[\gW^2_2\left(\gamma^N_{\tau'},\gamma_{\tau'}\right)\right]\dif\tau \\
    + &2\sup_{\tau \in [0,t]}\E\left[\gW^2_2(\widetilde{\gamma}^{N}_\tau, \gamma_\tau)\right],
    \end{aligned}
    \end{equation}
    for any $t \in [0,T]$. 
    
    Applying Gronwall's inequality again to the function $\sup_{\tau \in [0,t]}\E\left[\gW^2_2(\gamma^{N}_\tau, \gamma_\tau)\right]$ further implies that
    \begin{equation}
    \label{eqn: final upper bound}
    \sup_{\tau \in [0,t]}\E\left[\gW^2_2(\gamma^{N}_\tau, \gamma_\tau)\right] \leq 2\exp\left(4B_{\vy}^2T^4e^{2B_{\vy}^2T^4 + 4B_{\vy}T^2}\right)\sup_{\tau \in [0,t]}\E\left[\gW^2_2(\widetilde{\gamma}^{N}_\tau, \gamma_\tau)\right]
    \end{equation}
    for any $t \in [0,T]$. 
    
    By setting $t=T$ in~\eqref{eqn: final upper bound} above and taking the limit $N \rightarrow \infty$, we then have 
    $$\lim_{N \rightarrow \infty}\E\left[\gW^2_2(\gamma^{N}_\tau, \gamma_\tau)\right] = \lim_{N \rightarrow \infty}\E\left[\gW^2_2(\widetilde{\gamma}^{N}_\tau, \gamma_\tau)\right] =0,$$
    for any $\tau \in [0,T]$ with $T$ fixed, where the last equality above follows from Lemma~\ref{lem: derivation of single particle dynamics} and the law of large numbers (See, for instance, \cite[Corollary 2.14]{lacker2018mean}). This concludes our proof. 
\end{proof}

\begin{remark}
We note that one may also adopt similar arguments used in~\citep{domingo2020mean} to prove existence and uniqueness of solutions to the SDE systems~\eqref{eqn: PDE soln via ensemble of particles} and~\eqref{eqn: PDE soln via single particle}. In fact, such type of mean field analysis based on arguments from propogation of chaos have been widely adopted for studying different types of PDEs arising from subfields of not only physical sciences but also data sciences, such as fluid dynamics~\citep{goodman1990convergence}, kinetic theory~\citep{carrillo2021wasserstein,borghi2025wasserstein}, theory of two layer neural networks~\citep{mei2019mean,hu2021mean}, ensemble-based sampling and variational inference~\citep{lu2019scaling, kelly2014well,schillings2017analysis,schillings2018convergence,ding2021ensemble1,ding2021ensemble2}. For some good reference on related mathematical models, one may refer to~\citep{muntean2016macroscopic}. Therefore, it would be of independent interest to investigate whether we can develop more refined mathematical theory for the two sampling algorithms proposed in this paper by combining perspectives from gradient flows or numerical analysis. Moreover, it would also be interesting to investigate how existing mathematical theory~\citep{eberle2006convergence,schweizer2012non,eberle2013quantitative,beskos2014error,beskos2014stability,beskos2016convergence,giraud2017nonasymptotic} developed for SMC can be applied to analyze Algorithm~\ref{alg:sde} and Algorithm~\ref{alg:edm pf-ode + corrector} that we proposed here. 
\end{remark}

\section{Additional Implementation Details for Linear Inverse Problems}
\label{sec: implementation detail app}
\subsection{Datasets, model checkpoints and inverse problem setups}
\label{sec: implementation detail on dataset model and problem}
\paragraph{Data usage} We mainly test our methods and the baseline methods on the FFHQ-256~\citep{karras2019style} dataset and the ImageNet-256~\citep{deng2009imagenet} dataset. All images used for the tests in this paper are in RGB. For FFHQ-256, the 100 testing images were selected to be the first 100 images in the dataset, whoses indexes range from $00000$ to $00099$. For ImageNet-256, the 100 testing images were selected to be the first 100 images in the ImageNet-1k validation set. 

%However, when we further test only the two algorithms AFDPS-SDE and AFDPS-ODE proposed in this paper and exhibit their performance in Appendix~\ref{sec: experiment results detail app}, we have enlarged the testing dataset to be the whole FFHQ-dataset and ImageNet 1k-validation dataset. 

\paragraph{Model checkpoints}
The two pretrained score functions for the FFHQ-256 and the ImageNet-256 datasets used in this paper were directly taken from the ones used in~\citep{chung2022diffusion}, which are available in the following Google Drive \footnote{\href{https://drive.google.com/drive/folders/1jElnRoFv7b31fG0v6pTSQkelbSX3xGZh}{Pretrained score functions used in~\citep{chung2022diffusion}}}. However, since these checkpoints were all trained based on the DDPM formulation~\citep{ho2020denoising}, we adpoted the same transformation used in~\citep{wu2024principled} to convert the pretrained score function from the DDPM formulation to the EDM formulation~\citep{karras2022elucidating}. One may refer to the ``Preconditioning'' subsection in Appendix C.2 of~\citep{wu2024principled} for an explicit formula of the transformation deployed here. 

\paragraph{Inverse problem setups} Below we provide a discussion on the mathematical formulations of the four inverse problems we tested on here. 

\subparagraph{Super-resolution} The forward model in~\eqref{eqn: BIP setup} associated with the super-resolution problem we test on here can be written as 
$$\vy = P_{f}\vx + \vn$$
where $P_{f} \in \mathbb{R}^{\frac{n}{f} \times n}$ implements a block averaging filter that downscales each image by a factor of $f$ and $\vn \sim \gN(\vzero, 0.2 \mI_{\frac{n}{f}})$. Using similar setups as many previous work~\citep{chung2022diffusion,kawar2022denoising,wu2024principled} on solving inverse problems via diffusion models, here we pick $f =4$. 

\subparagraph{Gaussian and motion deblurring} The forward model associated with any deblurring problem can be summarized as 
$$\vy = B_{k}\vx + \vn$$
where $\vn \sim \gN(\vzero, 0.2 \mI_{n})$ and $B_{k} \in \mathbb{R}^{n \times n}$ is a circulant matrix that realizes a convolution with the kernel $k$ under circular boundary condition. Again, we adopt the same settings used in most previous work~\citep{chung2022diffusion,kawar2022denoising,wu2024principled}.

Specifically, for the Gaussian deblurring problem, the convolutional kernel $k$ is fixed to be a Gaussian kernel of standard deviation $3.0$ and size $61 \times 61$. For the motion deblurring problem, the kernel $k$ is randomly generated via code used in previous work~\citep{kawar2022denoising,wu2024principled}, where the size is chosen to be $61 \times 61$ and the intensity is set to be $0.5$. In order to ensure a fair comparison, we use the same motion kernel $k$ for each image across different methods. 

\subparagraph{Box inpainting} The forward model for the box inpainting problem is given by 
$$\vy = D \vx +\vn$$
where $\vn \sim \gN(\vzero, 0.2 \mI_{n})$ and $D$ is a diagonal matrix with either $0$ or $1$ on its diagonal. In particular, here we choose $D$ such that a centered square patch of size $64 \times 64$ (\emph{i.e.}, the side length is a quarter of the original image's side length) is masked out.

\subsection{Implementation details of AFDPS and all baseline methods}
Regarding computing resources, all experiments included in this paper were conducted on NVIDIA RTX A100 and A6000 GPUs. A major part of the code implementing Algorithm~\ref{alg:sde} and~\ref{alg:edm pf-ode + corrector} in this paper were adapted from the following Github repository\footnote{\href{https://github.com/zihuiwu/PnP-DM-public}{Source code for~\citep{wu2024principled}}}. Specifically, we used the same numerical discretization as that of the EDM framework~\citep{karras2019style}, which is also deployed in~\citep{wu2024principled}. One major difference is that we had tuned the terminal time to be $T=8$ for both AFDPS-SDE (Alg.~\ref{alg:sde}) and AFDPS-ODE (Alg.~\ref{alg:edm pf-ode + corrector}), while $T$ is set to be $80$ for both the SGS-EDM method~\citep{wu2024principled} and the original EDM framework~\citep{karras2022elucidating}. Moreover, we increased the number of discretized timesteps as our methods avoids running multiple backward diffusion processes for different iterations. Specifically, for AFDPS-SDE the number particles and discretized timesteps were set to be $10$ and $2000$, respectively. For the AFDPS-ODE method, in order to control the total number of evaluations (NFEs), we set the number of particles, discretized timesteps and number of corrector steps at each time to be $5$, $1000$ and $4$, respectively. Moreover, for both AFDPS-SDE (Algorithm~\ref{alg:sde}) and AFDPS-ODE (Algorithm~\ref{alg:edm pf-ode + corrector}), we save computational cost by skipping the resampling step specified in Algorithm~\ref{alg:resampling} in our implementation, which allows us to implement the dynamics of the particles' positions and weights in a parallel way. Finally, we return the particle associated with the largest weight as our best estimator of the recovered image. Given that we already take the logarithm of the weights in both AFDPS-SDE and AFDPS-ODE, they are guaranteed to remain numerically stable as time increases. 

Here we further elaborate on the implementation details associated with the baseline methods. One thing to note is that two extra baselines are included in the extended numerical results presented in Tables~\ref{tab:ffhq result} and~\ref{tab:imagenet result} above. The following list provides an extended summary of these baselines and how we choose the parameters:

\begin{itemize}[leftmargin=1.5em, itemsep=0pt, topsep=0pt, parsep=0pt]
    \item \emph{DPS~\citep{chung2022diffusion}}: a method that performs posterior sampling by guiding the reverse diffusion process with manifold-constrained gradients derived from the measurement likelihood, enabling efficient inference in general noisy (non)linear inverse problems. We adopt most parameters usedin the default setting. The only difference is that we increase the number of discretized timesteps from $1000$ to $1500$, which helps make the method more tolerant of problems with higher observational noises
    
    \item \emph{DCDP~\citep{li2024decoupled}}: a framework that alternates between data-consistent reconstruction and diffusion-based purification, which decouples data fidelity and prior sampling to improve flexibility and performance in image restoration tasks. In order to make the DCDP method adaptive to problems with higher observational noise, we change their settings by picking the number of iterations involved in both the data-reconstruction step and the diffusion-based purification step to be $100$. Regarding the learning rates used for the data-reconstruction step, we have tuned them to yield the best possible performance. Specifically, the learning rates for the Gaussian deblurring, box inpainting, motion deblurring and super-resolution problems were set to be $10,7,10$ and $3$, respectively. 
    
    \item \emph{SGS-EDM~\citep{wu2024principled}}: a method that couples a split Gibbs sampler with a diffusion model, interpreting posterior inference as alternating between likelihood-based updates and Gaussian denoising via a learned generative prior. For the SGS-EDM method, we adopt the default setting used in~\citep{wu2024principled}.
    
    \item \emph{FK-Corrector~\citep{skreta2025feynman}}: a method that uses the Feynman-Kac formula to design corrector steps within a sequential Monte Carlo framework, improving the accuracy of samples from forward diffusion trajectories. We use the same set of parameters deployed in the AFDPS-SDE method by setting the number of particles and discretized timesteps to be $10$ and $2000$ as well, which ensures a fair comparison. 
    
    \item \emph{PF-SMC-DM~\citep{dou2024diffusion}}: a framework that formulates posterior sampling as a particle filtering problem, combining sequential Monte Carlo with diffusion models for efficient inference in high-dimensional spaces. Again, to ensure a fair comparison, we increase the number of particles and discretized timesteps to be $10$ and $2000$ for PF-SMC-DM as well. 
    
\end{itemize}

\section{Additional Experimental Results and Discussions}
\label{sec: experiment results detail app}

In this section, we provide additional experimental results and detailed qualitative comparisons between our proposed methods and existing baselines.

\paragraph{Summary.}

Across the diverse inverse problems evaluated on FFHQ-256 and ImageNet-256 (detailed in Table~\ref{tab:ffhq result} and Table~\ref{tab:imagenet result}), the AFDPS framework consistently delivers strong results. The AFDPS-SDE variant, in particular, frequently distinguishes itself by producing visually compelling outcomes, excelling in the generation of sharp details and fine textures that contribute to high perceptual quality. This is evident in Figures~\ref{fig:gaussian-deblurring-grid}-\ref{fig:box-inpainting-grid}, where AFDPS-SDE's reconstructions often appear more intricate and realistic. The AFDPS-ODE variant also provides coherent results, which are typically characterized by a notable smoothness. For tasks where capturing the utmost detail and textural accuracy is paramount, AFDPS-SDE often provides a particularly effective solution, frequently leading in or strongly competing for the best perceptual metrics (LPIPS).

\paragraph{Gaussian Deblurring.}
In Gaussian deblurring, AFDPS-SDE showcases its ability to produce perceptually rich outputs, achieving the best LPIPS on ImageNet-256 (0.3925) and a competitive LPIPS on FFHQ-256 (0.2580). Figure~\ref{fig:gaussian-deblurring-grid} highlights SDE's strength in rendering sharp, defined textures like the dog's fur (ImageNet, row 2). Concurrently, AFDPS-ODE achieves high PSNR on both datasets and the best LPIPS on FFHQ-256 (24.98 PSNR, 0.2560 LPIPS), delivering notably clean and smooth outputs, for example, on the baby's facial skin (FFHQ, row 2).

\paragraph{Motion Deblurring.}
For motion deblurring, AFDPS-SDE demonstrates strong perceptual quality, securing the best LPIPS score (0.2869) on FFHQ-256, while PF-SMC-DM leads in PSNR. Figure~\ref{fig:motion-deblurring-grid} emphasizes SDE's proficiency in transforming blurred images into sharp, detailed reconstructions, meticulously recovering fine details like individual hair strands in FFHQ portraits (e.g., row 5). AFDPS-ODE also effectively removes blur, yielding coherent results, typically with a characteristically smoother finish.

\paragraph{Super-Resolution.}
AFDPS-SDE stands out in super-resolution, achieving the best PSNR and LPIPS scores on both FFHQ-256 (22.96 PSNR, 0.3063 LPIPS) and ImageNet-256 (20.97 PSNR, 0.4643 LPIPS). Figure~\ref{fig:super-resolution-grid} compellingly shows SDE generating sharp, highly detailed images from severely degraded inputs, adeptly reconstructing fine facial features (FFHQ, row 2 and 5) and intricate object textures like butterfly patterns (ImageNet, row 2). AFDPS-ODE also provides coherent upscaled outputs, especially for FFHQ dataset, reaffirming the metrics in the tables.

\paragraph{Box Inpainting.}
In box inpainting combined with denoising, AFDPS-SDE shows robust performance, securing the highest PSNR on ImageNet-256 (23.15). Figure~\ref{fig:box-inpainting-grid} highlights SDE's ability to generate detailed and realistically textured inpainted regions, such as the intricate dog fur (ImageNet, row 1) or sharp keyboard key structures (ImageNet, row 4). AFDPS-ODE also performs strongly, achieving best LPIPS on both datasets (FFHQ: 0.1969, ImageNet: 0.2716) and best PSNR on FFHQ (25.73), producing notably smooth and coherent fills, like seamless facial features (FFHQ, row 1).

\newpage
\begin{figure}[H]
    \centering
    % First 5 rows (Left panel)
    \begin{minipage}[t]{0.49\textwidth}
    \makebox[0.24\textwidth]{\small Measurement}
    \makebox[0.24\textwidth]{\small Ours (SDE)}
    \makebox[0.24\textwidth]{\small Ours (ODE)}
    \makebox[0.24\textwidth]{\small Ground Truth}
        
    \includegraphics[width=0.24\textwidth]{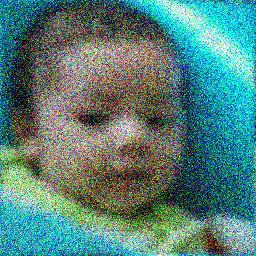}
    \includegraphics[width=0.24\textwidth]{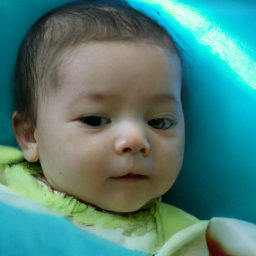}
    \includegraphics[width=0.24\textwidth]{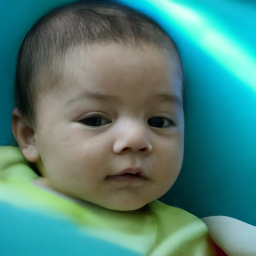}
    \includegraphics[width=0.24\textwidth]{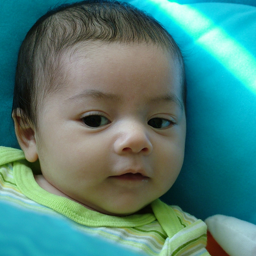}

    \includegraphics[width=0.24\textwidth]{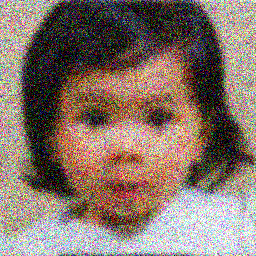}
    \includegraphics[width=0.24\textwidth]{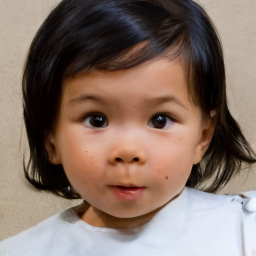}
    \includegraphics[width=0.24\textwidth]{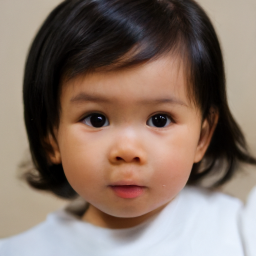}
    \includegraphics[width=0.24\textwidth]{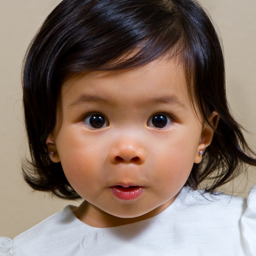}

    \includegraphics[width=0.24\textwidth]{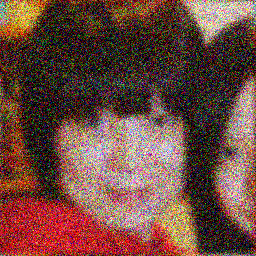}
    \includegraphics[width=0.24\textwidth]{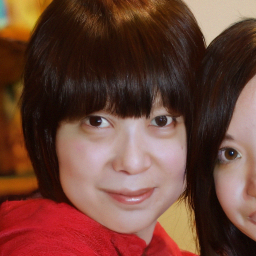}
    \includegraphics[width=0.24\textwidth]{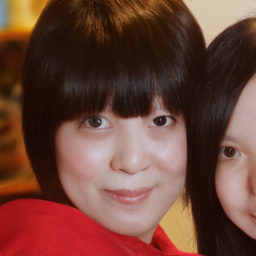}
    \includegraphics[width=0.24\textwidth]{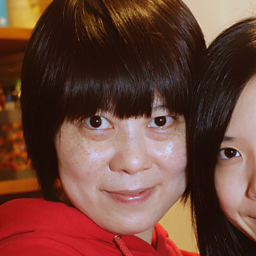}

    \includegraphics[width=0.24\textwidth]{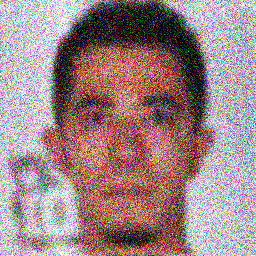}
    \includegraphics[width=0.24\textwidth]{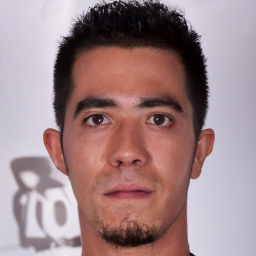}
    \includegraphics[width=0.24\textwidth]{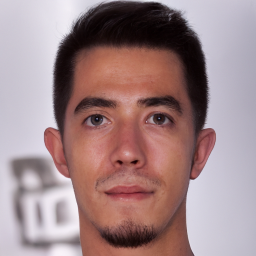}
    \includegraphics[width=0.24\textwidth]{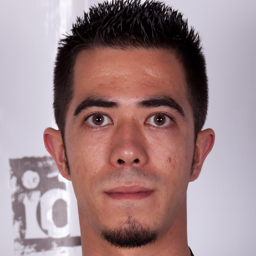}

    \includegraphics[width=0.24\textwidth]{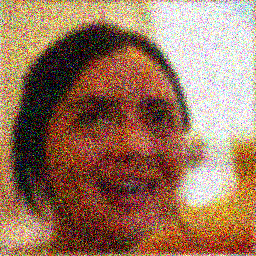}
    \includegraphics[width=0.24\textwidth]{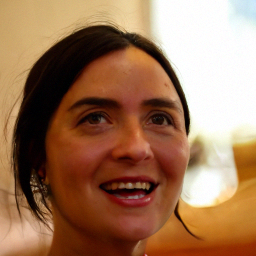}
    \includegraphics[width=0.24\textwidth]{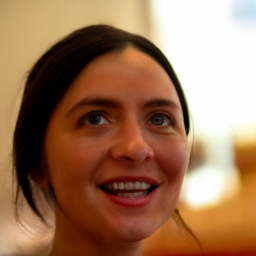}
    \includegraphics[width=0.24\textwidth]{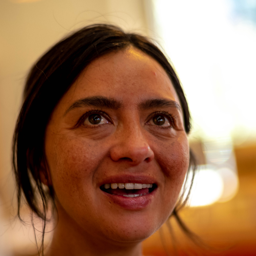}
    \end{minipage}
    \hfill
    % Second 5 rows (right panel)
    \begin{minipage}[t]{0.49\textwidth}
    \makebox[0.24\textwidth]{\small Measurement}
    \makebox[0.24\textwidth]{\small Ours (SDE)}
    \makebox[0.24\textwidth]{\small Ours (ODE)}
    \makebox[0.24\textwidth]{\small Ground Truth}

    \includegraphics[width=0.24\textwidth]{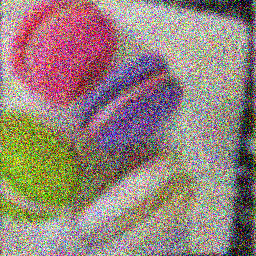}
    \includegraphics[width=0.24\textwidth]{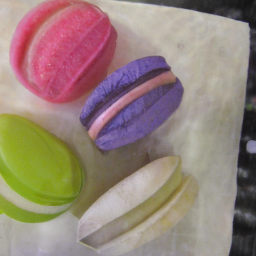}
    \includegraphics[width=0.24\textwidth]{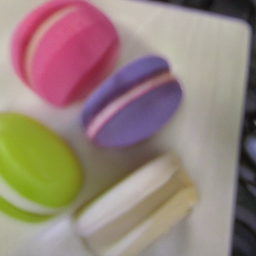}
    \includegraphics[width=0.24\textwidth]{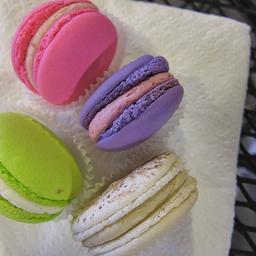}

    \includegraphics[width=0.24\textwidth]{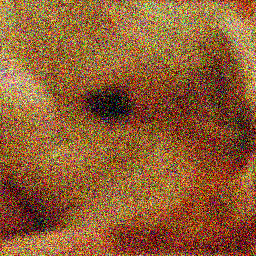}
    \includegraphics[width=0.24\textwidth]{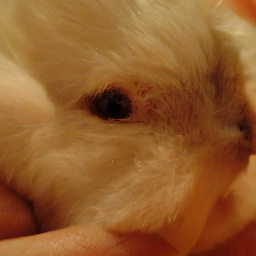}
    \includegraphics[width=0.24\textwidth]{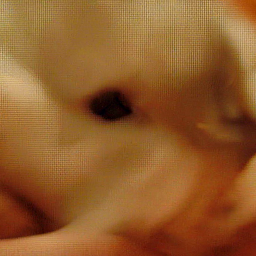}
    \includegraphics[width=0.24\textwidth]{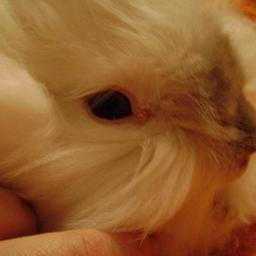}

    \includegraphics[width=0.24\textwidth]{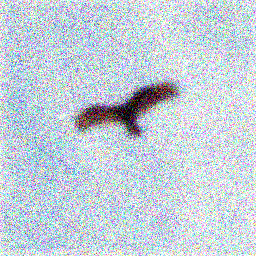}
    \includegraphics[width=0.24\textwidth]{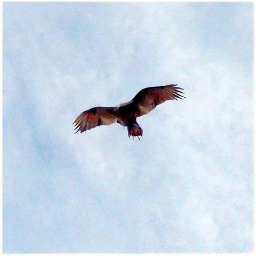}
    \includegraphics[width=0.24\textwidth]{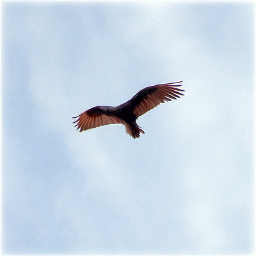}
    \includegraphics[width=0.24\textwidth]{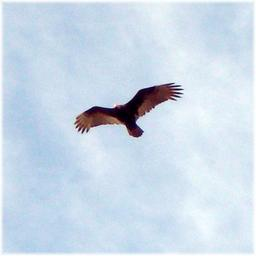}

    \includegraphics[width=0.24\textwidth]{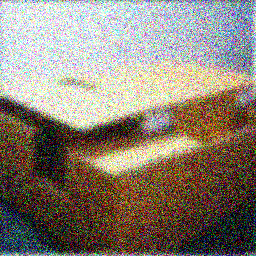}
    \includegraphics[width=0.24\textwidth]{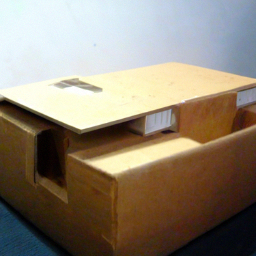}
    \includegraphics[width=0.24\textwidth]{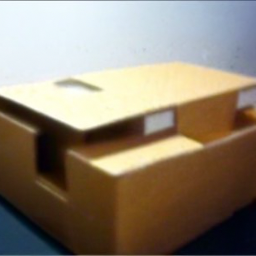}
    \includegraphics[width=0.24\textwidth]{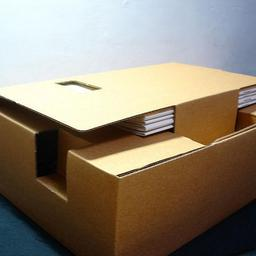}

    \includegraphics[width=0.24\textwidth]{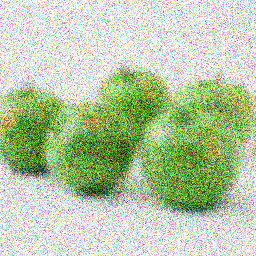}
    \includegraphics[width=0.24\textwidth]{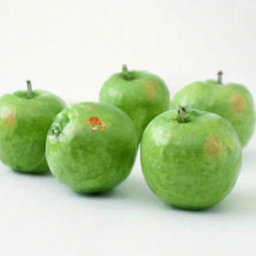}
    \includegraphics[width=0.24\textwidth]{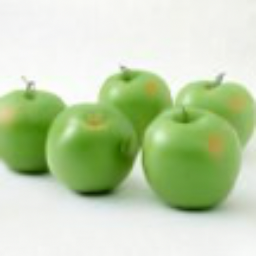}
    \includegraphics[width=0.24\textwidth]{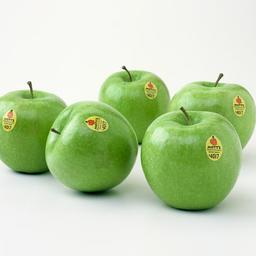}

    \end{minipage}
    \caption{Additional visual examples for the Gaussian deblurring problem on FFHQ and ImageNet.}
    \label{fig:gaussian-deblurring-grid}
\end{figure}

\begin{figure}[H]
    \centering
    
    % First 5 rows (Left panel)
    \begin{minipage}[t]{0.49\textwidth}
    \makebox[0.24\textwidth]{\small Measurement}
    \makebox[0.24\textwidth]{\small Ours (SDE)}
    \makebox[0.24\textwidth]{\small Ours (ODE)}
    \makebox[0.24\textwidth]{\small Ground Truth}
        
    \includegraphics[width=0.24\textwidth]{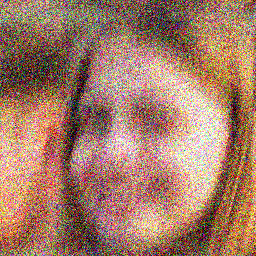}
    \includegraphics[width=0.24\textwidth]{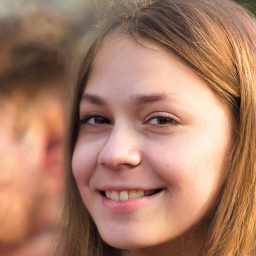}
    \includegraphics[width=0.24\textwidth]{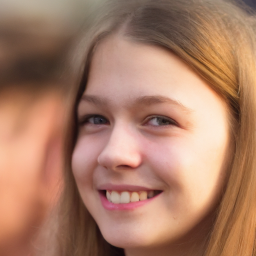}
    \includegraphics[width=0.24\textwidth]{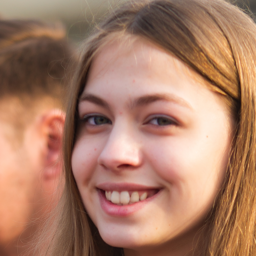}

    \includegraphics[width=0.24\textwidth]{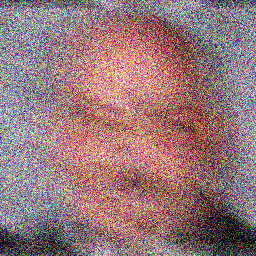}
    \includegraphics[width=0.24\textwidth]{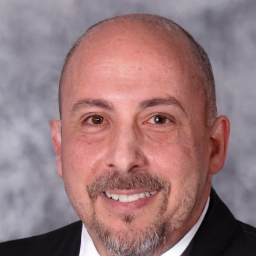}
    \includegraphics[width=0.24\textwidth]{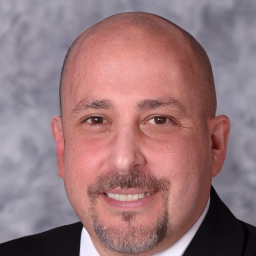}
    \includegraphics[width=0.24\textwidth]{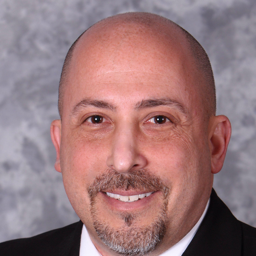}

    \includegraphics[width=0.24\textwidth]{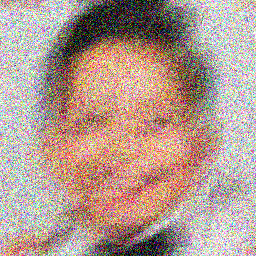}
    \includegraphics[width=0.24\textwidth]{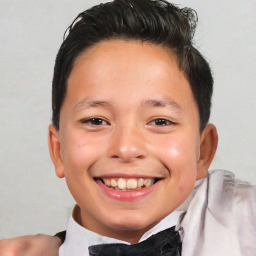}
    \includegraphics[width=0.24\textwidth]{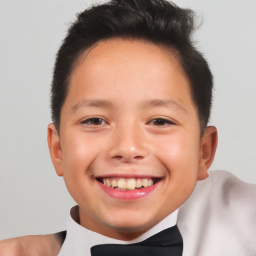}
    \includegraphics[width=0.24\textwidth]{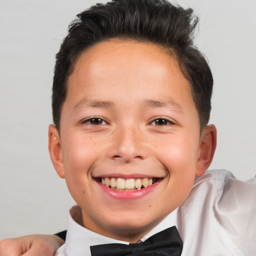}

    \includegraphics[width=0.24\textwidth]{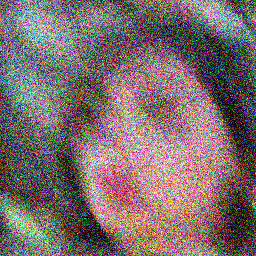}
    \includegraphics[width=0.24\textwidth]{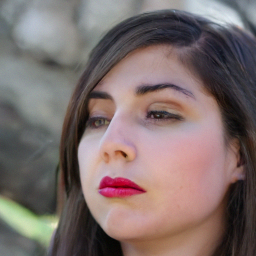}
    \includegraphics[width=0.24\textwidth]{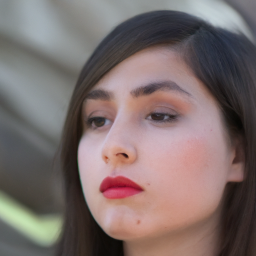}
    \includegraphics[width=0.24\textwidth]{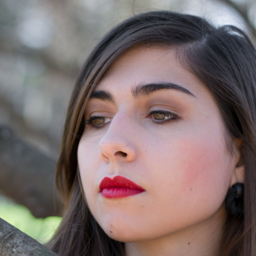}

    \includegraphics[width=0.24\textwidth]{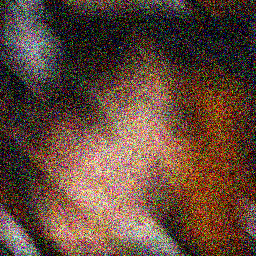}
    \includegraphics[width=0.24\textwidth]{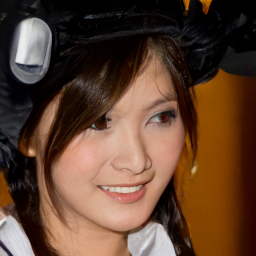}
    \includegraphics[width=0.24\textwidth]{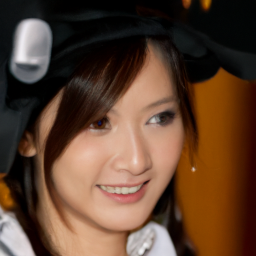}
    \includegraphics[width=0.24\textwidth]{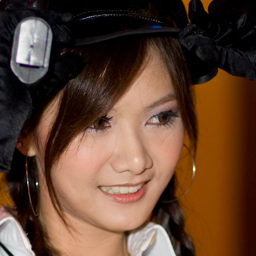}
    \end{minipage}
    \hfill
    % Second 5 rows (right panel)
    \begin{minipage}[t]{0.49\textwidth}
    \makebox[0.24\textwidth]{\small Measurement}
    \makebox[0.24\textwidth]{\small Ours (SDE)}
    \makebox[0.24\textwidth]{\small Ours (ODE)}
    \makebox[0.24\textwidth]{\small Ground Truth}

    \includegraphics[width=0.24\textwidth]{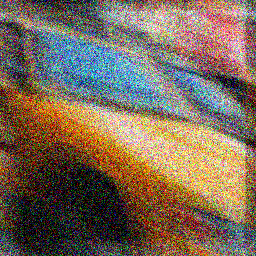}
    \includegraphics[width=0.24\textwidth]{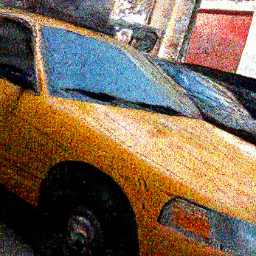}
    \includegraphics[width=0.24\textwidth]{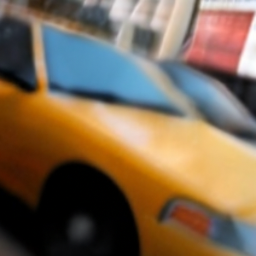}
    \includegraphics[width=0.24\textwidth]{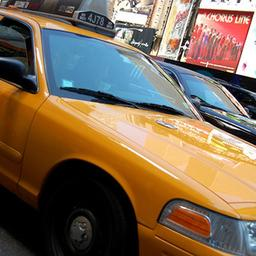}

    \includegraphics[width=0.24\textwidth]{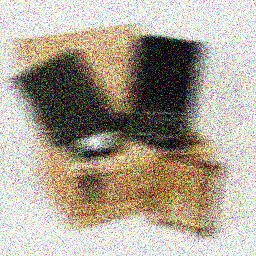}
    \includegraphics[width=0.24\textwidth]{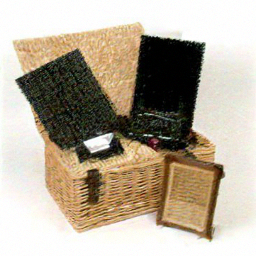}
    \includegraphics[width=0.24\textwidth]{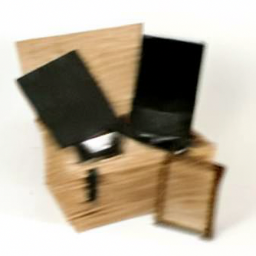}
    \includegraphics[width=0.24\textwidth]{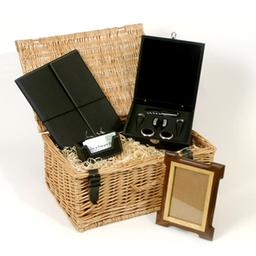}

    \includegraphics[width=0.24\textwidth]{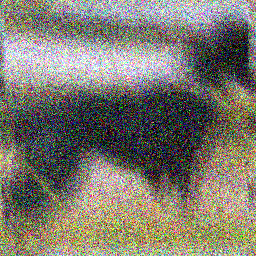}
    \includegraphics[width=0.24\textwidth]{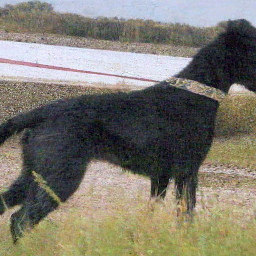}
    \includegraphics[width=0.24\textwidth]{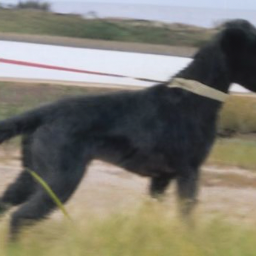}
    \includegraphics[width=0.24\textwidth]{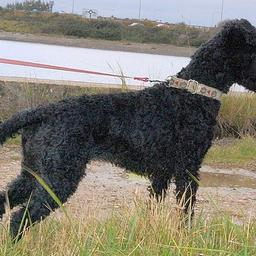}

     \includegraphics[width=0.24\textwidth]{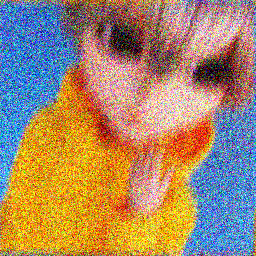}
    \includegraphics[width=0.24\textwidth]{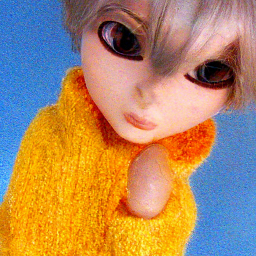}
    \includegraphics[width=0.24\textwidth]{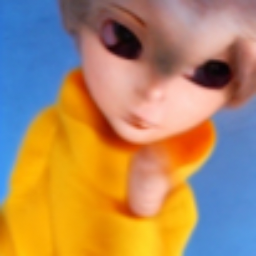}
    \includegraphics[width=0.24\textwidth]{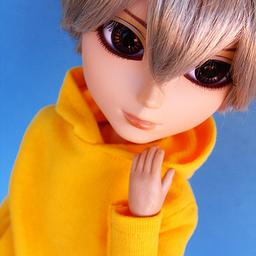}

     \includegraphics[width=0.24\textwidth]{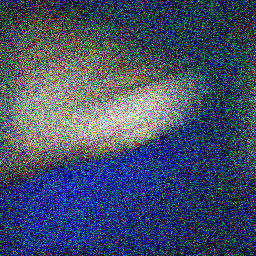}
    \includegraphics[width=0.24\textwidth]{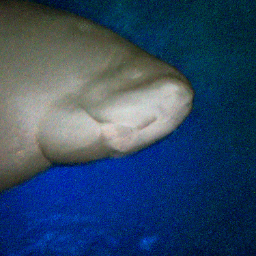}
    \includegraphics[width=0.24\textwidth]{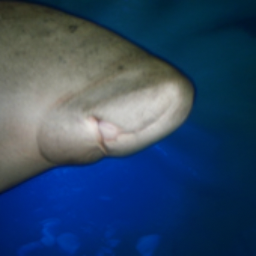}
    \includegraphics[width=0.24\textwidth]{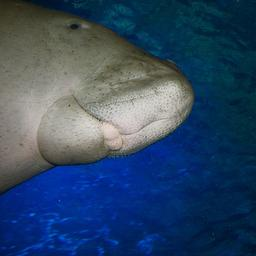}
    \end{minipage}
    \caption{Additional visual examples for the motion deblurring problem on FFHQ and ImageNet.}
    \label{fig:motion-deblurring-grid}
\end{figure}

\newpage
\begin{figure}[H]
    \centering
    
    % First 5 rows (Left panel)
    \begin{minipage}[t]{0.49\textwidth}
    \makebox[0.24\textwidth]{\small Measurement}
    \makebox[0.24\textwidth]{\small Ours (SDE)}
    \makebox[0.24\textwidth]{\small Ours (ODE)}
    \makebox[0.24\textwidth]{\small Ground Truth}
        
    \includegraphics[width=0.24\textwidth]{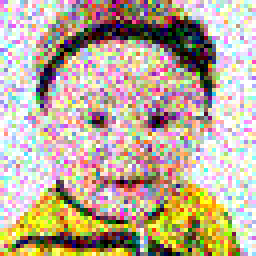}
    \includegraphics[width=0.24\textwidth]{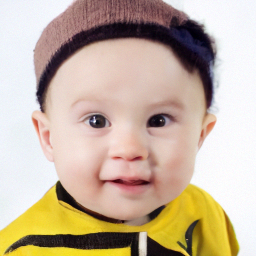}
    \includegraphics[width=0.24\textwidth]{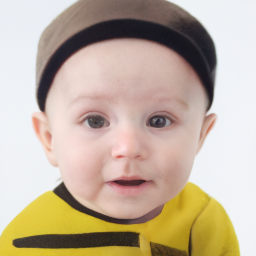}
    \includegraphics[width=0.24\textwidth]{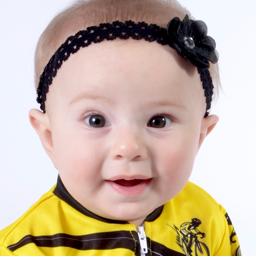}

    \includegraphics[width=0.24\textwidth]{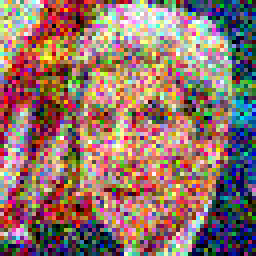}
    \includegraphics[width=0.24\textwidth]{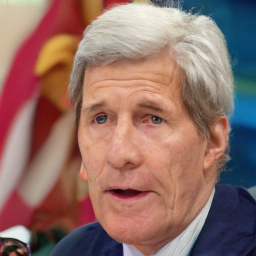}
    \includegraphics[width=0.24\textwidth]{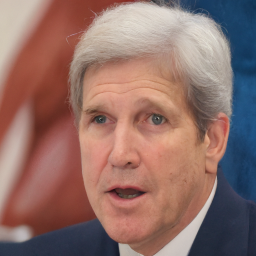}
    \includegraphics[width=0.24\textwidth]{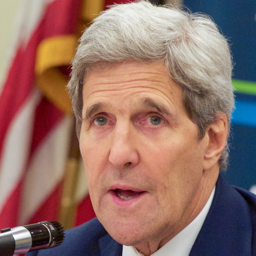}

    \includegraphics[width=0.24\textwidth]{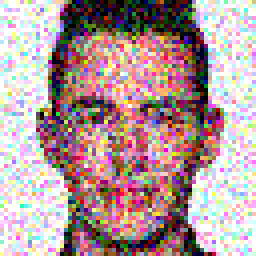}
    \includegraphics[width=0.24\textwidth]{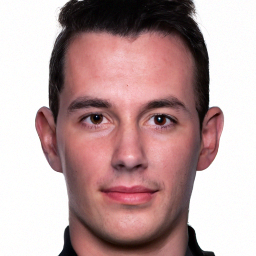}
    \includegraphics[width=0.24\textwidth]{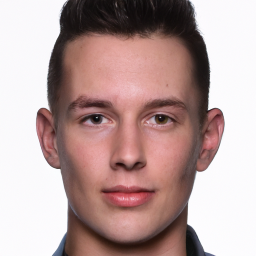}
    \includegraphics[width=0.24\textwidth]{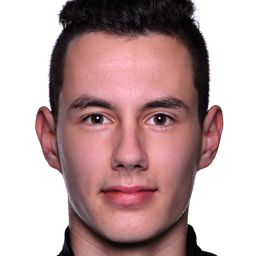}

    \includegraphics[width=0.24\textwidth]{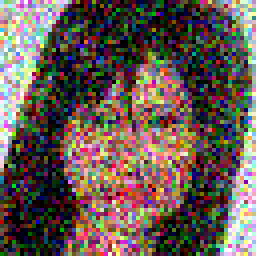}
    \includegraphics[width=0.24\textwidth]{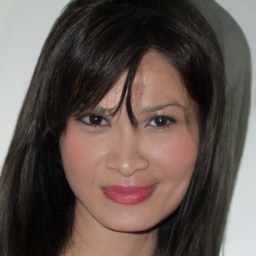}
    \includegraphics[width=0.24\textwidth]{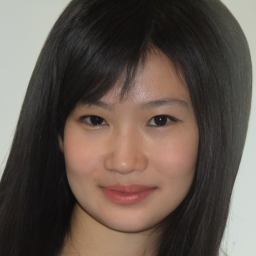}
    \includegraphics[width=0.24\textwidth]{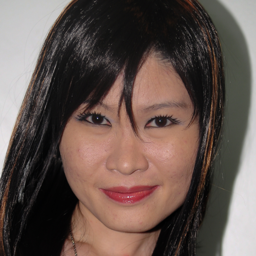}

    \includegraphics[width=0.24\textwidth]{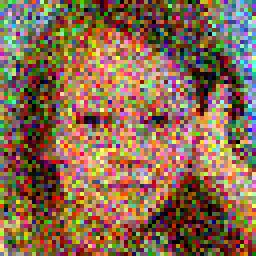}
    \includegraphics[width=0.24\textwidth]{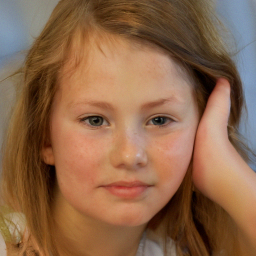}
    \includegraphics[width=0.24\textwidth]{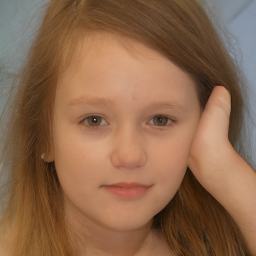}
    \includegraphics[width=0.24\textwidth]{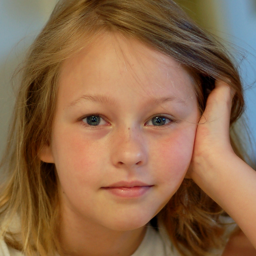}
    \end{minipage}
    \hfill
    % Second 5 rows (right panel)
    \begin{minipage}[t]{0.49\textwidth}
    \makebox[0.24\textwidth]{\small Measurement}
    \makebox[0.24\textwidth]{\small Ours (SDE)}
    \makebox[0.24\textwidth]{\small Ours (ODE)}
    \makebox[0.24\textwidth]{\small Ground Truth}
    
    \includegraphics[width=0.24\textwidth]{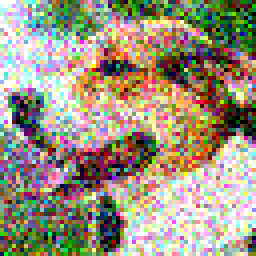}
    \includegraphics[width=0.24\textwidth]{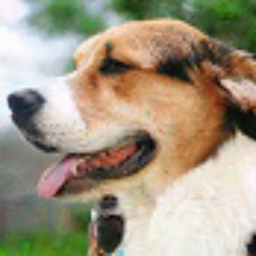}
    \includegraphics[width=0.24\textwidth]{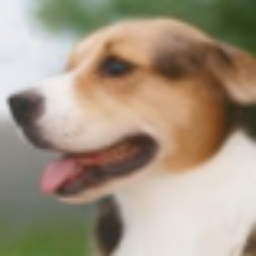}
    \includegraphics[width=0.24\textwidth]{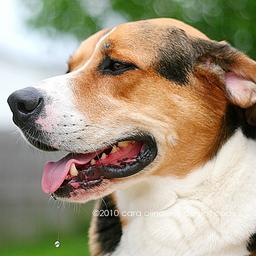}

     \includegraphics[width=0.24\textwidth]{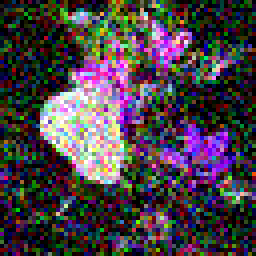}
    \includegraphics[width=0.24\textwidth]{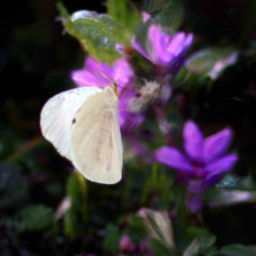}
    \includegraphics[width=0.24\textwidth]{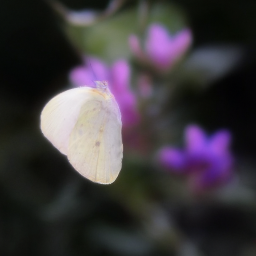}
    \includegraphics[width=0.24\textwidth]{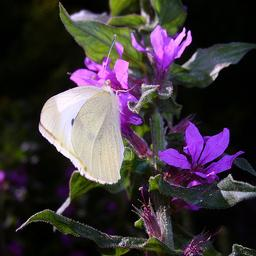}

    \includegraphics[width=0.24\textwidth]{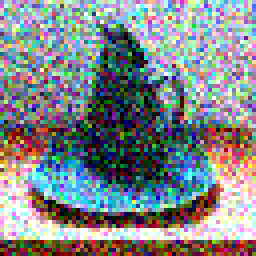}
    \includegraphics[width=0.24\textwidth]{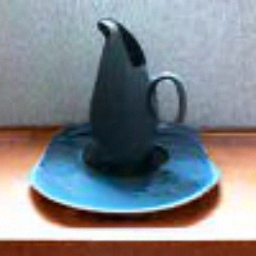}
    \includegraphics[width=0.24\textwidth]{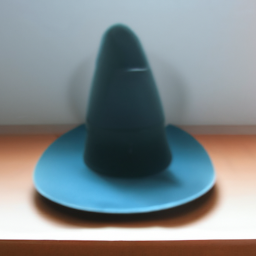}
    \includegraphics[width=0.24\textwidth]{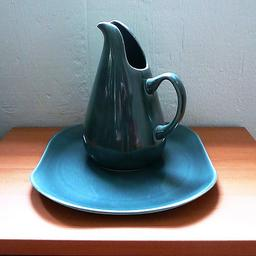}

    \includegraphics[width=0.24\textwidth]{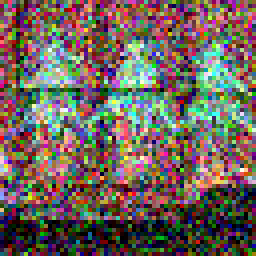}
    \includegraphics[width=0.24\textwidth]{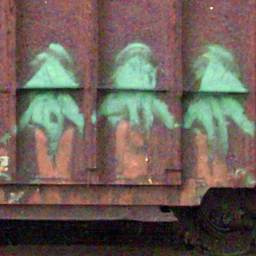}
    \includegraphics[width=0.24\textwidth]{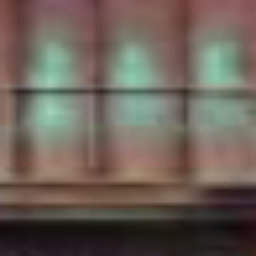}
    \includegraphics[width=0.24\textwidth]{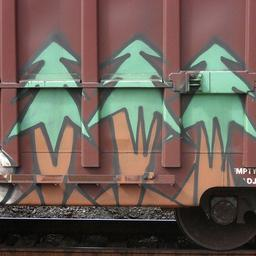}

    \includegraphics[width=0.24\textwidth]{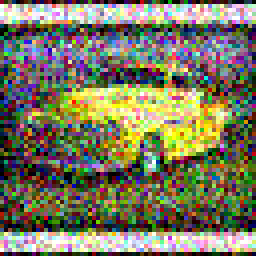}
    \includegraphics[width=0.24\textwidth]{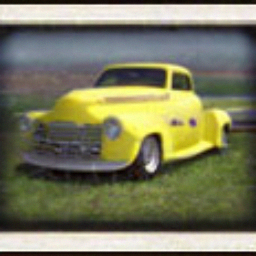}
    \includegraphics[width=0.24\textwidth]{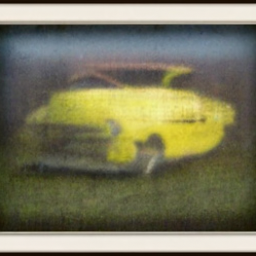}
    \includegraphics[width=0.24\textwidth]{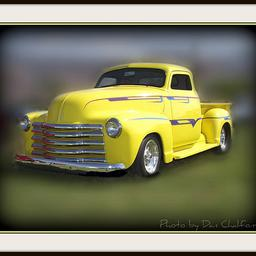}
    \end{minipage}
    \caption{Additional visual examples for the super-resolution problem on FFHQ and ImageNet.}
    \label{fig:super-resolution-grid}
\end{figure}

\begin{figure}[H]
    \centering    
    % First 5 rows (Left panel)
    \begin{minipage}[t]{0.49\textwidth}
    \makebox[0.24\textwidth]{\small Measurement}
    \makebox[0.24\textwidth]{\small Ours (SDE)}
    \makebox[0.24\textwidth]{\small Ours (ODE)}
    \makebox[0.24\textwidth]{\small Ground Truth}

    \includegraphics[width=0.24\textwidth]{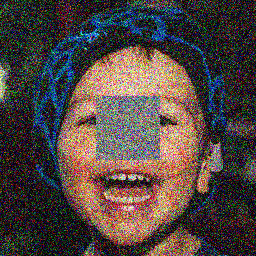}
    \includegraphics[width=0.24\textwidth]{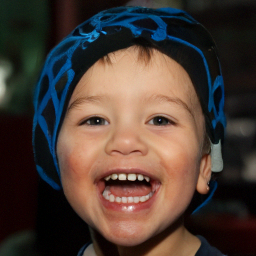}
    \includegraphics[width=0.24\textwidth]{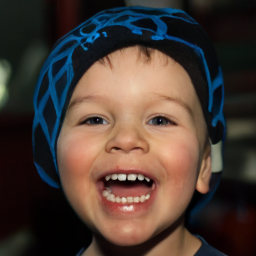}
    \includegraphics[width=0.24\textwidth]{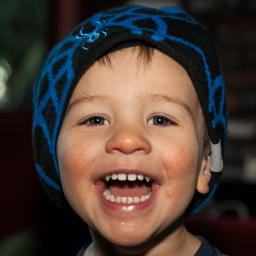}
        
    \includegraphics[width=0.24\textwidth]{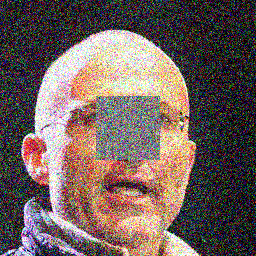}
    \includegraphics[width=0.24\textwidth]{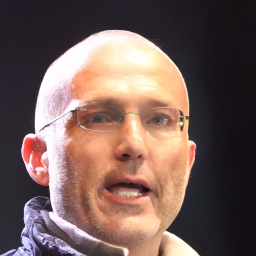}
    \includegraphics[width=0.24\textwidth]{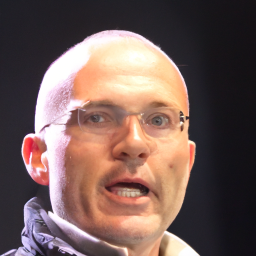}
    \includegraphics[width=0.24\textwidth]{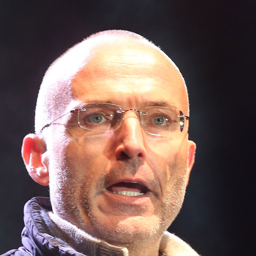}

    \includegraphics[width=0.24\textwidth]{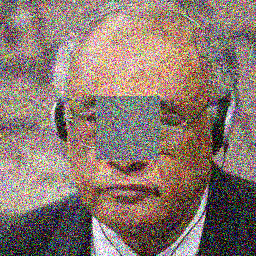}
    \includegraphics[width=0.24\textwidth]{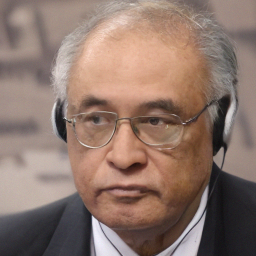}
    \includegraphics[width=0.24\textwidth]{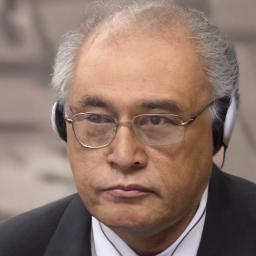}
    \includegraphics[width=0.24\textwidth]{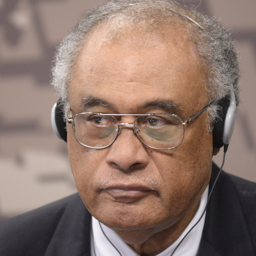}

    \includegraphics[width=0.24\textwidth]{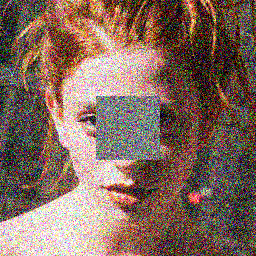}
    \includegraphics[width=0.24\textwidth]{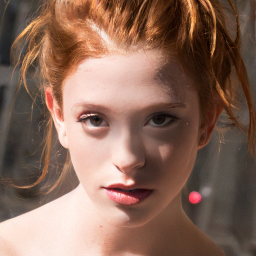}
    \includegraphics[width=0.24\textwidth]{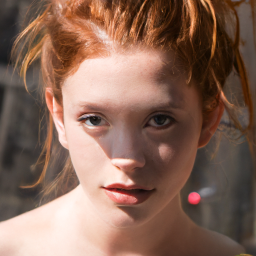}
    \includegraphics[width=0.24\textwidth]{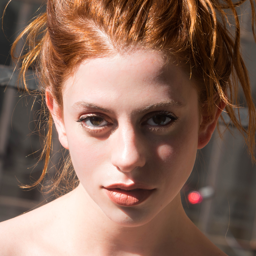}

    \includegraphics[width=0.24\textwidth]{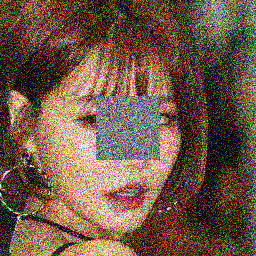}
    \includegraphics[width=0.24\textwidth]{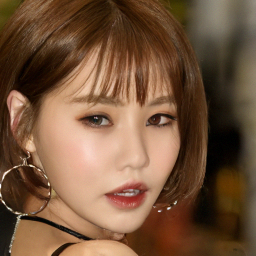}
    \includegraphics[width=0.24\textwidth]{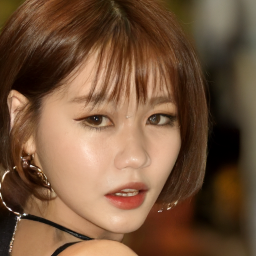}
    \includegraphics[width=0.24\textwidth]{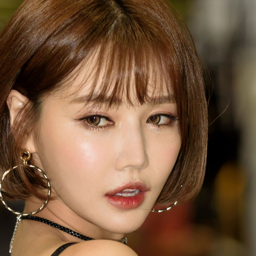}
    \end{minipage}
    \hfill
    % Second 5 rows (right panel)
    \begin{minipage}[t]{0.49\textwidth}
    \makebox[0.24\textwidth]{\small Measurement}
    \makebox[0.24\textwidth]{\small Ours (SDE)}
    \makebox[0.24\textwidth]{\small Ours (ODE)}
    \makebox[0.24\textwidth]{\small Ground Truth}
    
    \includegraphics[width=0.24\textwidth]{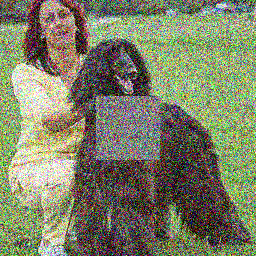}
    \includegraphics[width=0.24\textwidth]{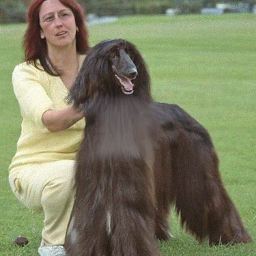}
    \includegraphics[width=0.24\textwidth]{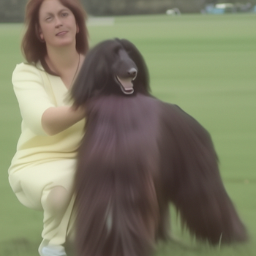}
    \includegraphics[width=0.24\textwidth]{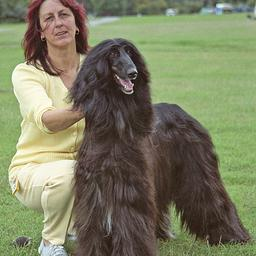}

    \includegraphics[width=0.24\textwidth]{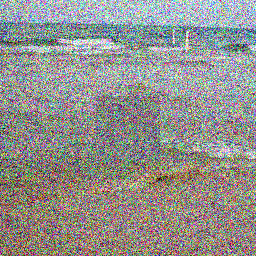}
    \includegraphics[width=0.24\textwidth]{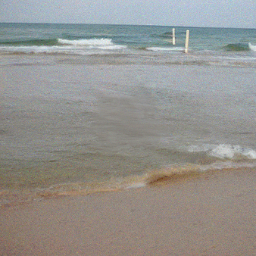}
    \includegraphics[width=0.24\textwidth]{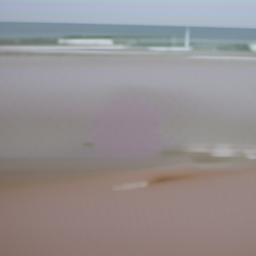}
    \includegraphics[width=0.24\textwidth]{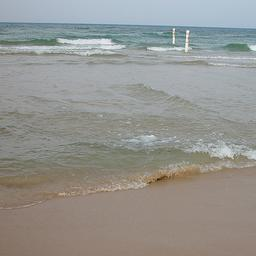}

    \includegraphics[width=0.24\textwidth]{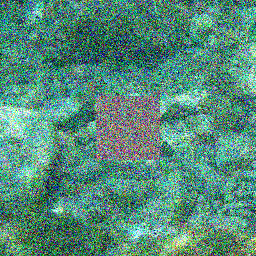}
    \includegraphics[width=0.24\textwidth]{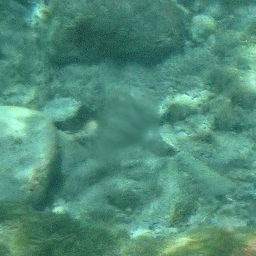}
    \includegraphics[width=0.24\textwidth]{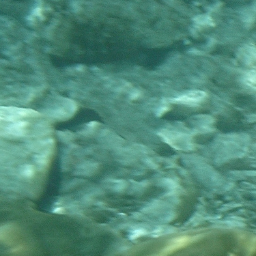}
    \includegraphics[width=0.24\textwidth]{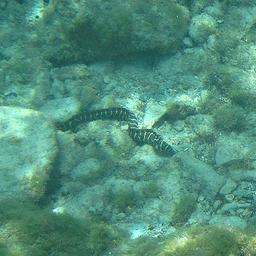}

    \includegraphics[width=0.24\textwidth]{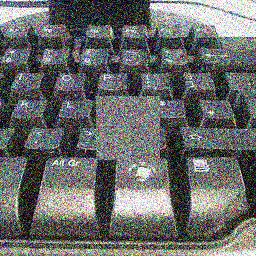}
    \includegraphics[width=0.24\textwidth]{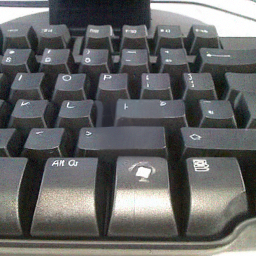}
    \includegraphics[width=0.24\textwidth]{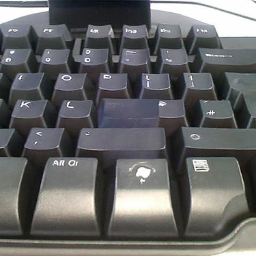}
    \includegraphics[width=0.24\textwidth]{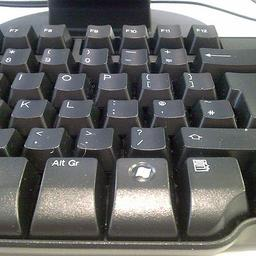}

    \includegraphics[width=0.24\textwidth]{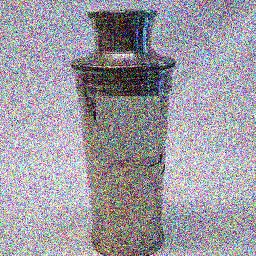}
    \includegraphics[width=0.24\textwidth]{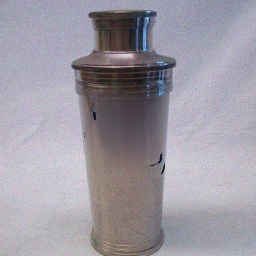}
    \includegraphics[width=0.24\textwidth]{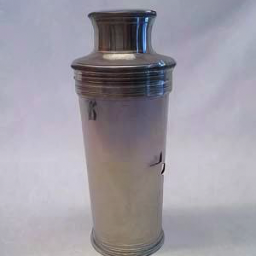}
    \includegraphics[width=0.24\textwidth]{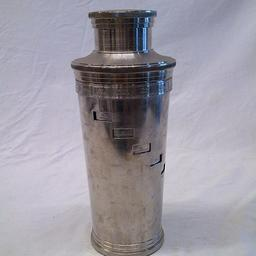}
    \end{minipage}
    \caption{Additional visual examples for the box inpainting problem on FFHQ and ImageNet.}
    \label{fig:box-inpainting-grid}
\end{figure}

%\haoxuan{Comparison II: remark about position of the left eye, the neck as well as the (potential) hairband above Comparison III: remark about the hair, the mustache and eyes extra visualization results ablation studies on the number of particles, computing time, plot of convergnce,  future work: discretized stepsizes}

\section{Supplementary Experiments}
\label{sec: supplementary experiments in app}
In order to better justify the effectiveness of our proposed methodology, we tested our methods on an additional set of tasks, such as validation on a larger set of images, nonlinear inverse problem and tuning of hyperparameters. Details of both the experimental settings and results are included in the following four subsections. Moreover, we note all the AFDPS method implemented below are slightly different compared to the version we used earlier. Specifically, the parametrized curve $\alpha:[0,T] \rightarrow [0,1]$ is picked to be the anenaling schedule $\alpha(t) := \frac{t}{T}$ instead of the uniform constant $\alpha(t) \equiv 1$. In the first subsection below, we evaluate the accuracy of AFDPS in sampling the posterior distribution on a simulated compressed sensing problem. 
\subsection{Synthetic experiment on compressed sensing}
\begin{figure}[h]
\centering
\begin{minipage}{0.45\linewidth}
\small
Similar to the synthetic experiment conducted in subsection 4.1 of~\citep{wu2024principled},, the measurement model is constructed by $\vy = \frac{1}{\sqrt{n}}A\vx + \vn \in \sR^m$, where $\vn \sim \gN(\vzero, 0.5 \mI_m)$ and the matrix $A \in \sR^{m \times n}$ satisfies $A_{ij} \sim \mathcal{N}(0,1)$ for any $1 \leq i \leq m$ and $1 \leq j \leq n$. The prior distribution on $\vx$ is further modeled as $\vx \sim \gN(\mu, \mI_n)$, where $\mu$ is picked to be the first image in the MNIST dataset. Under this setting, we have that the posterior distribution $p(\vx | \vy)$ is analytically tractable. Just as the experiment presented in section 4.1 of ~\citep{ren2025driftlite}, we evaluate the ensemble of weighted particles generated by AFDPS with respect to both the Maximum Mean Discrepancy (MMD) and Sliced Wsserstein Distance (SWD). Here we adopt the SDE-based implementation in Algorithm~\ref{alg:sde} and set the threshold on ESS for the resampling step to be $0.8$. From Figure~\ref{fig:synthetic_convergence}, we can see that AFDPS does converge to the target posterior distribution as the number of particles increases, as desired. 

\end{minipage}
\hfill
\begin{minipage}{0.5\linewidth}
    \centering
    \includegraphics[width=\linewidth]{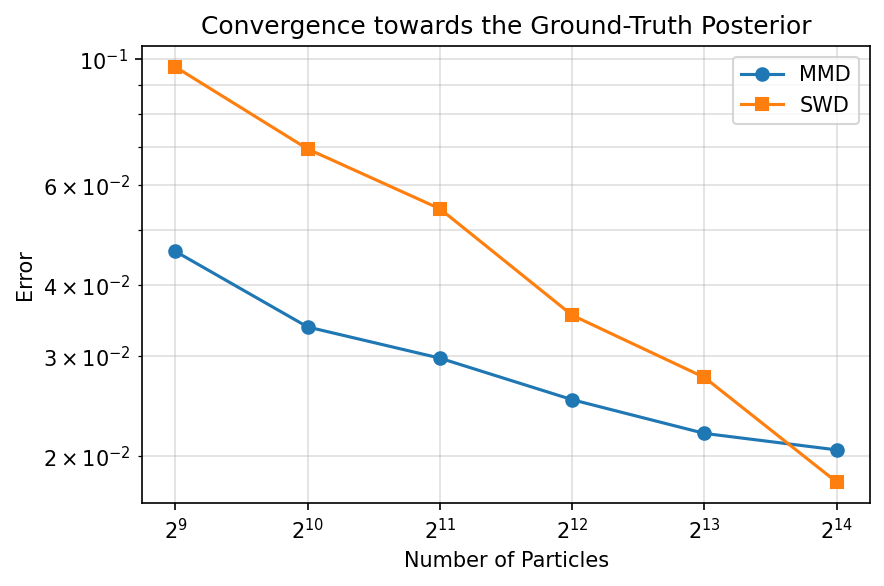}
    \caption{Convergence of AFDPS towards the ground-truth posterior versus number of particles.}
    \label{fig:synthetic_convergence}
\end{minipage}
\end{figure}

\begin{figure}[H]
    \centering
    % Left panel
    \begin{minipage}[t]{0.49\textwidth}
        \makebox[0.33\linewidth]{\small Ours (AFDPS)}%
        \makebox[0.33\linewidth]{\small Baseline (SGS)}%
        \makebox[0.33\linewidth]{\small Ground Truth}\\[2pt]
        
        \includegraphics[width=0.33\linewidth]{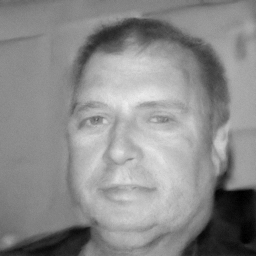}%
        \includegraphics[width=0.33\linewidth]{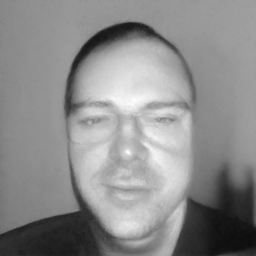}%
        \includegraphics[width=0.33\linewidth]{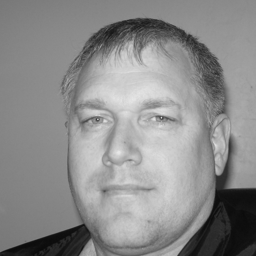}%

        \includegraphics[width=0.33\linewidth]{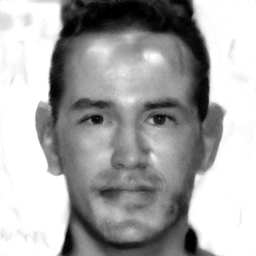}%
        \includegraphics[width=0.33\linewidth]{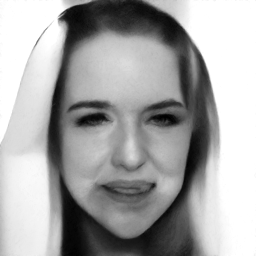}%
        \includegraphics[width=0.33\linewidth]{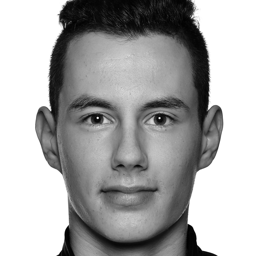}%

        \includegraphics[width=0.33\linewidth]{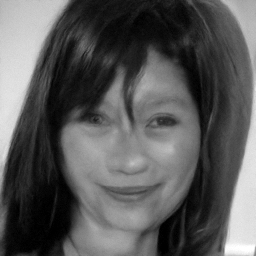}%
        \includegraphics[width=0.33\linewidth]{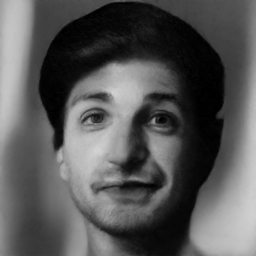}%
        \includegraphics[width=0.33\linewidth]{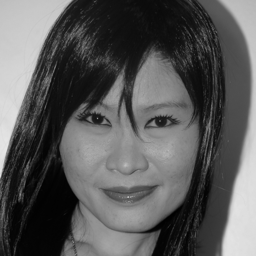}%
    \end{minipage}%
    \hfill
    % Right panel
    \begin{minipage}[t]{0.49\textwidth}
        \makebox[0.33\linewidth]{\small Ours (AFDPS)}%
        \makebox[0.33\linewidth]{\small Baseline (SGS)}%
        \makebox[0.33\linewidth]{\small Ground Truth}\\[2pt]
        
        \includegraphics[width=0.33\linewidth]{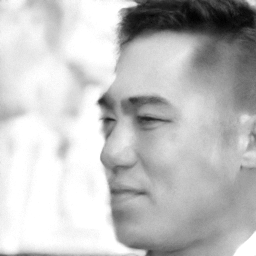}%
        \includegraphics[width=0.33\linewidth]{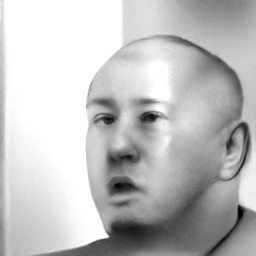}%
        \includegraphics[width=0.33\linewidth]{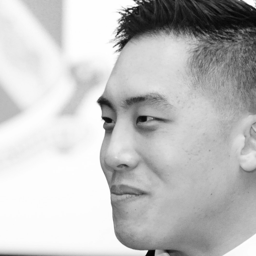}%

        \includegraphics[width=0.33\linewidth]{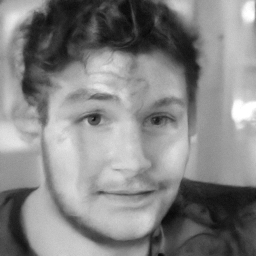}%
        \includegraphics[width=0.33\linewidth]{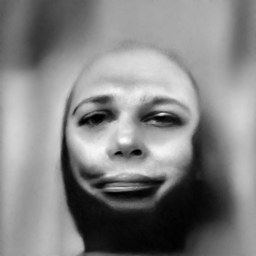}%
        \includegraphics[width=0.33\linewidth]{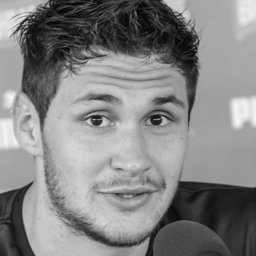}%

        \includegraphics[width=0.33\linewidth]{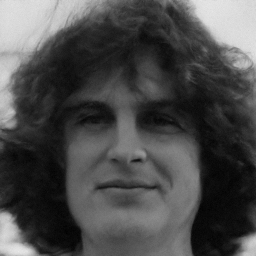}%
        \includegraphics[width=0.33\linewidth]{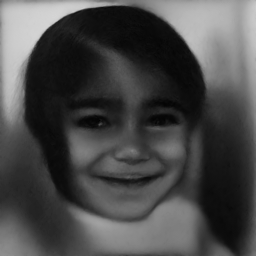}%
        \includegraphics[width=0.33\linewidth]{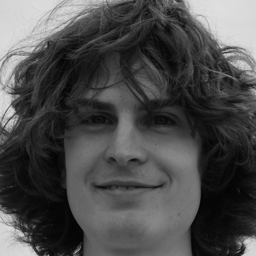}%
    \end{minipage}
    \caption{Comparison between AFDPS and SGS-EDM~\citep{wu2024principled} on Fourier Phase Retrieval.}
    \label{fig:fourier-phase-retrieval}
\end{figure}

\subsection{Nonlinear inverse problems}
The second subsection here is devoted to testing the proposed AFDPS method on a challenging nonlinear inverse problem. Here we consider the Fourier phase retrieval (FPR) problem, which is also tested in~\citep{chung2022score,wu2024principled}. Specifically, by using a matrix $F$ to represent the 2D Fourier Transform, we have that the measurement model can be expressed as $\vy = \left\|FP\vx\right\|_2 + \vn$, where $P$ denotes the oversampling matrix that pads the input $\vx$ with zero entries in 2D matrix form. Here we set the oversampling ratio to be $8$ and the observed Gaussian noise $\gN(\vzero, 0.2 \mI_m)$. We note that the inverse problem tested in this subsection is more challenging compared to the one tested on SGS-EDM in subsection 4.2 of~\citep{wu2024principled}, as the observational noise has much larger variance. Both AFDPS and SGS-EDM are then tested on the first 100 images from the FFHQ dataset. The baseline SGS-EDM achieves an avrage PSNR value of 12.63 and an average LPIPS value of 0.4052. In contrast, our method AFDPS yields an average PSNR value of 12.63 and an average LPIPS value of 0.4680, which outperforms SGS-EDM. A few examples of recovered images returned by both methods are also exhibited in~\ref{fig:fourier-phase-retrieval} above, which demonstrates the supremacy of our method compared to SGS-EDM.  

\subsection{Extra baseline and validation on a larger dataset}
The third subsection here focuses on comparing AFDPS with the Decoupled Annealing Posterior Sampling (DAPS) algorithm proposed in a recent work~\citep{zhang2025improving}. Following the experimental setup described in Appendix~\ref{sec: implementation detail app} above, we evaluate the proposed method on Gaussian and motion deblurring problems using the first 100 images of the FFHQ-256 dataset. The corresponding PSNR and LPIPS scores are reported in Table~\ref{tab: comparison between DAPS and AFDPS}. From the results we can see that AFDPS does surpass DAPS on the two deblurring problems. 

\begin{table}[htbp]
\centering
\caption{Results of the comparison between DAPS and AFDPS on two deblurring problems.}
\resizebox{0.85\textwidth}{!}{%
\begin{tabular}{l|cc|cc}
\noalign{\hrule height 1pt}
\multirow{2}{*}{\bfseries Method}
  & \multicolumn{2}{c|}{\bfseries Gaussian Deblurring}
  & \multicolumn{2}{c}{\bfseries Motion Deblurring}\\
  & PSNR ($\uparrow$) & LPIPS ($\downarrow$)
  & PSNR ($\uparrow$) & LPIPS ($\downarrow$)\\
\noalign{\hrule height 1pt}
\rowcolor{lightgray}
DAPS (100 images)              & 22.10 & 0.4580 & 23.02 & 0.4661 \\
\noalign{\hrule height 1pt}
\bfseries AFDPS-SDE (100 images)  & 25.19 & 0.1291 & 24.07 & 0.1530\\
\bfseries AFDPS-SDE (1000 images) & 25.11 & 0.1329 & 24.00 & 0.1571\\
\noalign{\hrule height 1pt}
\end{tabular}%
}
\label{tab: comparison between DAPS and AFDPS}
\end{table}

To further validate the effectiveness of AFDPS, we extend our evaluation to the first 1,000 images from the FFHQ-256 dataset. We also include the super-resolution task alongside the two deblurring problems above. For the added super-resolution task, AFDPS achieves an average PSNR value of 23.13 and an average LPIPS value of 0.1775 on the first 1000 images of the FFHQ-256 dataset, which demonstrates the effectiveness of AFDPS. An additional set of visual examples from the first 1000 images of FFHQ-256 on the three tasks are also included in Figures~\ref{fig:gaussian-deblurring-FFHQ256-1000},~\ref{fig:motion-deblurring-FFHQ256-1000} and~\ref{fig:super-resolution-FFHQ256-1000} below, respectively.

\begin{figure}[H]
\centering
\makebox[0.16\linewidth]{\small AFDPS}%
\makebox[0.16\linewidth]{\small Ground-Truth}%
\makebox[0.16\linewidth]{\small AFDPS}%
\makebox[0.16\linewidth]{\small Ground Truth}%
\makebox[0.16\linewidth]{\small AFDPS}%
\makebox[0.16\linewidth]{\small Ground Truth}\\[2pt]
% Row 1
\includegraphics[width=0.16\linewidth]{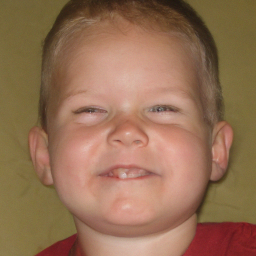}%
\includegraphics[width=0.16\linewidth]{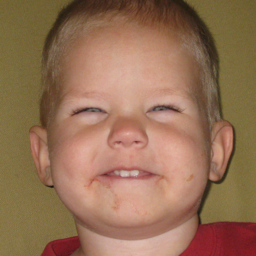}%
\includegraphics[width=0.16\linewidth]{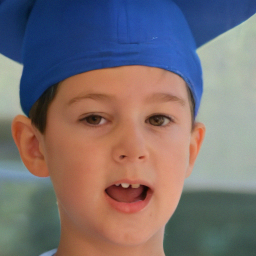}%
\includegraphics[width=0.16\linewidth]{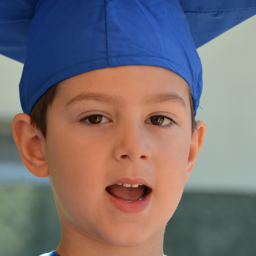}%
\includegraphics[width=0.16\linewidth]{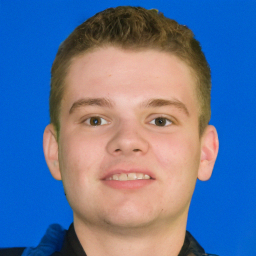}%
\includegraphics[width=0.16\linewidth]{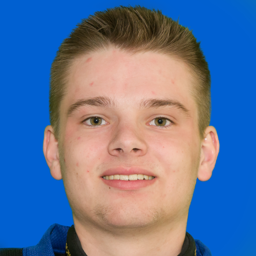}\\
% Row 2
\includegraphics[width=0.16\linewidth]{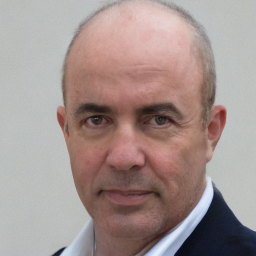}%
\includegraphics[width=0.16\linewidth]{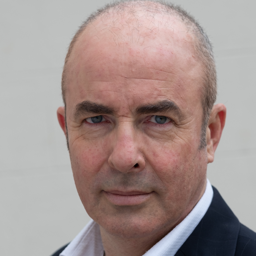}%
\includegraphics[width=0.16\linewidth]{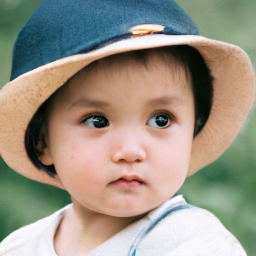}%
\includegraphics[width=0.16\linewidth]{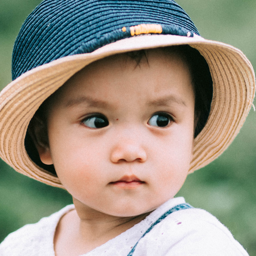}%
\includegraphics[width=0.16\linewidth]{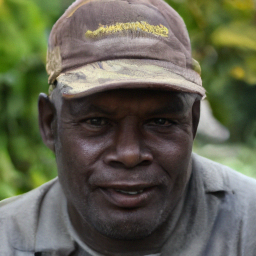}%
\includegraphics[width=0.16\linewidth]{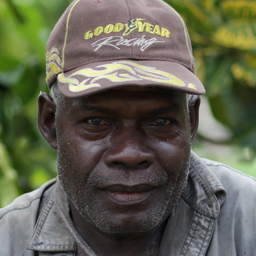}
\caption{Additional visual examples for the Gaussian deblurring problem on 1000 images from FFHQ}
\label{fig:gaussian-deblurring-FFHQ256-1000}
\end{figure}

\begin{figure}[H]
\centering
\makebox[0.1666\linewidth]{\small AFDPS}%
\makebox[0.1666\linewidth]{\small Ground-Truth}%
\makebox[0.1666\linewidth]{\small AFDPS}%
\makebox[0.1666\linewidth]{\small Ground Truth}%
\makebox[0.1666\linewidth]{\small AFDPS}%
\makebox[0.1666\linewidth]{\small Ground Truth}\\[2pt]
% Row 1
\includegraphics[width=0.1666\linewidth]{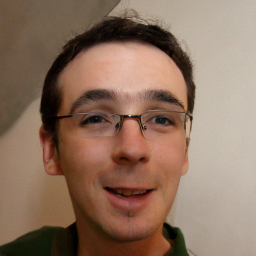}%
\includegraphics[width=0.1666\linewidth]{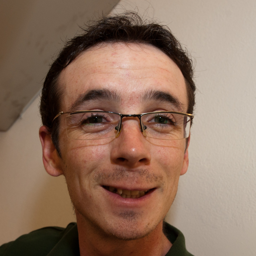}%
\includegraphics[width=0.1666\linewidth]{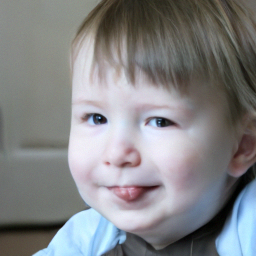}%
\includegraphics[width=0.1666\linewidth]{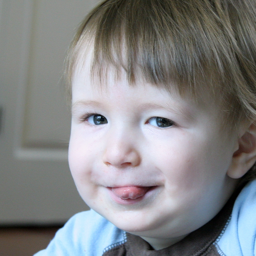}%
\includegraphics[width=0.1666\linewidth]{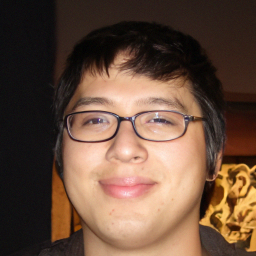}%
\includegraphics[width=0.1666\linewidth]{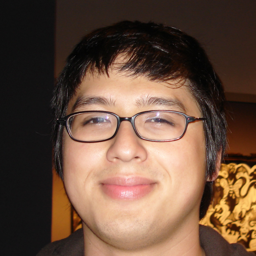}\\
% Row 2
\includegraphics[width=0.1666\linewidth]{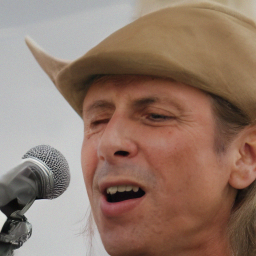}%
\includegraphics[width=0.1666\linewidth]{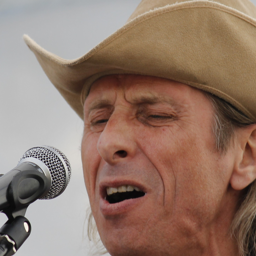}%
\includegraphics[width=0.1666\linewidth]{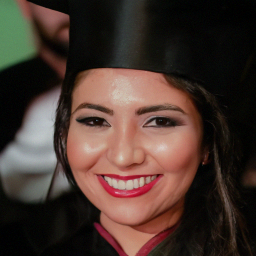}%
\includegraphics[width=0.1666\linewidth]{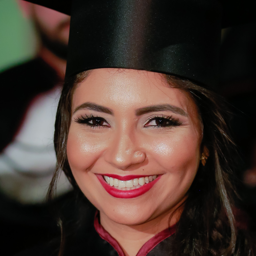}%
\includegraphics[width=0.1666\linewidth]{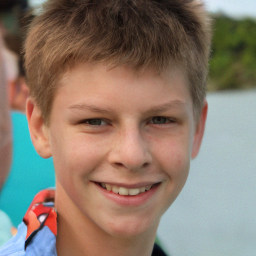}%
\includegraphics[width=0.1666\linewidth]{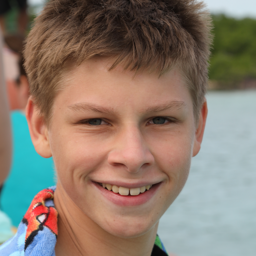}
\caption{Additional visual examples for the motion deblurring problem on 1000 images from FFHQ}
\label{fig:motion-deblurring-FFHQ256-1000}
\end{figure}

\begin{figure}[H]
\centering
\makebox[0.1666\linewidth]{\small AFDPS}%
\makebox[0.1666\linewidth]{\small Ground-Truth}%
\makebox[0.1666\linewidth]{\small AFDPS}%
\makebox[0.1666\linewidth]{\small Ground Truth}%
\makebox[0.1666\linewidth]{\small AFDPS}%
\makebox[0.1666\linewidth]{\small Ground Truth}\\[2pt]
% Row 1
\includegraphics[width=0.1666\linewidth]{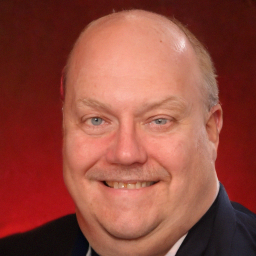}%
\includegraphics[width=0.1666\linewidth]{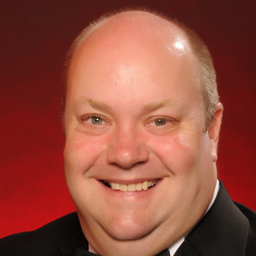}%
\includegraphics[width=0.1666\linewidth]{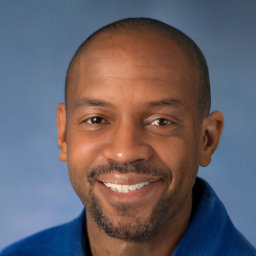}%
\includegraphics[width=0.1666\linewidth]{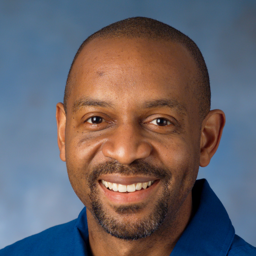}%
\includegraphics[width=0.1666\linewidth]{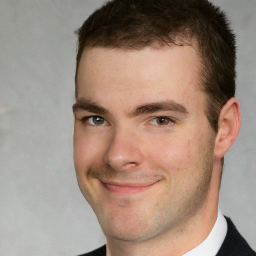}%
\includegraphics[width=0.1666\linewidth]{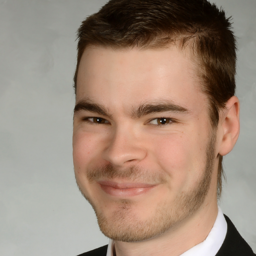}\\
% Row 2
\includegraphics[width=0.1666\linewidth]{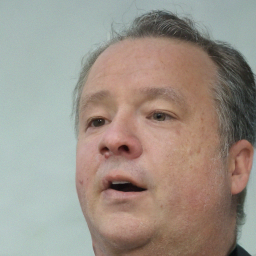}%
\includegraphics[width=0.1666\linewidth]{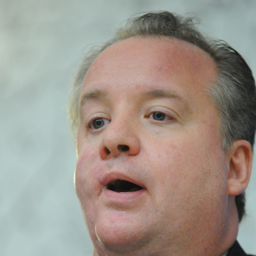}%
\includegraphics[width=0.1666\linewidth]{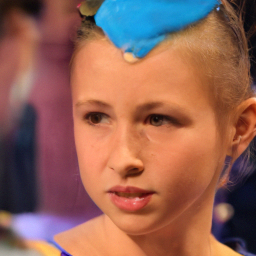}%
\includegraphics[width=0.1666\linewidth]{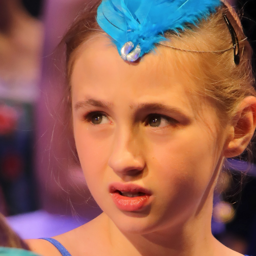}%
\includegraphics[width=0.1666\linewidth]{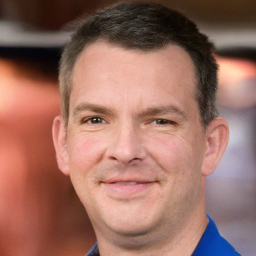}%
\includegraphics[width=0.1666\linewidth]{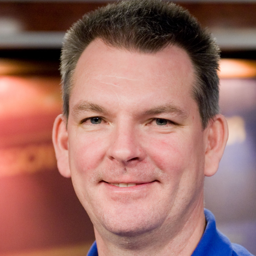}
\caption{Additional visual examples for the super-resolution problem on 1000 images from FFHQ}
\label{fig:super-resolution-FFHQ256-1000}
\end{figure}

\subsection{Tuning and studies of other hyperparameters}
In the last subsection, we investigate the impact of a  few important parameters on the performance of the AFDPS algorithm. Throughout this subsection, we focus on the Gaussian deblurring problem applied to the first 100 images of the FFHQ-256 dataset and fix the number of discretized timesteps to be $500$. 

\begin{figure}[H]
    \centering
    \includegraphics[width=0.9\linewidth]{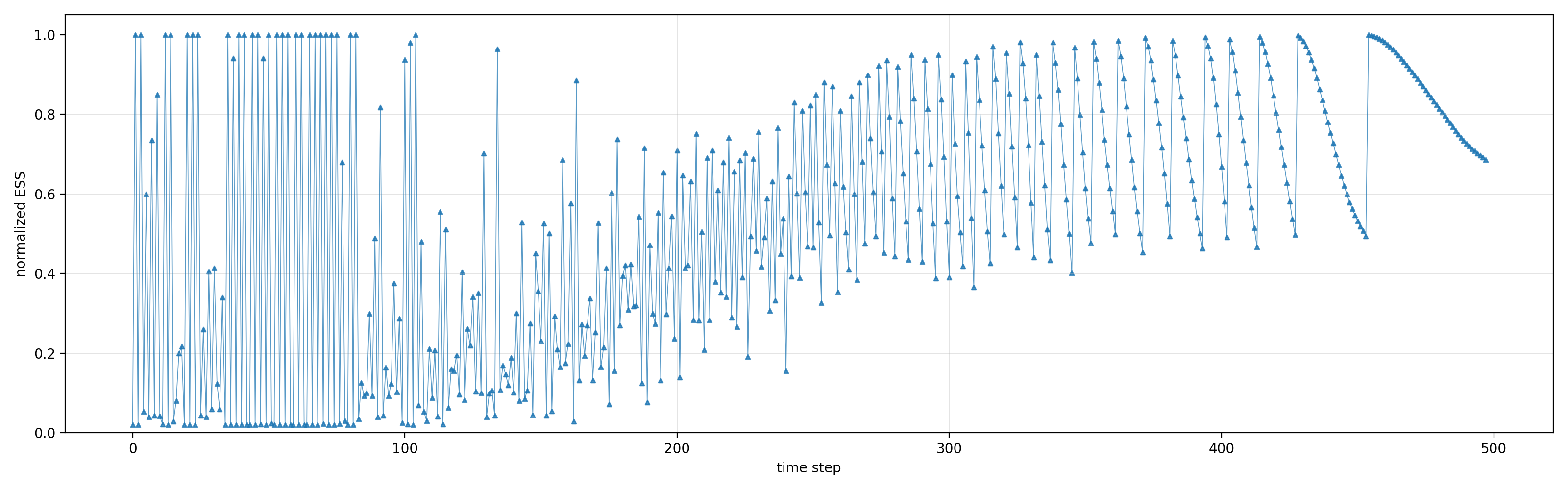}
    \caption{Plot of ESS with respect to time on the Gaussian deblurring task}
    \label{fig:evolving_ESS}
\end{figure}

 The first parameter we investigate here is the Effective Sample Size (ESS). Here we fix the number of particles to be $50$. By setting the threshold for resampling to be $0.5$ and plotting the ESS with respect to the discretized timestep, we obtain Figure~\ref{fig:evolving_ESS} above. From the plot we can see that the ESS oscillates during the first half of the inference time. For the second half of the inference time, ESS gradually decays and gets periodically restored through resampling, which is similar to the behavior exhibited in Figure 2 of~\citep{ren2025driftlite}.

\begin{figure}[h]
    \centering
    \includegraphics[width=0.9\linewidth]{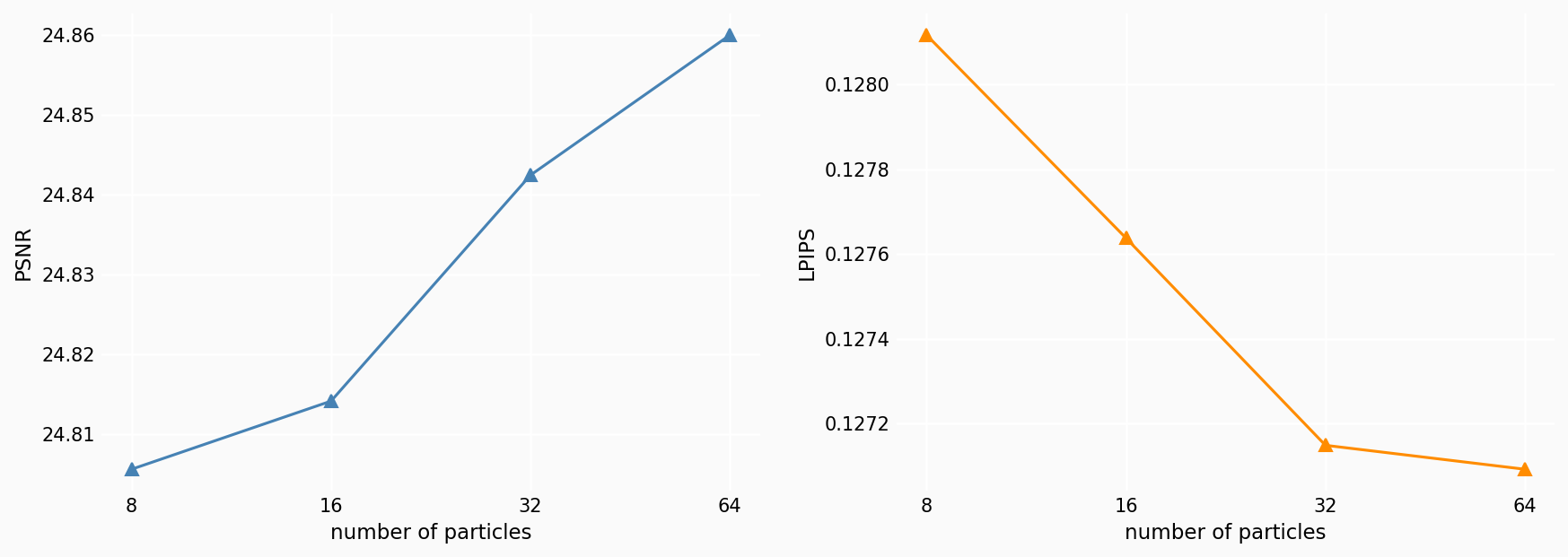}
    \caption{Tuning the number of particles in AFDPS-SDE on the Gaussian deblurring task}
    \label{fig:evolving_number_of_particles}
\end{figure}

Secondly, we study how the change in the number of particles might impact the AFDPS-SDE algorithm's performance in terms of the PSNR and LPIPS metrics. From the two plots in Figure~\ref{fig:evolving_number_of_particles}, we can see that image quality gets improved as the number of particles increases. Moreover, the gains are marginal as the number of particles becomes sufficiently large, suggesting that around $50$ particles can yield a balance between the reconstruction quality and computational cost.

\begin{figure}[h]
    \centering
    \includegraphics[width=0.9\linewidth]{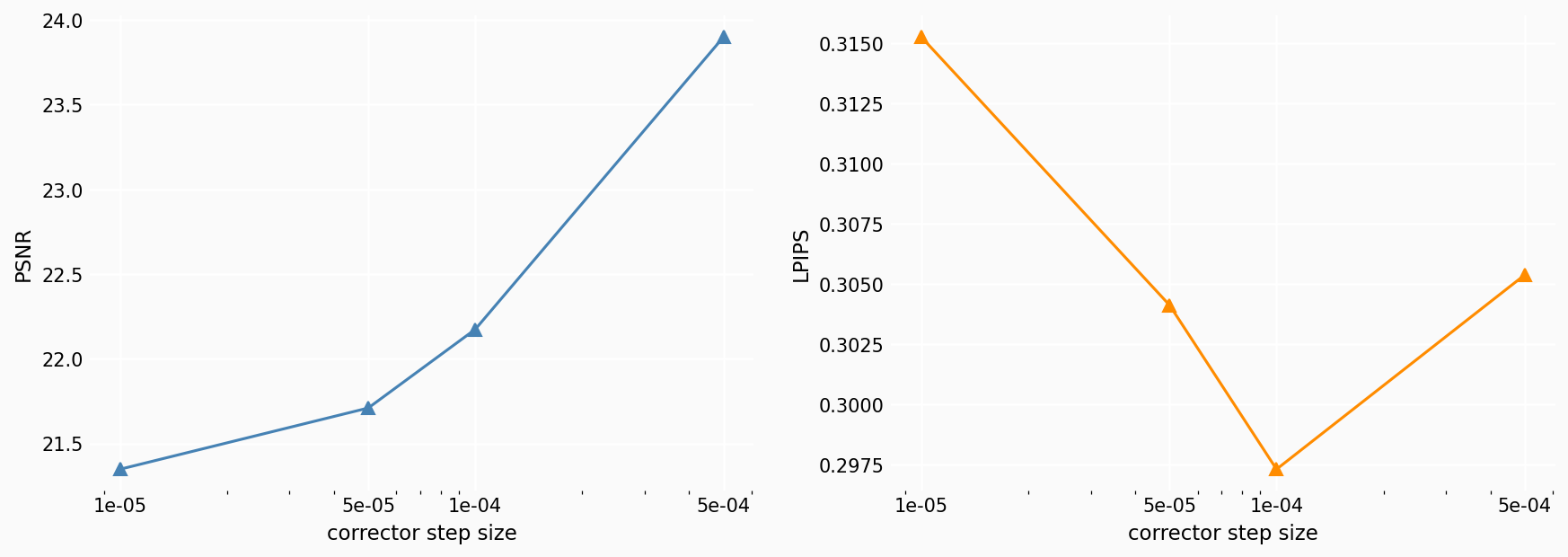}
    \caption{Tuning the stepsize of the corrector in AFDPS-ODE on the Gaussian deblurring task}
    \label{fig:evolving_corrector_stepsize}
\end{figure}

Finally, we investigate the impact of the corrector stepsize on the performance of the AFDPS-ODE algorithm. Here we fix the number of particles, the number of discretized timesteps and the number of corrector steps to be $50$, $500$ and $5$, respectively. Just as Figure~\ref{fig:evolving_number_of_particles} above, we plot both the PSNR and the LPIPS metric with respect to different stepsizes. From the plots in Figure~\ref{fig:evolving_corrector_stepsize}, we have that larger stepsizes for the corrector step in general yield images of better quality. With additional computational resource, it would also be interesting to investigate how the number of corrector steps might impact the algorithm's performance in the future. 

%Parameters associated with the Corrector Step

\end{document}